\documentclass[twoside,11pt]{article}
\usepackage{jmlr2e}
\usepackage{natbib}
\usepackage{epsfig,ifthen}
\usepackage{latexsym}
\usepackage{amsfonts}
\usepackage{eepic}
\usepackage{amsopn}
\usepackage{amsmath,amssymb,rotating,graphicx,rotating,float}
\usepackage{accents}
\usepackage[dvips,backref,pageanchor=true,plainpages=false, 
pdfpagelabels,bookmarks,bookmarksnumbered,breaklinks,
pdfborder={0 0 0},
]{hyperref}
\usepackage{cite}
\usepackage[mathscr]{eucal}
\usepackage{url}

\usepackage{ulem}
\usepackage{color}

\newlength{\dhatheight}

\newcommand{\Px}{\mathcal{P}}
\newcommand{\PXY}{\mathcal{P}_{XY}}

\newcommand{\tsyba}{\alpha}
\newcommand{\tsybca}{a}

\newcommand{\s}{\mathfrak{s}}

\newcommand{\dc}{\theta}
\newcommand{\zc}{\varphi}
\newcommand{\maxzc}{\hat{\zc}}

\newcommand{\dd}{D}

\newcommand{\bound}{\beta}

\newcommand{\zo}{{\scriptscriptstyle{01}}}

\newcommand{\ERM}{{\rm ERM}}

\newcommand{\vc}{{\rm vc}}
\renewcommand{\dim}{d}
\newcommand{\cs}{n}
\newcommand{\emp}[1]{#1}
\newcommand{\empcs}{\emp{\cs}}
\newcommand{\covering}{{\cal{N}}}
\newcommand{\G}{{\cal{G}}}

\newcommand{\Log}{{\rm Log}}
\newcommand{\polylog}{{\rm polylog}}

\newcommand{\CAL}{{\rm CAL}}

\newcommand{\ZCPassive}{{\rm Algorithm~1}}
\newcommand{\conf}{\delta}
\newcommand{\eps}{\varepsilon}

\newcommand{\X}{\mathcal X}

\newcommand{\Y}{\mathcal Y}
\newcommand{\alg}{\mathbb{A}}
\renewcommand{\H}{\mathcal H}
\renewcommand{\L}{\mathcal L}

\newcommand{\Z}{\mathcal{Z}}
\newcommand{\target}{f^{\star}}
\newcommand{\agtarget}{h^{\star}}

\newcommand{\er}{{\rm er}}

\newcommand{\Ball}{{\rm B}}
\DeclareSymbolFont{bbold}{U}{bbold}{m}{n}
\DeclareSymbolFontAlphabet{\mathbbold}{bbold}
\newcommand{\ind}{\mathbbold{1}}

\newcommand{\C}{\mathbb{C}}
\newcommand{\A}{\mathcal{A}}

\renewcommand{\P}{\mathbb P}
\newcommand{\nats}{\mathbb{N}}
\newcommand{\reals}{\mathbb{R}}

\newcommand{\E}{\mathbb{E}}

\newcommand{\diam}{{\rm diam}}
\newcommand{\DIS}{{\rm DIS}}

\newcommand{\argmin}{\mathop{\rm argmin}}

\renewcommand{\a}{a}
\renewcommand{\b}{b}

\newsavebox{\savepar}
\newenvironment{bigboxit}{\begin{center}\begin{lrbox}{\savepar}
\begin{minipage}[h]{5.2in}
\normalfont
\begin{flushleft}}
{\end{flushleft}\end{minipage}\end{lrbox}\fbox{\usebox{\savepar}}
\end{center}}

\newenvironment{bigbigboxit}{\begin{center}\begin{lrbox}{\savepar}
\begin{minipage}[h]{5.8in}
\normalfont
\begin{flushleft}}
{\end{flushleft}\end{minipage}\end{lrbox}\fbox{\usebox{\savepar}}
\end{center}}

\makeatletter
\newcommand{\vast}{\bBigg@{3}}
\newcommand{\Vast}{\bBigg@{4}}
\makeatother

\renewenvironment{proof}[1][]{\par\noindent{\bf Proof #1\ }}{\hfill\BlackBox\\[2mm]}

\jmlrheading{17}{2016}{1-55}{12/15; Revised 7/16}{8/16}{Steve Hanneke}

\ShortHeadings{Refined Error Bounds}{Hanneke}
\firstpageno{1}

\begin{document}

\title{Refined Error Bounds for Several Learning Algorithms}

\author{%
\name Steve Hanneke \email steve.hanneke@gmail.com%
}

\editor{John Shawe-Taylor}

\maketitle

\begin{abstract}%
This article studies the achievable guarantees on the error rates of 
certain learning algorithms, with particular focus on refining logarithmic factors.
Many of the results are based on a general technique for obtaining bounds on the error rates
of sample-consistent classifiers with monotonic error regions, in the realizable case. 
We prove bounds of this type expressed in terms of either the VC dimension or the sample compression size.
This general technique also enables us to derive several new bounds on the error rates of 
general sample-consistent learning algorithms, as well as refined bounds on the label complexity of the CAL
active learning algorithm.  Additionally, we establish a simple necessary and sufficient condition for the 
existence of a distribution-free bound on the error rates of all sample-consistent learning rules,
converging at a rate inversely proportional to the sample size.  
We also study learning in the presence of classification noise, deriving a new 
excess error rate guarantee for general VC classes under Tsybakov's noise condition, and establishing a 
simple and general necessary and sufficient condition for the minimax excess risk under bounded noise to 
converge at a rate inversely proportional to the sample size.
\end{abstract}


\section{Introduction}

Supervised machine learning is a classic topic, in which a learning rule is tasked with producing a classifier
that mimics the classifications that would be assigned by an expert for a given task.  To achieve this, the 
learner is given access to a collection of examples (assumed to be i.i.d.) labeled with the correct classifications.  
One of the major theoretical questions of interest in learning theory is: How many examples are necessary and sufficient
for a given learning rule to achieve low classification error rate?  This quantity is known as the \emph{sample complexity},
and varies depending on how small the desired classification error rate is, the type of classifier we are attempting 
to learn, and various other factors.  Equivalently, the question is: How small of an error rate can we guarantee 
a given learning rule will achieve, for a given number of labeled training examples?

A particularly simple setting for supervised learning is the \emph{realizable case}, in which it is assumed that,
within a given set $\C$ of classifiers, there resides some classifier that is \emph{always} correct.
The optimal sample complexity of learning in the realizable case has recently been completely resolved,
up to constant factors, in a sibling paper to the present article \citep*{hanneke:16a}.  However, there remains the
important task of identifying interesting general families of algorithms achieving this optimal sample complexity.
For instance, the best known general upper bounds for the general family of \emph{empirical risk minimization} algorithms
differ from the optimal sample complexity by a logarithmic factor, and it is known that there exist spaces $\C$
for which this is unavoidable \citep*{auer:07}.  This same logarithmic factor gap appears in the analysis of several other learning methods as well.
The present article focuses on this logarithmic factor, arguing that for certain types of learning rules, it can be entirely removed in some cases, 
and for others it can be somewhat refined.  The technique leading to these results is rooted in an idea introduced in the author's
doctoral dissertation \citep*{hanneke:thesis}.  By further exploring this technique, we also obtain new
results for the related problem of \emph{active learning}.  We also derive interesting new results for learning 
with classification noise, where again the focus is on a logarithmic factor gap between upper and lower bounds.

\subsection{Basic Notation}

Before further discussing the results, we first introduce some essential notation.
Let $\X$ be any nonempty set, called the \emph{instance space}, 
equipped with a $\sigma$-algebra defining the measurable sets;
for simplicity, we will suppose the sets in $\{\{x\} : x \in \X\}$ are all measurable.
Let $\Y = \{-1,+1\}$ be the label space.
A \emph{classifier} is any measurable function $h : \X \to \Y$.
Following \citet*{vapnik:71}, define the VC dimension of a set $\A$ of subsets of $\X$, denoted $\vc(\A)$, as 
the maximum cardinality $|S|$ over subsets $S \subseteq \X$ such that $\{ S \cap A : A \in \A \} = 2^{S}$ (the power set of $S$);
if no such maximum cardinality exists, define $\vc(\A) = \infty$.
For any set $\H$ of classifiers, 
denote by $\vc(\H) = \vc( \{ \{x : h(x)=+1\} : h \in \H \} )$ the VC dimension of $\H$. 
Throughout, we fix a set $\C$ of classifiers, known as the \emph{concept space},
and abbreviate $\dim = \vc(\C)$.  To focus on nontrivial cases, 
throughout we suppose $|\C| \geq 3$, which implies $\dim \geq 1$.  We will also generally 
suppose $\dim < \infty$ (though some of the results would still 
hold without this restriction).

For any $L_{m} = \{(x_{1},y_{1}),\ldots,(x_{m},y_{m})\} \in (\X \times \Y)^{m}$, 
and any classifier $h$, define $\er_{L_{m}}(h) = \frac{1}{m} \sum_{(x,y) \in L_{m}} \ind[ h(x) \neq y ]$.
For completeness, also define $\er_{\{\}}(h) = 0$.
Also, for any set $\H$ of classifiers, denote $\H[L_{m}] = \{ h \in \H : \forall (x,y) \in L_{m}, h(x) = y \}$, 
referred to as the set of classifiers in $\H$ \emph{consistent} with $L_{m}$; for completeness, also define $\H[\{\}] = \H$.
Fix an arbitrary probability measure $\Px$ on $\X$ (called the \emph{data distribution}), 
and a classifier $\target \in \C$ (called the \emph{target function}).
For any classifier $h$, denote $\er(h) = \Px(x : h(x) \neq \target(x))$, the \emph{error rate} of $h$.
Let $X_{1},X_{2},\ldots$ be independent $\Px$-distributed random variables.
We generally denote $\L_{m} = \{ (X_{1},\target(X_{1})),\ldots,(X_{m},\target(X_{m}))\}$,
and $V_{m} = \C[\L_{m}]$ (called the \emph{version space}).  The general setting
in which we are interested in producing a classifier $\hat{h}$ with small $\er(\hat{h})$,
given access to the data $\L_{m}$, is a special case of supervised learning known as the \emph{realizable case}
(in contrast to settings where the observed labeling
might not be realizable by any classifier in $\C$, due to label noise or model misspecification,
as discussed in Section~\ref{sec:noise}).

We adopt a few convenient notational conventions.
For any $m \in \nats$, denote $[m] = \{1,\ldots,m\}$; also denote $[0] = \{\}$.
We adopt a shorthand notation for sequences, so that for a sequence $x_{1},\ldots,x_{m}$,
we denote $x_{[m]} = (x_{1},\ldots,x_{m})$.
For any $\reals$-valued functions $f,g$, we write $f(z) \lesssim g(z)$ or $g(z) \gtrsim f(z)$ 
if there exists a finite numerical constant $c > 0$ such that $f(z) \leq c g(z)$ for all $z$.
For any $x,y \in \reals$, denote $x \lor y = \max\{x,y\}$ and $x \land y = \min\{x,y\}$.
For $x \geq 0$, denote $\Log(x) = \ln(x \lor e)$ and $\Log_{2}(x) = \log_{2}(x \lor 2)$.
We also adopt the conventions that for $x > 0$, $x / 0 = \infty$, 
and $0 \Log(x/0) = 0 \Log(\infty) = 0 \cdot \infty = 0$.
It will also be convenient to use the notation $\Z^{0} = \{()\}$ for a set $\Z$, where $()$ is the empty sequence.
Throughout, we also make the usual implicit assumption that all 
quantities required to be measurable in the proofs and lemmas from the literature
are indeed measurable.  See, for instance, \citet*{van-der-Vaart:96,van-der-Vaart:11}, for discussions
of conditions on $\C$ that typically suffice for this.

\subsection{Background and Summary of the Main Results}

This work concerns the study of the error rates achieved by various \emph{learning rules}: that is, 
mappings from the data set $\L_{m}$ to a classifier $\hat{h}_{m}$; for simplicity, we sometimes refer
to $\hat{h}_{m}$ itself as a learning rule, leaving dependence on $\L_{m}$ implicit.
There has been a substantial amount of work on bounding the error rates of various learning rules
in the realizable case.  Perhaps the most basic and natural type of learning rule in this setting is
the family of \emph{consistent} learning rules: that is, those that choose $\hat{h}_{m} \in V_{m}$.
There is a general upper bound for all consistent learning rules $\hat{h}_{m}$, due to \citet*{vapnik:74,blumer:89},
stating that with probability at least $1-\conf$,
\begin{equation}
\label{eqn:pac-upper-bound}
\er\left(\hat{h}_{m}\right) \lesssim \frac{1}{m} \left( \dim \Log\left(\frac{m}{\dim}\right) + \Log\left(\frac{1}{\conf}\right) \right).
\end{equation}
This is complemented by a general lower bound of \citet*{ehrenfeucht:89}, which states that for any learning rule (consistent or otherwise),
there exists a choice of $\Px$ and $\target \in \C$ such that, with probability greater than $\conf$,
\begin{equation}
\label{eqn:pac-lower-bound}
\er\left(\hat{h}_{m}\right) \gtrsim \frac{1}{m} \left( \dim + \Log\left(\frac{1}{\conf}\right) \right).
\end{equation}
Resolving the logarithmic factor gap between \eqref{eqn:pac-lower-bound} and \eqref{eqn:pac-upper-bound}
has been a challenging subject of study for decades now, with many interesting contributions resolving special 
cases and proposing sometimes-better upper bounds \citep*[e.g.,][]{haussler:94,gine:06,auer:07,long:03}.  
It is known that the lower bound is sometimes \emph{not} achieved by certain consistent learning rules \citep*{auer:07}.  
The question of whether the lower bound \eqref{eqn:pac-lower-bound} can 
always be achieved by \emph{some} algorithm remained open for a number of years \citep*{ehrenfeucht:89,warmuth:04},
but has recently been resolved in a sibling paper to the present article \citep*{hanneke:16a}.
That work proposes a learning rule $\hat{h}_{m}$ based on a majority vote of classifiers consistent with 
carefully-constructed subsamples of the data, and proves that with probability at least $1-\conf$,
\begin{equation*}
\er\left(\hat{h}_{m}\right) \lesssim \frac{1}{m} \left( \dim + \Log\left(\frac{1}{\conf}\right) \right).
\end{equation*} 
However, several avenues for investigation remain open, including identifying interesting general families of learning rules
able to achieve this optimal bound under general conditions on $\C$.  In particular, it remains an open problem 
to determine necessary and sufficient conditions on $\C$ for the entire family of consistent learning rules
to achieve the above optimal error bound.

The work of \citet*{gine:06} includes a bound that refines the logarithmic factor in \eqref{eqn:pac-upper-bound} in certain scenarios.
Specifically, it states that, for any consistent learning rule $\hat{h}_{m}$, with probability at least $1-\conf$, 
\begin{equation}
\label{eqn:gk-dc-intro}
\er\left(\hat{h}_{m}\right) \lesssim \frac{1}{m} \left( \dim \Log\left(\dc\left(\frac{\dim}{m}\right)\right) + \Log\left(\frac{1}{\conf}\right) \right),
\end{equation}
where $\dc(\cdot)$ is the \emph{disagreement coefficient} (defined below in Section~\ref{sec:cal}).
The doctoral dissertation of \citet*{hanneke:thesis} contains a simple and direct proof of this bound,
based on an argument which splits the data set in two parts, and considers the second part as containing
a subsequence sampled from the conditional distribution given the region of disagreement of the version
space induced by the first part of the data.
Many of the results in the present work are based on variations of this argument, 
including a variety of interesting new bounds on the error rates achieved by certain families of learning rules.

As one of the cornerstones of this work, we find that a variant of this argument for consistent learning rules
with \emph{monotonic} error regions
leads to an upper bound that \emph{matches} the lower bound 
\eqref{eqn:pac-lower-bound} up to constant factors.
For such monotonic consistent learning rules to exist, we 
would need a very special kind of concept space.  However, they do exist in some important cases.
In particular, in the special case of learning \emph{intersection-closed} concept spaces, 
the \emph{Closure} algorithm \citep*{natarajan:87,auer:04,auer:07} can be shown to satisfy this monotonicity property.
Thus, this result immediately implies that, with probability at least $1-\conf$, the Closure algorithm
achieves
\begin{equation*}
\er(\hat{h}_{m}) \lesssim \frac{1}{m}\left( \dim + \Log\left( \frac{1}{\conf} \right) \right),
\end{equation*}
which was an open problem of \citet*{auer:04,auer:07};
this fact was recently
also obtained by \citet*{darnstadt:15},
via a related direct argument.
We also discuss a variant of this result for monotone learning rules expressible as \emph{compression schemes},
where we remove a logarithmic factor present in a result of \citet*{littlestone:86} and \citet*{floyd:95}, so that for $\hat{h}_{m}$
based on a compression scheme of size $\cs$, which has monotonic error regions (and is permutation-invariant), with probability at least 
$1-\conf$, 
\begin{equation*}
\er(\hat{h}_{m}) \lesssim \frac{1}{m}\left( \cs + \Log\left( \frac{1}{\conf} \right) \right).
\end{equation*}

This argument also has implications for \emph{active learning}.
In many active learning algorithms, the \emph{region of disagreement} of the version space induced by $m$ samples, 
$\DIS(V_{m}) = \{ x \in \X : \exists h,g \in V_{m} \text{ s.t. } h(x) \neq g(x) \}$,
plays an important role.  In particular, the label complexity of the CAL active learning algorithm \citep*{cohn:94} is 
largely determined by the rate at which $\Px(\DIS(V_{m}))$ decreases, so that any bound on this quantity 
can be directly converted into a bound on the label complexity of CAL \citep*{hanneke:11a,hanneke:thesis,hanneke:fntml,el-yaniv:12}.
\citet*{hanneke:15a} have argued that the region $\DIS(V_{m})$ can be described
as a compression scheme, where the size of the compression scheme, denoted $\hat{\cs}_{m}$, is known as the \emph{version space compression set size}
(Definition~\ref{def:td-hat} below).  By further observing that $\DIS(V_{m})$ is monotonic in $m$, 
applying our general argument yields the fact that, with probability at least $1-\conf$, letting $\hat{\cs}_{1:m} = \max_{t \in [m]} \hat{\cs}_{t}$,
\begin{equation}
\label{eqn:PDIS-intro}
\Px(\DIS(V_{m})) \lesssim \frac{1}{m} \left( \hat{\cs}_{1:m} + \Log\left(\frac{1}{\conf}\right) \right),
\end{equation}
which is typically an improvement over the best previously-known general bound by a logarithmic factor.

In studying the distribution-free minimax label complexity of active learning, 
\citet*{hanneke:15b} found that a simple combinatorial quantity $\s$, which they term the \emph{star number},
is of fundamental importance.  Specifically (see also Definition~\ref{def:star}), $\s$ is the largest number $s$ of distinct points $x_{1},\ldots,x_{s} \in \X$
such that $\exists h_{0},h_{1},\ldots,h_{s} \in \C$ with $\forall i \in [s]$, $\DIS(\{h_{0},h_{i}\}) \cap \{x_{1},\ldots,x_{s}\} = \{x_{i}\}$,
or else $\s = \infty$ if no such largest $s$ exists.
Interestingly, the work of \citet*{hanneke:15b} also establishes that the largest possible value of $\hat{\cs}_{m}$ (over $m$ and the data set) is exactly $\s$.
Thus, \eqref{eqn:PDIS-intro} also implies a \emph{data-independent} and \emph{distribution-free} bound: with probability at least $1-\conf$, 
\begin{equation*}
\Px(\DIS(V_{m})) \lesssim \frac{1}{m} \left( \s + \Log\left(\frac{1}{\conf}\right) \right).
\end{equation*}

Now one interesting observation at this point is that the direct proof of \eqref{eqn:gk-dc-intro} from \citet*{hanneke:thesis}
involves a step in which $\Px(\DIS(V_{m}))$ is relaxed to a bound in terms of $\dc(\dim/m)$.  If we instead use
\eqref{eqn:PDIS-intro} in this step, we arrive at a new bound on the error rates of \emph{all} consistent learning rules $\hat{h}_{m}$: 
with probability at least $1-\conf$, 
\begin{equation}
\label{eqn:erm-intro}
\er(\hat{h}_{m}) \lesssim \frac{1}{m}\left( \dim \Log\left( \frac{\hat{\cs}_{1:m}}{\dim} \right) + \Log\left(\frac{1}{\conf}\right) \right).
\end{equation}
Since \citet*{hanneke:15b} have shown that the maximum possible value of $\dc(\dim/m)$ (over $m$, $\Px$, and $\target$)
is also exactly the star number $\s$, while $\hat{\cs}_{1:m} / \dim$ has as its maximum possible value $\s/\dim$, we see that the bound in 
\eqref{eqn:erm-intro} sometimes reflects an improvement over \eqref{eqn:gk-dc-intro}.
It further implies a new data-independent and distribution-free bound for any consistent learning rule $\hat{h}_{m}$: with probability at least $1-\conf$,
\begin{equation*}
\er(\hat{h}_{m}) \lesssim \frac{1}{m}\left( \dim \Log\left( \frac{\min\{\s, m\}}{\dim} \right) + \Log\left(\frac{1}{\conf}\right) \right).
\end{equation*}
Interestingly, we are able to complement this with a \emph{lower bound} in Section~\ref{sec:lower-bounds}.
Though not quite matching the above in terms of its joint dependence on $\dim$ and $\s$ (and necessarily so), 
this lower bound does provide the interesting observation that $\s < \infty$ is \emph{necessary and sufficient} for 
there to exist a distribution-free bound on the error rates of all consistent learning rules, 
converging at a rate $\Theta(1/m)$, and otherwise (when $\s=\infty$) the best such bound is $\Theta(\Log(m)/m)$.

Continuing with the investigation of general consistent learning rules, we also find a variant of the argument of \citet*{hanneke:thesis}
that refines \eqref{eqn:gk-dc-intro} in a different way: namely, replacing $\dc(\cdot)$ with a quantity based on 
considering a well-chosen \emph{subregion} of the region of disagreement, as studied by \citet*{balcan:07,zhang:14}.
Specifically, in the context of active learning, \citet*{zhang:14} have proposed a general quantity $\zc_{c}(\cdot)$ (Definition~\ref{def:zc} below), 
which is never larger than $\dc(\cdot)$, and is sometimes significantly smaller.  By adapting our general argument 
to replace $\DIS(V_{m})$ with this well-chosen subregion, we derive a bound for all consistent learning rules $\hat{h}_{m}$:
with probability at least $1-\conf$,
\begin{equation*}
\er(\hat{h}_{m}) \lesssim \frac{1}{m}\left( \dim \Log\left( \zc_{c}\left( \frac{\dim}{m} \right) \right) + \Log\left(\frac{1}{\conf}\right) \right).
\end{equation*}
In particular, as a special case of this general result, we recover the theorem of \citet*{balcan:13} that all consistent learning rules have optimal 
sample complexity (up to constants) for the problem of learning homogeneous linear separators under isotropic log-concave distributions,
as $\zc_{c}(\dim/m)$ is bounded by a finite numerical constant in this case.
In Section~\ref{sec:noise}, we also extend this result to the problem of learning with \emph{classification noise},
where there is also a logarithmic factor gap between the known general-case upper and lower bounds.
In this context, we derive a new general upper bound under the Bernstein class condition (a generalization of Tsybakov's noise condition),
expressed in terms of a quantity related to $\zc_{c}(\cdot)$, which applies to a particular learning rule.  
This sometimes reflects an improvement over the best previous general upper bounds \citep*{massart:06,gine:06,hanneke:12b},
and again recovers a result of \citet*{balcan:13} for homogeneous linear separators under isotropic log-concave distributions,
as a special case.

For many of these results, we also state bounds on the \emph{expected} error rate: $\E\left[ \er(\hat{h}_{m}) \right]$.
In this case, the optimal distribution-free bound is known to be within a constant factor of $\dim/m$ \citep*{haussler:94,li:01},
and this rate is achieved by the one-inclusion graph prediction algorithm of \citet*{haussler:94}, as well as the majority voting method of \citet*{hanneke:16a}.
However, there remain interesting questions about whether other algorithms achieve this optimal performance, or require an extra logarithmic factor.
Again we find that \emph{monotone} consistent learning rules indeed achieve this optimal $\dim/m$ rate (up to constant factors),
while a distribution-free bound on $\E\left[ \er(\hat{h}_{m}) \right]$ with $\Theta(1/m)$ dependence on $m$
is achieved by all consistent learning rules if and only if $\s < \infty$, and otherwise the best such bound has 
$\Theta(\Log(m)/m)$ dependence on $m$.

As a final interesting result, in the context of learning with classification noise, under the \emph{bounded} noise assumption \citep*{massart:06},
we find that the condition $\s < \infty$ is actually \emph{necessary and sufficient} for the \emph{minimax optimal} excess error rate to 
decrease at a rate $\Theta(1/m)$, and otherwise (if $\s=\infty$) it decreases at a rate $\Theta(\Log(m)/m)$.
This result generalizes several special-case analyses from the literature \citep*{massart:06,raginsky:11}.
Note that the ``necessity'' part of this statement is significantly stronger
than the above result for consistent learning rules in the realizable case,
since this result applies to the best error guarantee achievable by \emph{any} learning rule.

\section{Bounds for Consistent Monotone Learning}
\label{sec:abstract-bounds}

In order to state our results for monotonic learning rules in an abstract form, we introduce the following notation.
Let $\Z$ denote any space, equipped with a $\sigma$-algebra
defining the measurable subsets.
For any collection $\A$ of measurable subsets of $\Z$,
a \emph{consistent monotone rule} is any 
sequence of functions $\psi_{t} : \Z^{t} \to \A$, $t \in \nats$,
such that $\forall z_{1},z_{2},\ldots \in \Z$, $\forall t \in \nats$, 
$\psi_{t}(z_{1},\ldots,z_{t}) \cap \{z_{1},\ldots,z_{t}\} = \emptyset$,
and $\forall t \in \nats$, $\psi_{t+1}(z_{1},\ldots,z_{t+1}) \subseteq \psi_{t}(z_{1},\ldots,z_{t})$.
We begin with the following very simple result, the proof of which 
will also serve to introduce, in its simplest form, the core technique 
underlying many of the results presented in later sections below.

\begin{theorem}
\label{thm:monotone-erm}
Let $\A$ be a collection of measurable subsets of $\Z$,
and let $\psi_{t} : \Z^{t} \to \A$ (for $t \in \nats$) be any consistent monotone rule.
Fix any $m \in \nats$, any $\conf \in (0,1)$, 
and any probability measure $P$ on $\Z$.
Letting $Z_{1},\ldots,Z_{m}$ be independent $P$-distributed
random variables, and denoting $A_{m} = \psi_{m}(Z_{1},\ldots,Z_{m})$, 
with probability at least $1-\conf$,
\begin{equation}
\label{eqn:monotone-erm-eps-conf}
P(A_{m}) \leq \frac{4}{m} \left( 17 \vc(\A) + 4 \ln\left(\frac{4}{\conf}\right) \right).
\end{equation}
Furthermore, 
\begin{equation}
\label{eqn:monotone-erm-expectation}
\E[P(A_{m})] \leq \frac{68 (\vc(\A)+1)}{m}.
\end{equation}
\end{theorem}

The overall structure of this proof is based on an argument of \citet*{hanneke:thesis}.
The most-significant novel element here is the use of monotonicity to further refine a logarithmic factor.
The proof relies on the following classic result.  Results of this type are originally due to \citet*{vapnik:74};
the version stated here features slightly better constant factors, due to \citet*{blumer:89}.

\begin{lemma}
\label{lem:classic-erm}
For any collection $\A$ of measurable subsets of $\Z$,
any $\conf \in (0,1)$, any $m \in \nats$, and any probability measure $P$ on $\Z$, 
letting $Z_{1},\ldots,Z_{m}$ be independent $P$-distributed random variables,
with probability at least $1-\conf$, every $A \in \A$ with 
$A \cap \{Z_{1},\ldots,Z_{m}\} = \emptyset$ satisfies
\begin{equation*}
P(A) \leq \frac{2}{m} \left( \vc(\A) \Log_{2}\left(\frac{2 e m}{\vc(\A)}\right) + \Log_{2}\left(\frac{2}{\conf}\right) \right).
\end{equation*}
\end{lemma}

We are now ready for the proof of Theorem~\ref{thm:monotone-erm}.

\begin{proof}[of Theorem~\ref{thm:monotone-erm}]
Fix any probability measure $P$, let $Z_{1},Z_{2},\ldots$ be independent $P$-distributed random variables, 
and for each $m \in \nats$ denote $A_{m} = \psi_{m}(Z_{1},\ldots,Z_{m})$.
We begin with the inequality in \eqref{eqn:monotone-erm-eps-conf}.
The proof proceeds by induction on $m$.
If $m \leq 200$, then since $\log_{2}(400 e) < 34$ and $\log_{2}\left(\frac{2}{\conf}\right) < 8 \ln\left(\frac{4}{\conf}\right)$, 
and since the definition of a consistent monotone rule implies $A_{m} \cap \{Z_{1},\ldots,Z_{m}\} = \emptyset$,
the stated bound follows immediately from Lemma~\ref{lem:classic-erm} for any $\conf \in (0,1)$.
Now, as an inductive hypothesis, fix any integer $m \geq 201$ such that,
$\forall m^{\prime} \in [m-1]$, $\forall \conf \in (0,1)$, 
with probability at least $1-\conf$, 
\begin{equation*}
P(A_{m^{\prime}}) \leq \frac{4}{m^{\prime}} \left( 17 \vc(\A) + 4 \ln\left(\frac{4}{\conf}\right) \right).
\end{equation*}

Now fix any $\conf \in (0,1)$
and define
\begin{equation*}
N = \left| \left\{ Z_{\lfloor m/2 \rfloor + 1},\ldots,Z_{m} \right\} \cap A_{\lfloor m/2 \rfloor} \right|,
\end{equation*}
and enumerate the elements of $\{ Z_{\lfloor m/2 \rfloor + 1},\ldots,Z_{m} \} \cap A_{\lfloor m/2 \rfloor}$
as $\hat{Z}_{1},\ldots,\hat{Z}_{N}$ (retaining their original order).

Note that $N = \sum_{t=\lfloor m/2 \rfloor + 1}^{m} \ind_{A_{\lfloor m/2 \rfloor}}(Z_{t})$ 
is conditionally ${\rm Binomial}(\lceil m/2 \rceil, P(A_{\lfloor m/2 \rfloor}))$-distributed
given $Z_{1},\ldots,Z_{\lfloor m/2 \rfloor}$.
In particular, with probability one, if $P(A_{\lfloor m/2 \rfloor}) = 0$, then $N = 0$.
Otherwise, if $P(A_{\lfloor m/2 \rfloor}) > 0$, then note that
$\hat{Z}_{1},\ldots,\hat{Z}_{N}$ are conditionally independent and $P(\cdot | A_{\lfloor m/2 \rfloor})$-distributed
given $Z_{1},\ldots,Z_{\lfloor m/2 \rfloor}$ and $N$.
Thus, since $A_{m} \cap \{\hat{Z}_{1},\ldots,\hat{Z}_{N}\} \subseteq A_{m} \cap \{Z_{1},\ldots,Z_{m}\} = \emptyset$,
applying Lemma~\ref{lem:classic-erm} (under the conditional distribution given $N$ and $Z_{1},\ldots,Z_{\lfloor m/2 \rfloor}$), combined with the 
law of total probability, we have that on an event $E_{1}$ of probability at least $1-\conf/2$, if $N > 0$, then
\begin{equation*}
P(A_{m} | A_{\lfloor m/2 \rfloor}) \leq \frac{2}{N} \left( \vc(\A) \Log_{2}\left(\frac{2 e N}{\vc(\A)}\right) + \log_{2}\left(\frac{4}{\conf}\right) \right).
\end{equation*}

Additionally, again since $N$ is conditionally ${\rm Binomial}(\lceil m/2 \rceil, P(A_{\lfloor m/2 \rfloor}))$-distributed
given $Z_{1},\ldots,Z_{\lfloor m/2 \rfloor}$,
applying a Chernoff bound (under the conditional distribution given $Z_{1},\ldots,Z_{\lfloor m/2 \rfloor}$),
combined with the law of total probability, we obtain that on an event $E_{2}$ of probability at least $1 - \conf/4$, 
if $P(A_{\lfloor m/2 \rfloor}) \geq \frac{16}{m} \ln\left(\frac{4}{\conf}\right)$,
then
\begin{equation*}
N \geq P(A_{\lfloor m/2 \rfloor}) \lceil m/2 \rceil / 2 
\geq P(A_{\lfloor m/2 \rfloor}) m / 4.
\end{equation*}
In particular, if $P(A_{\lfloor m/2 \rfloor}) \geq \frac{16}{m} \ln\left(\frac{4}{\conf}\right)$, 
then $P(A_{\lfloor m/2 \rfloor}) m / 4 > 0$, so that if this occurs with $E_{2}$, then we have $N > 0$.
Noting that $\Log_{2}(x) \leq \Log(x)/\ln(2)$, then by monotonicity of $x \mapsto \Log(x)/x$ for $x > 0$, 
we have that on $E_{1} \cap E_{2}$, if $P(A_{\lfloor m/2 \rfloor}) \geq \frac{16}{m} \ln\left(\frac{4}{\conf}\right)$, 
then
\begin{equation*}
P(A_{m} | A_{\lfloor m/2 \rfloor}) \leq \frac{8}{P(A_{\lfloor m/2 \rfloor}) m \ln(2)} \left( \vc(\A) \Log\left(\frac{e P(A_{\lfloor m/2 \rfloor}) m}{2 \vc(\A)}\right) + \ln\left(\frac{4}{\conf}\right) \right).
\end{equation*}
The monotonicity property of $\psi_{t}$ implies $A_{m} \subseteq A_{\lfloor m/2 \rfloor}$.
Together with monotonicity of probability measures, this implies $P(A_{m}) \leq P(A_{\lfloor m/2 \rfloor})$.
It also implies that, if $P(A_{\lfloor m/2 \rfloor}) > 0$, then $P(A_{m}) = P(A_{m} | A_{\lfloor m/2 \rfloor}) P(A_{\lfloor m/2 \rfloor})$.
Thus, on $E_{1} \cap E_{2}$, if
$P(A_{m}) \geq \frac{16}{m} \ln\left(\frac{4}{\conf}\right)$,
then
\begin{equation*}
P(A_{m}) \leq \frac{8}{m \ln(2)} \left( \vc(\A) \Log\left(\frac{e P(A_{\lfloor m/2 \rfloor}) m}{2 \vc(\A)}\right) + \ln\left(\frac{4}{\conf}\right) \right).
\end{equation*}

The inductive hypothesis implies that, on an event $E_{3}$ of probability at least $1 - \conf/4$, 
\begin{equation*}
P(A_{\lfloor m/2 \rfloor}) 
\leq \frac{4}{\lfloor m/2 \rfloor} \left( 17 \vc(\A) + 4 \ln\left(\frac{16}{\conf}\right) \right).
\end{equation*}
Since $m \geq 201$, we have $\lfloor m/2 \rfloor \geq (m-2)/2 \geq (199/402)m$, so that the above implies
\begin{equation*}
P(A_{\lfloor m/2 \rfloor}) 
\leq \frac{4 \cdot 402}{199 m} \left( 17 \vc(\A) + 4 \ln\left(\frac{16}{\conf}\right) \right). 
\end{equation*}
Thus, on $E_{1} \cap E_{2} \cap E_{3}$, if $P(A_{m}) \geq \frac{16}{m} \ln\left(\frac{4}{\conf}\right)$, then
\begin{equation*}
P(A_{m}) \leq \frac{8}{m \ln(2)} \left( \vc(\A) \Log\left(\frac{ 2 \cdot 402 e}{199} \left( 17 + \frac{4}{\vc(\A)} \ln\left(\frac{16}{\conf}\right) \right) \right) + \ln\left(\frac{4}{\conf}\right) \right).
\end{equation*}
Lemma~\ref{lem:log-factors-abstract} in Appendix~\ref{app:technical-lemmas} allows us to simplify the logarithmic term here,
revealing that the right hand side is at most
\begin{align*}
& \frac{8}{m \ln(2)} \left( \vc(\A) \Log\left(\frac{2 \cdot 402 e}{199} \left( 17 + 4\ln(4) + \frac{4}{\ln(4/e)} \right) \right) + \left(1 + \ln\left(\frac{4}{e}\right)\right)\ln\left(\frac{4}{\conf}\right) \right)
\\ & \leq \frac{4}{m} \left( 17 \vc(\A) + 4 \ln\left(\frac{4}{\conf}\right) \right).
\end{align*}
Since $\frac{16}{m} \ln\left(\frac{4}{\conf}\right) \leq \frac{4}{m} \left( 17 \vc(\A) + 4 \ln\left(\frac{4}{\conf}\right) \right)$,
we have that, on $E_{1} \cap E_{2} \cap E_{3}$, regardless of whether or not $P(A_{m}) \geq \frac{16}{m} \ln\left(\frac{4}{\conf}\right)$, we have
\begin{equation*}
P(A_{m}) \leq \frac{4}{m} \left( 17 \vc(\A) + 4 \ln\left(\frac{4}{\conf}\right) \right).
\end{equation*}
Noting that, by the union bound, the event $E_{1} \cap E_{2} \cap E_{3}$ has probability at least $1-\conf$,
this extends the inductive hypothesis to $m^{\prime} = m$.  By the principle of induction, this completes the proof of the 
first claim in Theorem~\ref{thm:monotone-erm}.

For the bound on the expectation in \eqref{eqn:monotone-erm-expectation}, 
we note that, letting $\eps_{m} \!=\! \frac{4}{m} \!\left( 17 \vc(\A) + 4 \ln(4) \right)$, 
by setting the bound in \eqref{eqn:monotone-erm-eps-conf} equal to a value $\eps$ and solving for $\conf$,
the value of which is in $(0,1)$ for any $\eps > \eps_{m}$,
the result just established can be restated as:
$\forall \eps > \eps_{m}$, 
\begin{equation*}
\P\left( P(A_{m}) > \eps \right) \leq 4 \exp\left\{ (17/4) \vc(\A) - \eps m / 16 \right\}.
\end{equation*}
Furthermore, for any $\eps \leq \eps_{m}$, we of course
still have $\P\left( P(A_{m}) > \eps \right) \leq 1$. 
Therefore, we have that
\begin{align*}
\E\left[ P(A_{m}) \right]
& = \int_{0}^{\infty} \P\left( P(A_{m}) > \eps \right) {\rm d}\eps
\leq \eps_{m} + \int_{\eps_{m}}^{\infty} 4 \exp\left\{ (17/4) \vc(\A) - \eps m / 16 \right\} {\rm d}\eps
\\ & = \eps_{m} + \frac{4 \cdot 16}{m} \exp\left\{ (17/4) \vc(\A) - \eps_{m} m / 16 \right\}
= \frac{4}{m} \left( 17 \vc(\A) + 4\ln(4) \right) + \frac{16}{m}
\\ & = \frac{4}{m} \left( 17 \vc(\A) + 4\ln(4e) \right)
\leq \frac{68 \vc(\A)+39}{m}
\leq \frac{68 (\vc(\A)+1)}{m}.
\end{align*}
\end{proof}

We can also state a variant of Theorem~\ref{thm:monotone-erm} applicable to \emph{sample compression schemes},
which will in fact be more useful for our purposes below.  To state this result,
we first introduce the following additional terminology.
For any $t \in \nats$, we say that a function $\psi : \Z^{t} \to \A$ is 
\emph{permutation-invariant} if every $z_{1},\ldots,z_{t} \in \Z$ and every 
bijection $\kappa : [t] \to [t]$ satisfy $\psi(z_{\kappa(1)},\ldots,z_{\kappa(t)}) = \psi(z_{1},\ldots,z_{t})$.
For
any $\cs \in \nats \cup \{0\}$, a 
\emph{consistent monotone sample compression rule of size $\cs$} is a 
consistent monotone rule $\psi_{t}$ with the additional properties that, 
$\forall t \in \nats$, $\psi_{t}$ is permutation-invariant,
and $\forall z_{1},\ldots,z_{t} \in \Z$, $\exists \empcs_{t}(z_{[t]}) \in [\min\{\cs,t\}] \cup \{0\}$
such that 
\begin{equation*}
\psi_{t}(z_{1},\ldots,z_{t}) = \phi_{t,\empcs_{t}(z_{[t]})}( z_{i_{t,1}(z_{[t]})},\ldots,z_{i_{t,\empcs_{t}(z_{[t]})}(z_{[t]})} ),
\end{equation*}
where $\phi_{t,k} : \Z^{k} \to \A$ is a permutation-invariant function for each $k \in [\min\{\cs,t\}] \cup \{0\}$,
and $i_{t,1},\ldots,i_{t,\cs}$ are functions $\Z^{t} \to [t]$
such that $\forall z_{1},\ldots,z_{t} \in \Z$, 
$i_{t,1}(z_{[t]}),\ldots,i_{t,\empcs_{t}(z_{[t]})}(z_{[t]})$ are all distinct.
In words, the element of $\A$ mapped to by $\psi_{t}(z_{1},\ldots,z_{t})$
depends only on the unordered (multi)set $\{z_{1},\ldots,z_{t}\}$, 
and can be specified by an unordered subset of $\{z_{1},\ldots,z_{t}\}$
of size at most $\cs$.
Following the terminology from the literature on sample compression schemes,
we refer to the collection of functions $\{ (\empcs_{t},i_{t,1},\ldots,i_{t,\empcs_{t}}) : t \in \nats \}$
as the \emph{compression function} of $\psi_{t}$, and to the collection of permutation-invariant functions $\{\phi_{t,k} : t \in \nats, k \in [\min\{\cs,t\}] \cup \{0\}\}$
as the \emph{reconstruction function} of $\psi_{t}$.

This kind of $\psi_{t}$ is a type of sample compression scheme \citep*[see][]{littlestone:86,floyd:95},
though certainly not all permutation-invariant compression schemes yield consistent monotone rules.
Below, we find that consistent monotone sample compression rules of a quantifiable size
arise naturally in the analysis of certain learning algorithms
(namely, the Closure algorithm and the CAL active learning algorithm).

With the above terminology in hand, we can now state our second abstract result.

\begin{theorem}
\label{thm:monotone-compression}
Fix any $\cs \in \nats \cup \{0\}$, 
let $\A$ be a collection of measurable subsets of $\Z$,
and let $\psi_{t} : \Z^{t} \to \A$ (for $t \in \nats$) be any consistent monotone sample compression rule of size $\cs$.
Fix any $m \in \nats$, $\conf \in (0,1)$,
and any probability measure $P$ on $\Z$.
Letting $Z_{1},\ldots,Z_{m}$ be independent $P$-distributed
random variables, and denoting $A_{m} = \psi_{m}(Z_{1},\ldots,Z_{m})$, 
with probability at least $1-\conf$, 
\begin{equation}
\label{eqn:monotone-compression-eps-conf}
P(A_{m}) \leq \frac{1}{m}\left( 21 \cs + 16 \ln\left(\frac{3}{\conf}\right) \right).
\end{equation}
Furthermore, 
\begin{equation}
\label{eqn:monotone-compression-expectation}
\E[P(A_{m})] \leq \frac{21 \cs + 34}{m}.
\end{equation}
\end{theorem}

The proof of Theorem~\ref{thm:monotone-compression} relies on the following 
classic result due to \citet*{littlestone:86,floyd:95}
(see also \citealp*{herbrich:02,hanneke:15a}, for a clear and direct proof).

\begin{lemma}
\label{lem:classic-compression}
Fix any collection $\A$ of measurable subsets of $\Z$,
any $m \in \nats$ and $\cs \in \nats \cup \{0\}$ with $\cs < m$,
and any permutation-invariant functions $\phi_{k} : \Z^{k} \to \A$, $k \in [\cs] \cup \{0\}$.
For any probability measure $P$ on $\Z$,
letting $Z_{1},\ldots,Z_{m}$ be independent $P$-distributed random variables,
for any $\conf \in (0,1)$, with probability at least $1-\conf$, 
every $k \in [\cs] \cup \{0\}$, and every distinct $i_{1},\ldots,i_{k} \in [m]$ with 
$\phi_{k}(Z_{i_{1}},\ldots,Z_{i_{k}}) \cap \{Z_{1},\ldots,Z_{m}\} = \emptyset$ satisfy
\begin{equation*}
P\left( \phi_{k}(Z_{i_{1}},\ldots,Z_{i_{k}}) \right) \leq \frac{1}{m-\cs} \left( \cs \Log\left(\frac{e m}{\cs}\right) + \Log\left(\frac{1}{\conf}\right) \right).
\end{equation*}
\end{lemma}

With this lemma in hand, we are ready for the proof of Theorem~\ref{thm:monotone-compression}.

\begin{proof}[of Theorem~\ref{thm:monotone-compression}]
The proof follows analogously to that of Theorem~\ref{thm:monotone-erm},
but with several additional complications due to the form of
Lemma~\ref{lem:classic-compression} being somewhat different
from that of Lemma~\ref{lem:classic-erm}.
Let $\{ (\empcs_{t},i_{t,1},\ldots,i_{t,\empcs_{t}}) : t \in \nats \}$ and $\{\phi_{t,k} : t \in \nats, k \in [\min\{\cs,t\}] \cup \{0\} \}$
be the compression function and reconstruction function of $\psi_{t}$, respectively.
For convenience, also denote $\psi_{0}() = \Z$, and note that this extends the monotonicity
property of $\psi_{t}$ to $t \in \nats \cup \{0\}$.
Fix any probability measure $P$,
let $Z_{1},Z_{2},\ldots$ be independent $P$-distributed random variables, 
and for each $m \in \nats$ denote $A_{m} = \psi_{m}(Z_{1},\ldots,Z_{m})$.

We begin with the inequality in \eqref{eqn:monotone-compression-eps-conf}.
The special case of $\cs = 0$ is directly implied by Lemma~\ref{lem:classic-compression},
so for the remainder of the proof of \eqref{eqn:monotone-compression-eps-conf}, we suppose $\cs \geq 1$.
The proof proceeds by induction on $m$.
Since $P(A) \leq 1$ for all $A \in \A$,
and since $21 + 16 \ln(3) > 38$,
the stated bound is trivially satisfied for all $\conf \in (0,1)$ if $m \leq \max\{38,21\cs\}$.
Now, as an inductive hypothesis, fix any integer $m > \max\{38,21\cs\}$ such that,
$\forall m^{\prime} \in [m-1]$, $\forall \conf \in (0,1)$, 
with probability at least $1-\conf$, 
\begin{equation*}
P(A_{m^{\prime}}) \leq \frac{1}{m^{\prime}} \left( 21 \cs + 16 \ln\left(\frac{3}{\conf}\right) \right).
\end{equation*}

Fix any $\conf \in (0,1)$
and define
\begin{equation*}
N = \left| \left\{ Z_{\lfloor m/2 \rfloor + 1},\ldots,Z_{m} \right\} \cap A_{\lfloor m/2 \rfloor} \right|,
\end{equation*}
and enumerate the elements of $\{ Z_{\lfloor m/2 \rfloor + 1},\ldots,Z_{m} \} \cap A_{\lfloor m/2 \rfloor}$
as $\hat{Z}_{1},\ldots,\hat{Z}_{N}$.
Also enumerate the elements of $\{Z_{\lfloor m/2 \rfloor + 1},\ldots,Z_{m}\} \setminus A_{\lfloor m/2 \rfloor}$
as $\hat{Z}_{1}^{\prime},\ldots,\hat{Z}_{\lceil m/2 \rceil - N}^{\prime}$.
Now note that, by the monotonicity property of $\psi_{t}$, we have $A_{m} \subseteq A_{\lfloor m/2 \rfloor}$.
Furthermore, by permutation-invariance of $\psi_{t}$, 
we have that 
\begin{equation*}
A_{m} = \psi_{m}\left(\hat{Z}_{1},\ldots,\hat{Z}_{N}, Z_{1},\ldots,Z_{\lfloor m/2 \rfloor}, \hat{Z}_{1}^{\prime},\ldots,\hat{Z}_{\lceil m/2 \rceil - N}^{\prime} \right).
\end{equation*}
Combined with the monotonicity property of $\psi_{t}$, this implies that
$A_{m} \subseteq \psi_{N}\left( \hat{Z}_{1},\ldots,\hat{Z}_{N} \right)$.
Altogether, we have that
\begin{equation}
\label{eqn:monotone-compression-conditional-subset}
A_{m} \subseteq A_{\lfloor m/2 \rfloor} \cap \psi_{N}\left( \hat{Z}_{1},\ldots,\hat{Z}_{N} \right).
\end{equation}

Note that $N = \sum_{t=\lfloor m/2 \rfloor + 1}^{m} \ind_{A_{\lfloor m/2 \rfloor}}(Z_{t})$ 
is conditionally ${\rm Binomial}(\lceil m/2 \rceil, P(A_{\lfloor m/2 \rfloor}))$-distributed
given $Z_{1},\ldots,Z_{\lfloor m/2 \rfloor}$.
In particular, with probability one, if $P(A_{\lfloor m/2 \rfloor}) = 0$, then $N = 0 \leq \cs$.
Otherwise, if $P(A_{\lfloor m/2 \rfloor}) > 0$, then note that
$\hat{Z}_{1},\ldots,\hat{Z}_{N}$ are conditionally independent and $P(\cdot | A_{\lfloor m/2 \rfloor})$-distributed
given $N$ and $Z_{1},\ldots,Z_{\lfloor m/2 \rfloor}$.
Since $\psi_{t}$ is a consistent monotone rule, we have that 
$\psi_{N}( \hat{Z}_{1}, \ldots, \hat{Z}_{N} ) \cap \{ \hat{Z}_{1},\ldots,\hat{Z}_{N} \} = \emptyset$.
We also have, by definition of $\psi_{N}$, that 
$\psi_{N}( \hat{Z}_{1}, \ldots, \hat{Z}_{N} ) = \phi_{N,\empcs_{N}(\hat{Z}_{[N]})}\left( \hat{Z}_{i_{N,1}(\hat{Z}_{[N]})}, \ldots, \hat{Z}_{i_{N,\empcs_{N}(\hat{Z}_{[N]})}(\hat{Z}_{[N]})} \right)$.
Thus, applying Lemma~\ref{lem:classic-compression} (under the conditional distribution given $N$ and $Z_{1},\ldots,Z_{\lfloor m/2 \rfloor}$),
combined with the law of total probability, we have that on an event $E_{1}$ of probability at least $1-\conf/3$, 
if $N > \cs$, then
\begin{equation*}
P\left( \psi_{N}\left( \hat{Z}_{1}, \ldots, \hat{Z}_{N} \right) \middle| A_{\lfloor m/2 \rfloor} \right)
\leq \frac{1}{N-\cs} \left( \cs \ln\left( \frac{e N}{\cs} \right) + \ln\left(\frac{3}{\conf}\right) \right).
\end{equation*}
Combined with \eqref{eqn:monotone-compression-conditional-subset} and monotonicity of
measures, this implies
that on $E_{1}$, if $N > \cs$, then
\begin{align*}
P(A_{m}) \leq P\!\left( A_{\lfloor m/2 \rfloor} \!\cap\! \psi_{N}\!\left( \hat{Z}_{1}, \ldots, \hat{Z}_{N} \right) \right)
& = P(A_{\lfloor m/2 \rfloor}) P\!\left( A_{\lfloor m/2 \rfloor} \!\cap\! \psi_{N}\!\left( \hat{Z}_{1}, \ldots, \hat{Z}_{N} \right) \middle| A_{\lfloor m/2 \rfloor} \right)
\\ & \leq P(A_{\lfloor m/2 \rfloor}) \frac{1}{N-\cs} \left( \cs \ln\left( \frac{e N}{\cs} \right) + \ln\left(\frac{3}{\conf}\right) \right).
\end{align*}

Additionally, again since $N$ is conditionally ${\rm Binomial}(\lceil m/2 \rceil, P(A_{\lfloor m/2 \rfloor}))$-distributed
given $Z_{1},\ldots,Z_{\lfloor m/2 \rfloor}$,
applying a Chernoff bound (under the conditional distribution given $Z_{1},\ldots,Z_{\lfloor m/2 \rfloor}$),
combined with the law of total probability, we obtain that on an event $E_{2}$ of probability at least $1 - \conf/3$, 
if $P(A_{\lfloor m/2 \rfloor}) \geq \frac{16}{m} \ln\left(\frac{3}{\conf}\right) \geq \frac{8}{\lceil m/2 \rceil} \ln\left(\frac{3}{\conf}\right)$, then
\begin{equation*}
N \geq P(A_{\lfloor m/2 \rfloor}) \lceil m/2 \rceil / 2 \geq P(A_{\lfloor m/2 \rfloor}) m / 4.
\end{equation*}
Also note that if $P(A_{m}) \geq \frac{1}{m} \left( 21 \cs +  16 \ln\left(\frac{3}{\conf}\right) \right)$, 
then \eqref{eqn:monotone-compression-conditional-subset} and monotonicity of probability measures
imply $P(A_{\lfloor m/2 \rfloor}) \geq \frac{1}{m} \left( 21 \cs +  16 \ln\left(\frac{3}{\conf}\right) \right)$ as well.
In particular, if this occurs with $E_{2}$, then we have $N \geq P(A_{\lfloor m/2 \rfloor}) m / 4 > 5 \cs$.
Thus, by monotonicity of $x \mapsto \Log(x)/x$ for $x > 0$, we have that on $E_{1} \cap E_{2}$, 
if $P(A_{m}) \geq \frac{1}{m} \left( 21 \cs + 16 \ln\left(\frac{3}{\conf}\right) \right)$, 
then
\begin{align*}
P(A_{m}) & < P(A_{\lfloor m/2 \rfloor}) \frac{1}{N - (N/5)} \left( \cs \Log\left(\frac{e N}{\cs}\right) + \ln\left(\frac{3}{\conf}\right) \right)
\\ & \leq \frac{5}{m} \left( \cs \Log\left(\frac{e P(A_{\lfloor m/2 \rfloor}) m}{4 \cs}\right) + \ln\left(\frac{3}{\conf}\right) \right).
\end{align*}

The inductive hypothesis implies that, on an event $E_{3}$ of probability at least $1 - \conf/3$, 
\begin{equation*}
P(A_{\lfloor m/2 \rfloor}) 
\leq \frac{1}{\lfloor m/2 \rfloor} \left( 21 \cs + 16 \ln\left(\frac{9}{\conf}\right) \right).
\end{equation*}
Since $m \geq 39$, we have $\lfloor m/2 \rfloor \geq (m-2)/2 \geq (37/78)m$, so that the above implies
\begin{equation*}
P(A_{\lfloor m/2 \rfloor}) \leq \frac{78}{37 m} \left( 21 \cs + 16 \ln\left(\frac{9}{\conf}\right) \right). 
\end{equation*}
Thus, on $E_{1} \cap E_{2} \cap E_{3}$, if $P(A_{m}) \geq \frac{1}{m} \left( 21 \cs + 16 \ln\left(\frac{3}{\conf}\right) \right)$, 
then
\begin{align*}
P(A_{m}) & < \frac{5}{m} \left( \cs \Log\left( \frac{78 e}{4 \cdot 37} \left( 21 + \frac{16}{\cs} \ln\left(\frac{9}{\conf}\right) \right) \right) + \ln\left(\frac{3}{\conf}\right) \right)
\\ & \leq \frac{5}{m} \left( \cs \Log\left( \frac{78 \cdot 20}{37 \cdot 11} \left( \frac{21 \cdot 11 e}{16 \cdot 5} + \frac{11 e}{5} \ln(3) + \frac{11 e}{5 \cs} \ln\left(\frac{3}{\conf}\right) \right) \right) + \ln\left(\frac{3}{\conf}\right) \right).
\end{align*}
By Lemma~\ref{lem:log-factors-abstract} in Appendix~\ref{app:technical-lemmas}, 
this last expression is at most 
\begin{equation*}
\frac{5}{m} \left( \cs \Log\left( \frac{78 \cdot 20}{37 \cdot 11} \left( \frac{21 \cdot 11 e}{16 \cdot 5} + \frac{11 e}{5} \ln(3) + e \right) \right) + \frac{16}{5} \ln\left(\frac{3}{\conf}\right) \right)
< \frac{1}{m} \left( 21 \cs + 16 \ln\left(\frac{3}{\conf}\right) \right),
\end{equation*}
contradicting the condition $P(A_{m}) \geq \frac{1}{m} \left( 21 \cs + 16 \ln\left(\frac{3}{\conf}\right) \right)$.
Therefore,
on $E_{1} \cap E_{2} \cap E_{3}$,
\begin{equation*}
P(A_{m}) < \frac{1}{m} \left( 21 \cs + 16 \ln\left(\frac{3}{\conf}\right) \right).
\end{equation*}
Noting that, by the union bound, the event $E_{1} \cap E_{2} \cap E_{3}$ has probability at least $1-\conf$,
this extends the inductive hypothesis to $m^{\prime} = m$.  By the principle of induction, this completes the proof of the 
first claim in Theorem~\ref{thm:monotone-compression}.

For the bound on the expectation in \eqref{eqn:monotone-compression-expectation}, 
we note that (as in the proof of Theorem~\ref{thm:monotone-erm}),
letting $\eps_{m} = \frac{1}{m} \left( 21 \cs + 16 \ln(3) \right)$, 
the result just established can be restated as:
$\forall \eps > \eps_{m}$, 
\begin{equation*}
\P\left( P(A_{m}) > \eps \right) \leq 3 \exp\left\{ (21/16) \cs - \eps m / 16 \right\}.
\end{equation*}
Specifically, this is obtained by setting the bound in \eqref{eqn:monotone-compression-eps-conf}
equal to $\eps$ and solving for $\conf$, 
the value of which is in $(0,1)$ for any $\eps > \eps_{m}$.
Furthermore, for any $\eps \leq \eps_{m}$, we of course
still have $\P\left( P(A_{m}) > \eps \right) \leq 1$. 
Therefore, we have that
\begin{align*}
\E\left[ P(A_{m}) \right]
& = \int_{0}^{\infty} \P\left( P(A_{m}) > \eps \right) {\rm d}\eps
\leq \eps_{m} + \int_{\eps_{m}}^{\infty} 3 \exp\left\{ (21/16) \cs - \eps m / 16 \right\} {\rm d}\eps
\\ & = \eps_{m} + \frac{3 \cdot 16}{m} \exp\left\{ (21/16) \cs - \eps_{m} m / 16 \right\}
= \frac{1}{m} \left(21 \cs + 16 \ln(3)\right) + \frac{16}{m}
\\ & = \frac{1}{m} \left( 21 \cs + 16 \ln(3e) \right)
\leq \frac{21 \cs + 34}{m}.
\end{align*}
{\vskip -6.5mm}
\end{proof}

\section{Application to the Closure Algorithm for Intersection-Closed Classes}
\label{sec:int-closed}

One family of concept spaces studied in the learning theory literature, 
due to their interesting special properties, is the \emph{intersection-closed} classes \citep*{natarajan:87,helmbold:90,haussler:94,kuhlmann:99,auer:07}.
Specifically, the class $\C$ is called \emph{intersection-closed} if
the collection of sets $\{ \{x : h(x) = +1\} : h \in \C \}$ is closed under intersections:
that is, for every $h,g \in \C$, the classifier $x \mapsto 2 \ind[ h(x) = g(x) = +1 ] - 1$ is also contained in $\C$.
For instance, the class of conjunctions on $\{0,1\}^{p}$, the class of axis-aligned rectangles on $\reals^{p}$,
and the class $\{ h : | \{ x : h(x) = +1 \} | \leq \dim \}$ of classifiers labeling at most $\dim$ points positive,
are all intersection-closed.

In the context of learning in the realizable case, there is a general learning strategy, called the \emph{Closure} algorithm,
designed for learning with intersection-closed concept spaces, which has been a subject of frequent study.
Specifically, for any $m \in \nats \cup \{0\}$, given any data set $L_{m} = \{(x_{1},y_{1}),\ldots,(x_{m},y_{m})\} \in (\X \times \Y)^{m}$ with $\C[L_{m}] \neq \emptyset$,
the Closure algorithm $\alg(L_{m})$ for $\C$ produces the classifier 
$\hat{h}_{m} : \X \to \Y$ with $\{ x : \hat{h}_{m}(x) = +1 \} = \bigcap_{h \in \C[L_{m}]} \{ x : h(x) = +1 \}$:
that is, $\hat{h}_{m}(x) = +1$ if and only if every $h \in \C$ consistent with $L_{m}$ (i.e., $\er_{L_{m}}(h) = 0$) has $h(x) = +1$.\footnote{For 
simplicity, we suppose $\C$ is such that this set $\bigcap_{h \in \C[L_{m}]} \{ x : h(x) = +1 \}$ is measurable for every $L_{m}$, which is the 
case for essentially all intersection-closed concept spaces of practical interest.}
Defining $\bar{\C}$ as the set of all classifiers $h : \X \to \Y$ for which there exists a nonempty $\G \subseteq \C$ with $\{x : h(x) = +1\} = \bigcap_{g \in \G} \{x : g(x) = +1\}$,
\citet*{auer:07} have argued that $\bar{\C}$ is an intersection-closed concept space containing $\C$, with $\vc(\bar{\C}) = \vc(\C)$.
Thus, for $\hat{h}_{m} = \alg(\L_{m})$ (where $\alg$ is the Closure algorithm),
since $\hat{h}_{m} \in \bar{\C}[\L_{m}]$, Lemma~\ref{lem:classic-erm} immediately
implies that, for any $m \in \nats$,
with probability at least $1-\conf$, $\er\left( \hat{h}_{m}\right) \lesssim \frac{1}{m} \left( \dim \Log(\frac{m}{\dim}) + \Log(\frac{1}{\conf}) \right)$.
However, by a more-specialized analysis, \citet*{auer:04,auer:07} were able to 
show that, for intersection-closed classes $\C$, the Closure algorithm in fact 
achieves $\er\left( \hat{h}_{m} \right) \lesssim \frac{1}{m} \left( \dim \Log(\dim) + \Log(\frac{1}{\conf}) \right)$
with probability at least $1-\conf$, which is an improvement for large $m$.
They also argued that, for a special subfamily of intersection-closed classes (namely, those with \emph{homogeneous spans}),
this bound can be further refined to $\frac{1}{m} \left( \dim + \Log(\frac{1}{\conf}) \right)$,
which matches (up to constant factors) the lower bound \eqref{eqn:pac-lower-bound}.
However, they left open the question of whether this refinement is achievable for 
general intersection-closed concept spaces (by Closure, or any other algorithm).

In the following result, we prove that the Closure algorithm indeed always achieves
the optimal bound (up to constant factors) for intersection-closed concept spaces, 
as a simple consequence of either Theorem~\ref{thm:monotone-erm} or Theorem~\ref{thm:monotone-compression}.
This fact was very recently
also obtained by \citet*{darnstadt:15} via a related direct approach;
however, we note that the constant factors obtained here
are significantly smaller (by roughly a factor of $15.5$, for large $\dim$).

\begin{theorem}
\label{thm:int-closed}
If $\C$ is intersection-closed and $\alg$ is the Closure algorithm,
then for any $m \in \nats$ and $\conf \in (0,1)$, letting $\hat{h}_{m} = \alg( \{ (X_{1},\target(X_{1})),\ldots,(X_{m},\target(X_{m})) \} )$, 
with probability at least $1-\conf$,
\begin{equation*}
\er\left( \hat{h}_{m} \right) \leq \frac{1}{m} \left( 21 \dim + 16 \ln\left(\frac{3}{\conf}\right) \right).
\end{equation*}
Furthermore, 
\begin{equation*}
\E\left[ \er\left( \hat{h}_{m} \right) \right] \leq \frac{21 \dim + 34}{m}.
\end{equation*}
\end{theorem}
\begin{proof}
For each $t \in \nats \cup \{0\}$ and $x_{1},\ldots,x_{t} \in \X$,
define $\psi_{t}(x_{1},\ldots,x_{t}) = \{ x \in \X : \hat{h}_{x_{[t]}}(x) \neq \target(x) \}$, 
where $\hat{h}_{x_{[t]}} = \alg( \{(x_{1},\target(x_{1})),\ldots,(x_{t},\target(x_{t}))\} )$.
Fix any $x_{1},x_{2},\ldots \in \X$, let $L_{t} = \{ (x_{1},\target(x_{1})),\ldots,(x_{t},\target(x_{t})) \}$ for each $t \in \nats$,
and note that for any $t \in \nats$, the classifier $\hat{h}_{x_{[t]}}$ produced by $\alg(L_{t})$ is consistent with $L_{t}$,
which implies $\psi_{t}(x_{1},\ldots,x_{t}) \cap \{x_{1},\ldots,x_{t}\} = \emptyset$.
Furthermore, since $\target \in \C[ L_{t} ]$,
we have that $\{ x : \hat{h}_{x_{[t]}}(x) = +1 \} \subseteq \{ x : \target(x) = +1 \}$, 
which together with the definition of $\hat{h}_{x_{[t]}}$ implies
\begin{align}
\psi_{t}(x_{1},\ldots,x_{t})  & = \{ x \in \X : \hat{h}_{x_{[t]}}(x) = -1, \target(x) = +1 \} \notag
\\ & = \bigcup_{h \in \C[L_{t}]} \{ x \in \X : h(x) = -1, \target(x) = +1 \} \label{eqn:closure-error-region}
\end{align}
for every $t \in \nats$.
Furthermore, for any $t \in \nats$,
$\C[ L_{t+1}  ] \subseteq \C[ L_{t} ]$.
Together with monotonicity of the union, these two observations imply
\begin{align*}
\psi_{t+1}(x_{1},\ldots,x_{t+1}) 
& = \bigcup_{h \in \C[L_{t+1}]} \{ x \in \X : h(x) = -1, \target(x) = +1 \}
\\ & \subseteq \bigcup_{h \in \C[L_{t}]} \{ x \in \X : h(x) = -1, \target(x) = +1 \}
= \psi_{t}(x_{1},\ldots,x_{t}).
\end{align*}
Thus, $\psi_{t}$ defines a consistent monotone rule.
Also, since $\alg$ always produces a function in $\bar{\C}$,
we have $\psi_{t}(x_{1},\ldots,x_{t}) \in \{ \{ x \in \X : h(x) \neq \target(x) \} : h \in \bar{\C} \}$ for every $t \in \nats$,
and it is straightforward to show that the VC dimension of this collection of sets is exactly $\vc(\bar{\C})$ \citep*[see][Lemma 4.12]{vidyasagar:03},
which \citet*{auer:07} have argued equals $\dim$.
From this, we can already infer a bound $\frac{4}{m}\left( 17 \dim + 4 \ln\left(\frac{4}{\conf}\right) \right)$ via Theorem~\ref{thm:monotone-erm}.
However, we can refine the constant factors in this bound by noting that $\psi_{t}$ can also be represented as a 
consistent monotone sample compression rule of size $\dim$, and invoking Theorem~\ref{thm:monotone-compression}.
The rest of this proof focuses on establishing this fact.

Fix any $t \in \nats$.
It is well known in the literature \citep*[see e.g.,][Theorem 1]{auer:07} that 
there exist $k \in [\dim] \cup \{0\}$ and distinct $i_{1},\ldots,i_{k} \in [t]$
such that $\target(x_{i_{j}}) = +1$ for all $j \in [k]$,
and letting $L_{i_{[k]}} = \{(x_{i_{1}},+1),\ldots,(x_{i_{k}},+1)\}$,
we have $\bigcap_{h \in \C[ L_{i_{[k]}} ]} \{x : h(x) = +1\} = \bigcap_{h \in \C[L_{t}]} \{x : h(x) = +1\}$;
in particular, letting $\hat{h}_{x_{i_{[k]}}} = \alg(L_{i_{[k]}})$,
this implies $\hat{h}_{x_{i_{[k]}}} = \hat{h}_{x_{[t]}}$.
This further implies $\psi_{t}(x_{1},\ldots,x_{t}) = \psi_{k}(x_{i_{1}},\ldots,x_{i_{k}})$,
so that defining the compression function $(\empcs_{t}(x_{[t]}),i_{t,1}(x_{[t]}),\ldots,i_{t,\empcs_{t}(x_{[t]})}(x_{[t]})) = (k,i_{1},\ldots,i_{k})$
for $k$ and $i_{1},\ldots,i_{k}$ as above, for each $x_{1},\ldots,x_{t} \in \X$, 
and defining the reconstruction function $\phi_{t,k^{\prime}}(x_{1}^{\prime},\ldots,x_{k^{\prime}}^{\prime}) = \psi_{k^{\prime}}(x_{1}^{\prime},\ldots,x_{k^{\prime}}^{\prime})$ 
for each $t \in \nats$, $k^{\prime} \in [\dim] \cup \{0\}$, and $x_{1}^{\prime},\ldots,x_{k^{\prime}}^{\prime} \in \X$,
we have that 
$\psi_{t}(x_{1},\ldots,x_{t})$ $= \phi_{t,\empcs_{t}(x_{[t]})}( x_{i_{t,1}(x_{[t]})}, \ldots, x_{i_{t,\empcs_{t}(x_{[t]})}(x_{[t]})} )$
for all $t \in \nats$ and $x_{1},\ldots,x_{t} \in \X$.
Furthermore, since $(x_{1},\ldots,x_{t}) \mapsto \C[\{(x_{1},\target(x_{1})),\ldots,(x_{t},\target(x_{t}))\}]$ is invariant to permutations of its arguments,
it follows from \eqref{eqn:closure-error-region} that $\psi_{t}$ is permutation-invariant for every $t \in \nats$;
this also means that, for the choice of $\phi_{t,k^{\prime}}$ above, the function $\phi_{t,k^{\prime}}$ is also permutation-invariant.
Altogether, we have that $\psi_{t}$ is a consistent monotone sample compression rule of size $\dim$.
Thus, since $\er\left( \hat{h}_{m} \right) = \Px( \psi_{m}(X_{1},\ldots,X_{m}) )$ for $m \in \nats$,
the stated result follows directly from Theorem~\ref{thm:monotone-compression}
(with $\Z = \X$, $P = \Px$, and $\psi_{t}$ defined as above).
\end{proof}

\section{Application to the CAL Active Learning Algorithm}
\label{sec:cal}

As another interesting application of Theorem~\ref{thm:monotone-compression},
we derive an improved bound on the label complexity of a well-studied active learning
algorithm, usually referred to as CAL after its authors \citet*{cohn:94}.
Formally, in the active learning protocol, the learning algorithm $\alg$ is given access
to the \emph{unlabeled} data sequence $X_{1},X_{2},\ldots$ (or some sufficiently-large 
finite initial segment thereof), and then sequentially requests to observe the labels:
that is, it selects an index $t_{1}$ and requests to observe the label $\target(X_{t_{1}})$, 
at which time it is permitted access to $\target(X_{t_{1}})$;
it may then select another index $t_{2}$ and request to observe the label $\target(X_{t_{2}})$, 
is then permitted access to $\target(X_{t_{2}})$, and so on.
This continues until at most some given number $n$ of labels have been requested (called the \emph{label budget}), 
at which point the algorithm should halt and return a classifier $\hat{h}$; 
we denote this as $\hat{h} = \alg(n)$ (leaving the dependence on the unlabeled data implicit, for simplicity).
We are then interested in characterizing a sufficient size for the budget $n$ so that,
with probability at least $1-\conf$, $\er(\hat{h}) \leq \eps$; this size is known as
the \emph{label complexity} of $\alg$.

The CAL active learning algorithm is based on a very elegant and natural principle: 
never request a label that can be deduced from information already obtained.
CAL is defined solely by this principle, employing no additional criteria in its choice of queries.
Specifically, the algorithm proceeds by considering randomly-sampled data points one at a time,
and to each it applies the above principle, skipping over the labels that can be deduced, 
and requesting the labels that cannot be.  In favorable scenarios, as the number of 
label requests grows, the frequency of encountering a sample whose label 
cannot be deduced should diminish.  The key to bounding the label complexity of 
CAL is to characterize the rate at which this frequency shrinks.
To further pursue this discussion with rigor, let us define the \emph{region of disagreement}
for any set $\H$ of classifiers:
\begin{equation*}
\DIS(\H) = \{x \in \X : \exists h,g \in \H \text{ s.t. } h(x) \neq g(x)\}.
\end{equation*}
Then the CAL active learning algorithm is formally defined as follows.

\begin{bigboxit}
Algorithm: \textbf{\CAL}($n$)
\\ 0. $m \gets 0$, $t \gets 0$, $V_{0} \gets \C$
\\ 1. While $t < n$  and $m < 2^{n}$
\\ 2. \quad $m \gets m+1$
\\ 3. \quad If $X_{m} \in \DIS(V_{m-1})$
\\ 4. \qquad Request label $Y_m = \target(X_{m})$; let $V_{m} \gets V_{m-1}[\{(X_{m},Y_{m})\}]$, $t \gets t+1$
\\ 5. \quad Else $V_{m} \gets V_{m-1}$
\\ 6. Return any $\hat{h} \in V_{m}$
\end{bigboxit}

This algorithm has several attractive properties.
One is that, since it only removes classifiers from $V_{m}$ upon disagreement with $\target$, 
it maintains the invariant that $\target \in V_{m}$.
Another property is that, since it maintains $\target \in V_{m}$,
and it only refrains from requesting a label if every classifier in $V_{m}$ agrees on the label
(and hence agrees with $\target$, so that requesting the label would not affect $V_{m}$ anyway), 
it maintains the invariant that $V_{m} = \C[\L_{m}]$, where $\L_{m} = \{(X_{1},\target(X_{1})),\ldots,(X_{m},\target(X_{m}))\}$.

This algorithm has been studied a great deal in the literature \citep*{cohn:94,hanneke:thesis,hanneke:11a,hanneke:12a,hanneke:fntml,el-yaniv:12,hanneke:15a},
and has inspired an entire genre of active learning algorithms referred to as \emph{disagreement-based} (or sometimes as \emph{mellow}),
including several methods possessing desirable properties such as robustness to classification noise 
\citep*[e.g.,][]{balcan:06,balcan:09,dasgupta:07,koltchinskii:10,hanneke:12b,hanneke:fntml}.
There is a substantial literature studying the label complexity of CAL and other disagreement-based 
active learning algorithms; the interested reader is referred to the recent survey article of \citet*{hanneke:fntml} for a thorough discussion of this literature.
Much of that literature discusses characterizations of the label complexity in terms of a quantity 
known as the \emph{disagreement coefficient} \citep*{hanneke:07b,hanneke:thesis}.  However, 
\citet*{hanneke:15a} have recently discovered that a quantity known as the 
\emph{version space compression set size} (a.k.a. \emph{empirical teaching dimension})
can sometimes provide a smaller bound on the label complexity of CAL.
Specifically, the following quantity was introduced in the works of \citet*{el-yaniv:10,hanneke:07a}.

\begin{definition}
\label{def:td-hat}
For any $m \in \nats$ and $\L \in (\X \times \Y)^{m}$, 
the \emph{version space compression set} $\hat{C}_{\L}$ is a
smallest subset of $\L$ satisfying $\C[\hat{C}_{\L}] = \C[\L]$.
We then define $\hat{\cs}(\L) = |\hat{C}_{\L}|$, the 
\emph{version space compression set size}.
In the special case $\L = \L_{m}$,
we abbreviate $\hat{\cs}_{m} = \hat{\cs}(\L_{m})$.
Also define $\hat{\cs}_{1:m} = \max_{t \in [m]} \hat{\cs}_{t}$,
and for any $\conf \in (0,1)$, define 
$\tilde{\cs}_{m}(\conf) = \min\{ b \in [m] \cup \{0\} : \P(\hat{n}_{m} \leq b) \geq 1-\conf \}$ and
$\tilde{\cs}_{1:m}(\conf) = \min\{ b \in [m] \cup \{0\} : \P(\hat{n}_{1:m} \leq b) \geq 1-\conf \}$.
\end{definition}

The recent work of \citet*{hanneke:15a} studies this quantity for several concept spaces and distributions, 
and also identifies general relations between $\hat{\cs}_{m}$ and the more-commonly studied
disagreement coefficient $\dc$ of \citep*{hanneke:07b,hanneke:thesis}.  Specifically, 
for any $r > 0$, define $\Ball(\target,r) = \{h \in \C : \Px(x : h(x) \neq \target(x)) \leq r\}$.
Then the disagreement coefficient is defined, for any $r_{0} \geq 0$, as 
\begin{equation*}
\dc(r_{0}) = \sup_{r > r_{0}} \frac{\Px(\DIS(\Ball(\target,r)))}{r} \lor 1.
\end{equation*}
Both $\tilde{\cs}_{1:m}(\conf)$ and $\dc(r_{0})$ are complexity measures dependent on $\target$ and $\Px$.
\citet*{hanneke:15a} relate them by showing that
\begin{equation}
\label{eqn:dc-td-lb}
\dc(1/m) \lesssim \tilde{\cs}_{1:m}(1/20) \lor 1,
\end{equation}
and for general $\conf \in (0,1)$,\footnote{The original claim from \citet*{hanneke:15a} involved a maximum of 
minimal $(1-\conf)$-confidence bounds on $\hat{\cs}_{t}$ over $t \in [m]$, but the same proof can be used to establish this slightly stronger claim.}
\begin{equation}
\label{eqn:td-dc-ub}
\tilde{\cs}_{1:m}(\conf) \lesssim \dc(\dim/m) \left( \dim \Log(\dc(\dim/m)) + \Log\left(\frac{\Log(m)}{\conf}\right) \right) \Log(m).
\end{equation}

\citet*{hanneke:15a} prove that, for $\CAL(n)$ to 
produce $\hat{h}$ with $\er(\hat{h}) \leq \eps$ with probability at least $1-\conf$,
it suffices to take a budget $n$ of size proportional to 
\begin{equation}
\label{eqn:original-cal}
\left( \max_{m \in [M(\eps,\conf/2)]} \tilde{\cs}_{m}(\conf_{m}) \Log\left( \frac{m}{\tilde{\cs}_{m}(\conf_{m})} \right) + \Log\left(\frac{\Log(M(\eps,\conf/2))}{\conf}\right) \right) \Log(M(\eps,\conf/2)),
\end{equation}
where the values $\conf_{m} \in (0,1]$ are such that $\sum_{i=0}^{\lfloor \log_{2}(M(\eps,\conf/2)) \rfloor} \conf_{2^{i}} \leq \conf/4$,
and $M(\eps,\conf/2)$ is the smallest $m \in \nats$ for which
$\P\left( \sup_{h \in \C[\L_{m}]} \er(h) \leq \eps \right) \geq 1-\conf/2$;
the quantity $M(\eps,\conf)$ is discussed at length below in Section~\ref{sec:erm}.
They also argue that this is essentially a \emph{tight} characterization of the label complexity of CAL,
up to logarithmic factors.

The key to obtaining this result is establishing an upper bound on $\Px(\DIS(V_{m}))$
as a function of $m$, where (as in CAL) $V_{m} = \C[\L_{m}]$.
One basic observation indicating that $\Px(\DIS(V_{m}))$ can be related to the version space compression set size is that,
by exchangeability of the $X_{i}$ random variables,
\begin{align*}
\E\left[ \Px(\DIS(V_{m})) \right] 
& = \E\left[ \ind\left[ X_{m+1} \in \DIS( \C[ \L_{m} ] ) \right] \right]
\\ & = \frac{1}{m+1} \sum_{i=1}^{m+1} \E\left[ \ind\left[ X_{i} \in \DIS( \C[ \L_{m+1} \setminus \{(X_{i},\target(X_{i}))\} ] ) \right] \right]
\\ & \leq \frac{1}{m+1} \sum_{i=1}^{m+1} \E\left[ \ind\left[ (X_{i},\target(X_{i})) \in \hat{C}_{\L_{m+1}} \right] \right]
= \frac{\E\left[ \hat{n}_{m+1} \right]}{m+1},
\end{align*}
where the inequality is due to the observation that any $X_{i} \in \DIS( \C[ \L_{m+1} \setminus \{(X_{i},\target(X_{i}))\} ] )$
is necessarily in the version space compression set $\hat{C}_{\L_{m+1}}$, and the last equality is by linearity of the expectation.
However, obtaining the bound \eqref{eqn:original-cal} required a more-involved argument from 
\citet*{hanneke:15a}, to establish a high-confidence bound on $\Px(\DIS(V_{m}))$, rather than a bound
on its expectation.
Specifically, by combining a perspective introduced by 
\citet*{el-yaniv:10,el-yaniv:12}, with the observation that $\DIS(V_{m})$ may be 
represented as a sample compression scheme of size $\hat{\cs}_{m}$, and invoking 
Lemma~\ref{lem:classic-compression}, \citet*{hanneke:15a} prove that, 
with probability at least $1-\conf$, 
\begin{equation}
\label{eqn:whey-PDIS}
\Px(\DIS(V_{m})) \lesssim \frac{1}{m} \left( \hat{\cs}_{m} \Log\left( \frac{m}{\hat{\cs}_{m}} \right) + \Log\left(\frac{1}{\conf}\right) \right).
\end{equation}

In the present work, we are able to entirely eliminate the factor $\Log\left(\frac{m}{\hat{\cs}_{m}}\right)$ from the first term,
simply by observing that the region $\DIS(V_{m})$ is \emph{monotonic} in $m$.  Specifically, 
by combining this monotonicity observation with the description of $\DIS(V_{m})$
as a compression scheme from \citet*{hanneke:15a}, the refined bound follows from arguments similar to the proof of
Theorem~\ref{thm:monotone-compression}.
Formally, we have the following result.

\begin{theorem}
\label{thm:PDIS}
For any $m \in \nats$ and $\conf \in (0,1)$, with probability at least $1-\conf$, 
\begin{equation*}
\Px(\DIS(V_{m})) \leq \frac{16}{m} \left( 2 \hat{\cs}_{1:m} + \ln\left(\frac{3}{\conf}\right) \right).
\end{equation*}
\end{theorem}

We should note that, while Theorem~\ref{thm:PDIS} indeed eliminates a logarithmic factor
compared to \eqref{eqn:whey-PDIS}, this refinement is also accompanied by an increase
in the complexity measure, replacing $\hat{\cs}_{m}$ with $\hat{\cs}_{1:m}$.
This arises from our proof, since (as in the proof of Theorem~\ref{thm:monotone-compression})
the argument relies on $\hat{\cs}_{1:m}$ being a sample compression set size, not just for the full 
sample, but also for any prefix of the sample.
The effect of this increase is largely benign in this context, since the bound \eqref{eqn:original-cal} 
on the label complexity of CAL, derived from \eqref{eqn:whey-PDIS}, involves maximization over the sample size anyway.

Although Theorem~\ref{thm:PDIS} follows from the same principles as Theorem~\ref{thm:monotone-compression} 
(i.e., $\DIS(V_{t})$ being a consistent monotone rule expressible as a sample compression scheme),  
it does not quite follow as an immediate consequence of Theorem~\ref{thm:monotone-compression},
due fact that the size $\hat{\cs}_{1:m}$ of the sequence of sample compression schemes 
can vary based on the specific samples (including their \emph{order}).
For this reason, we provide a specialized proof of this result in Appendix~\ref{app:cal}, 
which follows an argument nearly-identical to that of Theorem~\ref{thm:monotone-compression},
with only a few minor changes to account for this variability of $\hat{\cs}_{1:m}$ using special properties of the sets $\DIS(V_{t})$.

Based on this result, and following precisely the same arguments as \citet*{hanneke:15a},\footnote{%
The only small twist is that we replace $\max_{m \leq M(\eps,\conf/2)} \tilde{\cs}_{m}(\conf_{m})$ from \eqref{eqn:original-cal} with $\tilde{\cs}_{1:M(\eps,\conf/2)}(\conf/4)$.
As the purpose of these $\tilde{\cs}_{m}(\conf_{m})$ values in the original proof is to provide bounds on their respective $\hat{\cs}_{m}$ values (which in our context, are $\hat{\cs}_{1:m}$ values),
holding simultaneously for all $m = 2^{i} \in [M(\eps,\conf/2)]$ with probability at least $1-\conf/4$, 
the value $\tilde{\cs}_{1:M(\eps,\conf/2)}(\conf/4)$ can clearly be used instead.  
If desired, by a union bound we can of course bound 
$\tilde{\cs}_{1:M(\eps,\conf/2)}(\conf/4) \leq \max_{m \in [M(\eps,\conf/2)]} \tilde{\cs}_{m}(\conf_{m})$,
for any sequence $\conf_{m}$ in $(0,1]$ with $\sum_{m \in [M(\eps,\conf/2)]} \conf_{m} \leq \conf/4$.}  
we arrive at the following bound on the label complexity of CAL.
For brevity, we omit the proof, referring the interested reader to the original exposition of \citet*{hanneke:15a} for the details.

\begin{theorem}
\label{thm:CAL-label-complexity}
There is a universal constant $c \in (0,\infty)$ such that, 
for any $\eps,\conf \in (0,1)$,
for any $n \in \nats$ with 
\begin{equation*}
n \geq c \left( \tilde{\cs}_{1:M(\eps,\conf/2)}(\conf/4) + \Log\left(\frac{\Log(M(\eps,\conf/2))}{\conf}\right) \right) \Log(M(\eps,\conf/2)),
\end{equation*}
with probability at least $1-\conf$, the classifier $\hat{h}_{n} = \CAL(n)$ has $\er(\hat{h}_{n}) \leq \eps$.
\end{theorem}

It is also possible to state a distribution-free variant of Theorem~\ref{thm:PDIS}.
Specifically, consider the following definition, from \citet*{hanneke:15b}.

\begin{definition}
\label{def:star}
The \emph{star number} $\s$ is the largest integer $s$ such that 
there exist distinct points $x_{1},\ldots,x_{s} \in \X$ and classifiers $h_{0},h_{1},\ldots,h_{s} \in \C$
with the property that $\forall i \in [s]$, $\DIS(\{h_{0},h_{i}\}) \cap \{x_{1},\ldots,x_{s}\} = \{x_{i}\}$;
if no such largest integer exists, define $\s = \infty$.
\end{definition}

The star number is a natural combinatorial complexity measure, corresponding to the largest 
possible degree in the data-induced one-inclusion graph.  \citet*{hanneke:15b} provide several
examples of concept spaces exhibiting a variety of values for the star number (though it 
should be noted that many commonly-used concept spaces have $\s = \infty$: e.g., linear separators).
As a basic relation, one can easily show that $\s \geq \dim$.
\citet*{hanneke:15b} also relate the star number to many other complexity measures 
arising in the learning theory literature, including $\hat{\cs}_{m}$.
Specifically, they prove that, for every $m \in \nats$ and $\L \in (\X \times \Y)^{m}$ with $\C[\L] \neq \emptyset$, 
$\hat{\cs}(\L) \leq \s$, with equality in the worst case (over $m$ and $\L$).  Based on this fact, Theorem~\ref{thm:monotone-compression}
implies the following result.

\begin{theorem}
\label{thm:PDIS-star}
For any $m \in \nats$ and $\conf \in (0,1)$, with probability at least $1-\conf$,
\begin{equation*}
\Px(\DIS(V_{m})) \leq \frac{1}{m} \left( 21 \s + 16 \ln\left(\frac{3}{\conf}\right) \right). 
\end{equation*}
\end{theorem}
\begin{proof}
For every $t \in \nats$ and $x_{1},\ldots,x_{t} \in \X$, define $\psi_{t}(x_{1},\ldots,x_{t}) = \DIS(\C[\L_{x_{[t]}}])$,
where $\L_{x_{[t]}} = \{(x_{1},\target(x_{1})),\ldots,(x_{t},\target(x_{t}))\}$; 
$\psi_{t}$ is clearly permutation-invariant, and satisfies $\psi_{t}(x_{1},\ldots,x_{t}) \cap \{x_{1},\ldots,x_{t}\} = \emptyset$
(since every $h \in \C[\L_{x_{[t]}}]$ agrees with $\target$ on $\{x_{1},\ldots,x_{t}\}$).
Furthermore, monotonicity of $\L \mapsto \C[\L]$ and $\H \mapsto \DIS(\H)$ imply 
that any $t \in \nats$ and $x_{1},\ldots,x_{t+1} \in \X$ satisfy $\psi_{t+1}(x_{1},\ldots,x_{t+1}) \subseteq \psi_{t}(x_{1},\ldots,x_{t})$,
so that $\psi_{t}$ is a consistent monotone rule.
Also define $\phi_{t,k}(x_{1},\ldots,x_{k}) = \psi_{k}(x_{1},\ldots,x_{k})$ for any $k \in [t]$ and $x_{1},\ldots,x_{k} \in \X$,
and $\phi_{t,0}() = \DIS(\C)$.  Since $\psi_{k}$ is permutation-invariant for every $k \in [t]$, so is $\phi_{t,k}$.
For any $x_{1},\ldots,x_{t} \in \X$, from Definition~\ref{def:td-hat}, 
there exist distinct $i_{t,1}(x_{[t]}),\ldots,i_{t,\hat{\cs}(\L_{x_{[t]}})}(x_{[t]}) \in [t]$ such that
$\hat{C}_{\L_{x_{[t]}}} = \{ (x_{i_{t,j}(x_{[t]})},\target(x_{i_{t,j}(x_{[t]})})) : j \in \{1,\ldots,\hat{\cs}(\L_{x_{[t]}})\} \}$,
and since $\C[\hat{C}_{\L_{x_{[t]}}}] = \C[\L_{x_{[t]}}]$, it follows that
$\phi_{t,\hat{\cs}(\L_{x_{[t]}})}(x_{i_{t,1}(x_{[t]})},\ldots,x_{i_{t,\hat{\cs}(\L_{x_{[t]}})}(x_{[t]})}) = \psi_{t}(x_{1},\ldots,x_{t})$.
Thus, since $\hat{\cs}(\L_{x_{[t]}}) \leq \s$ for all $t \in \nats$ \citep*{hanneke:15b},
$\psi_{t}$ is a consistent monotone sample compression rule of size $\s$.
The result immediately follows by applying Theorem~\ref{thm:monotone-compression} with $\Z = \X$, $P = \Px$, and $\psi_{t}$ as above.
\end{proof}

As a final implication for CAL, we can also plug the inequality $\hat{\cs}(\L) \leq \s$ into the bound 
from Theorem~\ref{thm:CAL-label-complexity} to reveal that CAL achieves a label complexity 
upper-bounded by a value proportional to 
$\s \Log(M(\eps,\conf/2)) + \Log\left(\frac{\Log(M(\eps,\conf/2))}{\conf}\right) \Log(M(\eps,\conf/2))$.

\noindent\textbf{Remark:} In addition to the above applications to active learning,
it is worth noting that, combined with the work of \citet*{el-yaniv:10}, the above results
also have implications for the setting of \emph{selective classification}: that is, the setting
in which, for each $t \in \nats$, given access to $(X_{1},\target(X_{1})),\ldots,(X_{t-1},\target(X_{t-1}))$ and $X_{t}$,
a learning algorithm is required either to make a prediction $\hat{Y}_{t}$ for $\target(X_{t})$, 
or to ``abstain'' from prediction; after each round $t$, the algorithm is permitted access 
to the value $\target(X_{t})$.  Then the error rate is the probability the prediction $\hat{Y}_{t}$ is incorrect (conditioned on $X_{[t-1]}$), 
given that the algorithm chooses to predict,
and the \emph{coverage} is the probability the algorithm chooses to make a prediction at time $t$ (conditioned on $X_{[t-1]}$).
\citet*{el-yaniv:10} explore an extreme variant, called \emph{perfect selective classification}, 
in which the algorithm is required to only make predictions that will be correct with \emph{certainty}
(i.e., for any data sequence $x_{1},x_{2},\ldots$, the algorithm will never misclassify a point it chooses
to predict for).  \citet*{el-yaniv:10} find that a selective classification algorithm based on 
principles analogous to the CAL active learning algorithm obtains the optimal coverage among
all perfect selective classification algorithms; the essential strategy is to predict only if 
$X_{t} \notin \DIS(V_{t-1})$, taking $\hat{Y}_{t}$ as the label agreed-upon by every $h \in V_{t-1}$.  
In particular, this implies that the optimal coverage rate in perfect selective classification, 
on round $t$, is $1-\Px(\DIS(V_{t-1}))$.  Thus, combined with 
Theorem~\ref{thm:PDIS} or Theorem~\ref{thm:PDIS-star}, we can immediately obtain bounds
on the optimal coverage rate for perfect selective classification as well; 
in particular, this typically refines the bound of \citet*{el-yaniv:10} (and a later refinement by \citealp*{hanneke:15a})
by at least a logarithmic factor (though again, it is not a ``pure'' improvement, as Theorem~\ref{thm:PDIS} uses $\hat{\cs}_{1:m}$ in place of $\hat{\cs}_{m}$).

\section{Application to General Consistent PAC Learners}
\label{sec:erm}

In general, a \emph{consistent learning algorithm} $\alg$ is a learning algorithm such that, 
for any $m \in \nats$ and $L \in (\X \times \Y)^{m}$ with $\C[L] \neq \emptyset$,
$\alg(L)$ produces a classifier $\hat{h}$ consistent with $L$ (i.e., $\hat{h} \in \C[L]$).
In the context of learning in the realizable case, this is equivalent to $\alg$ being an instance of 
the well-studied method of \emph{empirical risk minimization}.
The study of general consistent learning algorithms focuses on the quantity $\sup_{h \in V_{m}} \er(h)$,
where $V_{m} = \C[\L_{m}]$, as above.  It is clear that the error rate achieved by any consistent
learning algorithm, given $\L_{m}$ as input, 
is at most $\sup_{h \in V_{m}} \er(h)$.  Furthermore, it is not hard to see that, for any given 
$\Px$ and $\target \in \C$, there exist consistent learning rules obtaining error rates
arbitrarily close to $\sup_{h \in V_{m}} \er(h)$, so that obtaining guarantees on the error rate
that hold \emph{generally} for all consistent learning algorithms requires us to bound this value.

Based on Lemma~\ref{lem:classic-erm} (taking $\A = \{ \{x : h(x) \neq \target(x)\} : h \in \C \}$), 
one immediately obtains a classic result (due to \citealp*{vapnik:74,blumer:89}),
that with probability at least $1-\conf$,
\begin{equation*}
\sup_{h \in V_{m}} \er(h) \lesssim \frac{1}{m} \left( \dim \Log\left(\frac{m}{\dim}\right) + \Log\left(\frac{1}{\conf}\right) \right).
\end{equation*}
This has been refined by \citet*{gine:06},\footnote{See also \citet*{hanneke:thesis}, for a simple direct proof of this result.} who argue that, with probability at least $1-\conf$, 
\begin{equation}
\label{eqn:gk-dc}
\sup_{h \in V_{m}} \er(h) \lesssim \frac{1}{m} \left( \dim \Log\left( \dc\left( \frac{\dim}{m} \right) \right) + \Log\left(\frac{1}{\conf}\right) \right).
\end{equation}
In the present work, by combining an argument of \citet*{hanneke:thesis}
with Theorem~\ref{thm:PDIS} above, we are able to obtain a new result,
which replaces $\dc\left( \frac{\dim}{m} \right)$ in \eqref{eqn:gk-dc}
with $\frac{\hat{\cs}_{1:m}}{\dim}$.  Specifically, we have the following result.

\begin{theorem}
\label{thm:vc-erm}
For any $\conf \in (0,1)$ and $m \in \nats$, with probability at least $1-\conf$,
\begin{equation*}
\sup_{h \in V_{m}} \er(h) \leq \frac{8}{m} \left( \dim \ln\left( \frac{49 e \hat{\cs}_{1:m}}{\dim} + 37 \right) + 8 \ln\left(\frac{6}{\conf}\right) \right).
\end{equation*}
\end{theorem}

The proof of Theorem~\ref{thm:vc-erm} follows a similar strategy to 
the inductive step from the proofs of Theorems~\ref{thm:monotone-erm}, \ref{thm:monotone-compression}, and \ref{thm:PDIS}.
The details are included in Appendix~\ref{app:erm}.

Additionally, since \citet*{hanneke:15b} prove that $\max_{t \in [m]} \max_{L \in (\X \times \Y)^{t}} \hat{\cs}(L) = \min\{ \s, m \}$, 
where $\s$ is the star number, the following new
distribution-free bound immediately follows.\footnote{The bound on the expectation follows by integrating the exponential bound on 
$\P(\sup_{h \in V_{m}} \er(h) > \eps)$ implied by the first statement in the corollary, as was done, for instance, in the proofs of Theorems~\ref{thm:monotone-erm} and \ref{thm:monotone-compression}.
We also note that, by using Theorem~\ref{thm:PDIS-star} in place of Theorem~\ref{thm:PDIS} in the proof of Theorem~\ref{thm:vc-erm},
one can obtain mildly better numerical constants in the logarithmic term in this corollary.}

\begin{corollary}
\label{cor:erm-star}
For any $m \in \nats$ and $\conf \in (0,1)$, with probability at least $1-\conf$,
\begin{equation*}
\sup_{h \in V_{m}} \er(h) \lesssim \frac{1}{m} \left( \dim \Log\left( \frac{\min\{\s,m\}}{\dim} \right) + \Log\left(\frac{1}{\conf}\right) \right).
\end{equation*}
Furthermore,
\begin{equation*}
\E\left[ \sup_{h \in V_{m}} \er(h) \right] \lesssim \frac{\dim}{m} \Log\left( \frac{\min\{\s,m\}}{\dim} \right).
\end{equation*}
\end{corollary}

Let us compare this result to \eqref{eqn:gk-dc}.
Since \citet*{hanneke:15b} prove that 
\begin{equation*}
\max_{\Px} \max_{\target \in \C} \dc(r_{0}) = \min\left\{ \s, \frac{1}{r_{0}} \right\},
\end{equation*}
and also (as mentioned) that $\max_{t \in [m]} \max_{L \in (\X \times \Y)^{t}} \hat{\cs}(L) = \min\{ \s, m \}$, 
we see that, at least in some scenarios (i.e., for some choices of $\Px$ and $\target$), 
the new bound in Theorem~\ref{thm:vc-erm} represents an improvement over \eqref{eqn:gk-dc}.  
In particular, the best \emph{distribution-free} bound obtainable from \eqref{eqn:gk-dc} is proportional to
\begin{equation}
\label{eqn:gk-dist-free}
\frac{1}{m} \left( \dim \Log\left(\frac{\min\{ \dim\s,  m \}}{\dim}\right) + \Log\left(\frac{1}{\conf}\right) \right),
\end{equation}
which is somewhat larger than the bound stated in Corollary~\ref{cor:erm-star} (which has $\s$ in place of $\dim\s$).
Also, recalling that \citet*{hanneke:15a} established that $\dc(1/m) \lesssim \tilde{\cs}_{1:m}(\conf) \lesssim \dim \dc(\dim/m) \polylog(m,1/\conf)$,
we should expect that the bound in Theorem~\ref{thm:vc-erm} is typically not much larger than \eqref{eqn:gk-dc} (and indeed will be smaller in many interesting cases).

\subsection{Necessary and Sufficient Conditions for $\boldsymbol{1/m}$ Rates for All Consistent Learners}
\label{sec:lower-bounds}

Corollary~\ref{cor:erm-star} provides a sufficient condition for every consistent learning algorithm
to achieve error rate with $O(1/m)$ asymptotic dependence on $m$: namely, $\s < \infty$.  
Interestingly, we can show that this condition is in fact also \emph{necessary} for every 
consistent learner to have a distribution-free bound on the error rate with $O(1/m)$ dependence on $m$.
To be clear, in this context, we only consider $m$ as the asymptotic variable:
that is, $m \to \infty$ while $\conf$ and $\C$ (including $\dim$ and $\s$) are held fixed.
This result is proven via the following theorem, establishing a worst-case lower bound on $\sup_{h \in V_{m}} \er(h)$.

\begin{theorem}
\label{thm:erm-lb}
For any $m \in \nats$ and $\conf \in (0,1/100)$,
there exists a choice of $\Px$ and $\target \in \C$ such that,
with probability greater than $\conf$, 
\begin{equation*}
\sup_{h \in V_{m}} \er(h) \gtrsim \frac{\dim + \Log( \min\{\s,m\} ) + \Log\left(\frac{1}{\conf}\right)}{m} \land 1.
\end{equation*}
Furthermore,
\begin{equation*}
\E\left[ \sup_{h \in V_{m}} \er(h) \right] \gtrsim \frac{\dim + \Log( \min\{\s,m\} )}{m} \land 1.
\end{equation*}
\end{theorem}
\begin{proof}
Since any $a,b,c \in \reals$ have $a+b+c \leq 3\max\{a,b,c\}$ and $a+b \leq 2\max\{a,b\}$, 
it suffices to establish $\frac{\dim}{m} \land 1$, $\frac{\Log\left(\frac{1}{\conf}\right)}{m} \land 1$, and $\frac{\Log( \min\{\s,m\} )}{m}$
as lower bounds separately for the first bound, and $\frac{\dim}{m} \land 1$ and $\frac{\Log(\min\{\s,m\})}{m}$ as lower bounds separately for the second bound.
Lower bounds proportional to $\frac{\dim}{m} \land 1$ (in both bounds) and $\frac{\Log\left(\frac{1}{\conf}\right)}{m} \land 1$ (in the first bound) 
are known in the literature \citep*{blumer:89,ehrenfeucht:89,haussler:94}, and in fact hold as lower bounds on the error rate guarantees achievable by \emph{any} learning algorithm.

For the remaining term, note that this term (with appropriately small constant factors) follows immediately from the others if $\s \leq 56$, so suppose $\s \geq 57$.
Fix any $\eps \in (0,1/48)$, let $M_{\eps} = \left\lfloor \frac{1+\eps}{\eps} \right\rfloor$,
and let $x_{1},\ldots,x_{\min\{\s,M_{\eps}\}} \in \X$ and $h_{0},h_{1},\ldots,h_{\min\{\s,M_{\eps}\}} \in \C$ be as in Definition~\ref{def:star}.
Choose the probability measure $\Px$ such that $\Px(\{x_{i}\}) = \eps$ for every $i \in \{2,\ldots,\min\{\s,M_{\eps}\}\}$,
and $\Px(\{x_{1}\}) = 1 - (\min\{\s,M_{\eps}\}-1)\eps \geq 0$.  Choose the target function $\target = h_{0}$.
Then note that, for any $m \in \nats$, if $\exists i \in \{2,\ldots,\min\{\s,M_{\eps}\}\}$ with $x_{i} \notin \{X_{1},\ldots,X_{m}\}$,
then $h_{i} \in V_{m}$, so that $\sup_{h \in V_{m}} \er(h) \geq \er(h_{i}) = \eps$.  

Characterizing the probability that $\{x_{2},\ldots,x_{\min\{\s,M_{\eps}\}}\} \subseteq \{X_{1},\ldots,X_{m}\}$ 
can be approached as an instance of the so-called \emph{coupon collector's problem}.
Specifically, let 
\begin{equation*}
\hat{M} = \min\left\{ m \in \nats : \{x_{2},\ldots,x_{\min\{\s,M_{\eps}\}}\} \subseteq \{X_{1},\ldots,X_{m}\} \right\}.
\end{equation*}
Note that $\hat{M}$ may be represented as a sum $\sum_{k=1}^{\min\{\s,M_{\eps}\}-1} G_{k}$ of independent geometric random variables $G_{k} \sim {\rm Geometric}(\eps (\min\{\s,M_{\eps}\}-k))$,
where $G_{k}$ corresponds to the waiting time between encountering the $(k-1)^{{\rm th}}$ and $k^{{\rm th}}$ distinct elements of $\{x_{2},\ldots,x_{\min\{\s,M_{\eps}\}}\}$ in the $X_{t}$ sequence.
A simple calculation reveals that $\E[ \hat{M} ] = \frac{1}{\eps} H_{\min\{\s,M_{\eps}\}-1}$,
where $H_{t}$ is the $t^{{\rm th}}$ harmonic number; in particular, $H_{t} \geq \ln(t)$.
Another simple calculation with this sum of independent geometric random variables reveals
${\rm Var}(\hat{M}) < \frac{\pi^{2}}{6 \eps^{2}}$.
Thus, Chebyshev's inequality implies that, with probability greater than $1/2$, 
$\hat{M} \geq \frac{1}{\eps} \ln( \min\{\s,M_{\eps}\}-1 ) - \frac{\pi}{\sqrt{3} \eps}$.
Since $\ln(\min\{\s,M_{\eps}\}-1) \geq \ln(48) > 2 \frac{\pi}{\sqrt{3}}$, the right hand side 
of this inequality is at least 
$\frac{1}{2\eps} \ln( \min\{\s,M_{\eps}\}-1 ) 
= \frac{1}{2\eps} \ln\left( \min\left\{ \s-1, \left\lfloor \frac{1}{\eps} \right\rfloor \right\} \right)$.
Altogether, we have that for any $m < \frac{1}{2\eps} \ln\left( \min\left\{ \s-1, \left\lfloor \frac{1}{\eps} \right\rfloor \right\} \right)$,
with probability greater than $1/2$, $\sup_{h \in V_{m}} \er(h) \geq \eps$.
By Markov's inequality, this further implies that, for any such $m$, $\E\left[ \sup_{h \in V_{m}} \er(h) \right] > \eps/2$.

For any $m \leq 47$, 
the $\frac{\Log(\min\{\s,m\})}{m}$ term in both lower bounds (with appropriately small constant factors)
follows from the lower bound proportional to $\frac{\dim}{m} \land 1$,
so suppose $m \geq 48$.
In particular, for any $c \in (4,\ln(56))$, 
letting $\eps = \frac{\ln( \min\{ \s-1, m \} )}{c m}$,
one can easily verify that $0 < \eps < 1/48$, and $m < \frac{1}{2\eps} \ln\left( \min\left\{ \s-1, \left\lfloor \frac{1}{\eps} \right\rfloor \right\} \right)$.
Therefore, with probability greater than $1/2 > \conf$, 
\begin{equation*}
\sup_{h \in V_{m}} \er(h) \geq \frac{\ln( \min\{ \s-1, m \} )}{c m},
\end{equation*}
and furthermore,
\begin{equation*}
\E\left[ \sup_{h \in V_{m}} \er(h) \right] > \frac{\ln( \min\{ \s-1, m \} )}{2 c m}.
\end{equation*}
The result follows by noting $\ln( \min\{ \s-1, m \} ) \geq \ln( \min\{ \s,m \} / 2 ) \geq \ln( \min\{\s,m\} ) / 2$ for $\s,m \geq 4$.
\end{proof}

Comparing Theorem~\ref{thm:erm-lb} with Corollary~\ref{cor:erm-star},
we see that the asymptotic dependences on $m$ are identical,
though they differ in their joint dependences on $\dim$ and $m$.
The precise dependence on both $\dim$ and $m$ from Corollary~\ref{cor:erm-star} can 
be included in the lower bound of Theorem~\ref{thm:erm-lb} for certain 
types of concept spaces $\C$, but not all; the interested reader is referred
to the recent article of \citet*{hanneke:15b} for discussions relevant to this 
type of gap, and constructions of concept spaces which (one can easily verify)
span this gap: that is, for some spaces $\C$ the lower bound is tight, while for other spaces $\C$ the upper bound is tight, up to numerical constant factors.

An immediate corollary of Theorem~\ref{thm:erm-lb} and Corollary~\ref{cor:erm-star} is that 
$\s < \infty$ is \emph{necessary and sufficient} for arbitrary consistent learners 
to achieve $O(1/m)$ rates.
Formally, for any $\conf \in (0,1)$, let $R_{m}(\conf)$ denote the smallest value
such that, for all $\Px$ and all $\target \in \C$, with probability at least $1-\conf$, 
$\sup_{h \in V_{m}} \er(h) \leq R_{m}(\conf)$.
Also let $\bar{R}_{m}$ denote the supremum value of $\E\left[ \sup_{h \in V_{m}} \er(h) \right]$
over all $\Px$ and all $\target \in \C$.
We have the following corollary (which applies to any $\C$ with $0 < \dim < \infty$).

\begin{corollary}
\label{cor:erm-nec-suf}
$\bar{R}_{m} = \Theta\left( \frac{1}{m} \right)$
if and only if $\s < \infty$, and otherwise $\bar{R}_{m} = \Theta\left( \frac{\Log(m)}{m} \right)$.
Likewise, $\forall \conf \in (0,1/100)$,
$R_{m}(\conf) = \Theta\left( \frac{1}{m} \right)$ 
if and only if $\s < \infty$, and otherwise $R_{m}(\conf) = \Theta\left( \frac{\Log(m)}{m} \right)$.
\end{corollary}

\subsection{Using Subregions Smaller than the Region of Disagreement}
\label{sec:subregions}

In recent work, \citet*{zhang:14} have proposed a general active learning strategy,
which revises the CAL strategy so that the algorithm only requests a label if the 
corresponding $X_{m}$ is in a well-chosen \emph{subregion} of $\DIS(V_{m-1})$.
This general idea was first explored in the more-specific context of learning
linear separators under a uniform distribution by \citet*{balcan:07} (see also \citealp*{dasgupta:05b}, for related arguments).
Furthermore, following up on \citet*{balcan:07}, the work of \citet*{balcan:13} 
has also used this subregion idea to argue that any consistent learning algorithm
achieves the optimal sample complexity (up to constants) for the problem of 
learning linear separators under isotropic log-concave distributions.  In this section, 
we combine the abstract perspective of \citet*{zhang:14} with our general bounding
technique, to generalize the result of \citet*{balcan:13} by expressing a bound holding
for arbitrary concept spaces $\C$, distributions $\Px$, and target functions $\target \in \C$.
First, we need to introduce the following complexity measure $\zc_{c}(r_{0})$ based on the work of \citet*{zhang:14}.
As was true of $\dc(r_{0})$ above, this complexity measure $\zc_{c}(r_{0})$ generally depends on both $\Px$ and $\target$.

\begin{definition}
\label{def:zc}
For any nonempty set $\H$ of classifiers, and any $\eta \geq 0$, 
letting $X \sim \Px$, define
\begin{multline*}
\Phi(\H,\eta) = \min\left\{ \E[\gamma(X)] : \sup_{h \in \H} \E\left[ \ind[h(X)=+1]\zeta(X) + \ind[h(X)=-1]\xi(X) \right] \leq \eta, \right.
\\ \left. \phantom{\sup_{h \in \H}} \text{ where } \forall x \in \X, \gamma(x) + \zeta(x) + \xi(x) = 1 \text{ and } \gamma(x),\zeta(x),\xi(x) \in [0,1] \right\}.
\end{multline*}
Then, for any $r_{0} \in [0,1)$ and $c > 1$, define
\begin{equation*}
\zc_{c}(r_{0}) = \sup_{r_{0} < r \leq 1} \frac{\Phi(\Ball(\target,r),r/c)}{r} \lor 1.
\end{equation*}
\end{definition}

One can easily observe that, for the optimal choices of $\gamma$, $\zeta$, and $\xi$ in the definition of $\Phi$, 
we have $\gamma(x) = 0$ for (almost every) $x \notin \DIS(\H)$.
In the special case that $\gamma$ is binary-valued, 
the aforementioned well-chosen ``subregion'' of $\DIS(\H)$ corresponds to the set $\{x : \gamma(x) = 1\}$.
In general, the definition also allows for $\gamma(x)$ values in between $0$ and $1$, 
in which case $\gamma$ essentially re-weights the conditional distribution $\Px(\cdot|\DIS(\H))$.\footnote{Allowing 
these more-general values of $\gamma(x)$ typically does not affect the qualitative behavior of the minimal $\E[\gamma(X)]$ value; 
for instance, we argue in Lemma~\ref{lem:zc-binary} of Appendix~\ref{app:noise} that the minimal $\E[\gamma(X)]$ value 
achievable under the additional constraint that $\gamma(x) \in \{0,1\}$ is at most $2 \Phi(\H,\eta/2)$.  Thus, 
we do not lose much by thinking of $\Phi(\H,\eta)$ as describing the measure of a subregion of $\DIS(\H)$.}
As an example where this quantity is informative, \citet*{zhang:14} argue that, for $\C$ the class of homogeneous linear separators in $\reals^{k}$ ($k \in \nats$)
and $\Px$ any isotropic log-concave distribution, $\zc_{c}(r_{0}) \lesssim \Log(c)$ (which follows readily from arguments of \citealp*{balcan:13}).
Furthermore, they observe that $\zc_{c}(r_{0}) \leq \dc(r_{0})$ for any $c \in (1,\infty]$.

\citet*{zhang:14} propose the above quantities for the purpose of 
proving a bound on the label complexity of a certain active learning algorithm,
inspired both by the work of \citet*{balcan:07} on active learning with linear separators,
and by the connection between selective classification and active learning exposed by \citet*{el-yaniv:12}.
However, since the idea of using well-chosen subregions of $\DIS(V_{m})$ in the analysis of consistent learning algorithms
lead \citet*{balcan:13} to derive improved sample complexity bounds for these methods in the case of 
linear separators under isotropic log-concave distributions, and since the corresponding improvements
for active learning are reflected in the general results of \citet*{zhang:14}, it is natural to ask whether 
the sample complexity improvements of \citet*{balcan:13} for that special scenario can also be extended
to the general case by incorporating the complexity measure $\zc_{c}(r_{0})$.  
Here we provide such an extension.  Specifically, following the same basic strategy 
from Theorem~\ref{thm:PDIS}, with a few adjustments inspired by \citet*{zhang:14} to allow us 
to consider only a \emph{subregion} of $\DIS(V_{m})$ in the argument (or more generally, a reweighting
of the conditional distribution $\Px(\cdot | \DIS(V_{m}))$), we arrive at the following result.  The proof 
is included in Appendix~\ref{app:subregions}.

\begin{theorem}
\label{thm:vc-erm-subregion}
For any $\conf \in (0,1)$ and $m \in \nats$, for $c = 16$, with probability at least $1-\conf$,
\begin{equation*}
\sup_{h \in V_{m}} \er(h) \leq \frac{21}{m} \left( \dim \ln\left( 83 \zc_{c}\left(\frac{\dim}{m}\right) \right) + 3 \ln\left(\frac{4}{\conf}\right) \right).
\end{equation*}
\end{theorem}

In particular, in the special case of $\C$ the space of homogeneous linear separators on $\reals^{k}$, 
and $\Px$ an isotropic log-concave distribution, Theorem~\ref{thm:vc-erm-subregion} recovers
the bound of \citet*{balcan:13} proportional to $\frac{1}{m}(k + \Log(\frac{1}{\conf}))$ as a special case.
Furthermore, one can easily construct scenarios (concept spaces $\C$, distributions $\Px$, and target functions $\target \in \C$) 
where $\zc_{c}\left(\frac{\dim}{m}\right)$ is bounded while $\frac{\hat{\cs}_{1:m}}{\dim} = \frac{m}{\dim}$ almost surely 
(e.g., $\C = \{ x \mapsto 2 \ind_{\{t\}}(x) - 1 : t \in \reals \}$ the class of impulse functions on $\reals$, and $\Px$ uniform on $(0,1)$), so that Theorem~\ref{thm:vc-erm-subregion} 
sometimes reflects a significant improvement over Theorem~\ref{thm:vc-erm}.

One can easily show that we always have $\zc_{c}(r_{0}) \leq \left(1-\frac{1}{c}\right) \dc(r_{0})$,
so that Theorem~\ref{thm:vc-erm-subregion} is never worse than the bound \eqref{eqn:gk-dc} of \citet{gine:06}.
However, we argue in Appendix~\ref{sec:sup-zc} that $\forall c \geq 2$, $\forall r_{0} \in [0,1)$,
\begin{equation}
\label{eqn:sup-zc}
\left( 1 - \frac{1}{c} \right) \min\left\{ \s, \frac{1}{r_{0}} - \frac{1}{c-1} \right\} \leq \sup_{\Px} \sup_{\target \in \C} \zc_{c}(r_{0}) \leq \left(1 - \frac{1}{c}\right) \min\left\{ \s, \frac{1}{r_{0}} \right\}.
\end{equation}
Thus, at least in some cases, the bound in Theorem~\ref{thm:vc-erm} is smaller than that in Theorem~\ref{thm:vc-erm-subregion}
(as the former leads to Corollary~\ref{cor:erm-star} in the worst case, while the latter leads to \eqref{eqn:gk-dist-free} in the worst case).
In fact, if we let $\zc_{c}^{\zo}(r_{0})$ be defined identically to $\zc_{c}(r_{0})$, except that $\gamma$ is restricted to be $\{0,1\}$-valued
in Definition~\ref{def:zc}, then the same argument from Appendix~\ref{sec:sup-zc} reveals that, for any $c \geq 4$, 
\begin{equation*}
\sup_{\Px}\sup_{\target \in \C} \zc_{c}^{\zo}(r_{0}) = \min\left\{ \s, \frac{1}{r_{0}} \right\}.
\end{equation*}

\paragraph{Relation to the Doubling Dimension:}
To further put Theorem~\ref{thm:vc-erm-subregion} in context, we also note that it is possible to relate 
$\zc_{c}(r_{0})$ to the \emph{doubling dimension}.  Specifically, the doubling dimension (also known as the local metric entropy)
of $\C$ at $\target$ under $\Px$, denoted $\dd(r_{0})$, is defined as
\begin{equation*}
\dd(r_{0}) = \max_{r \geq r_{0}} \log_{2}\left( \covering\!\left( r/2, \Ball(\target,r), \Px\right) \right),
\end{equation*}
for $r_{0} > 0$,
where $\covering(r/2, \Ball(\target,r), \Px)$ is 
the smallest $n \in \nats$ such that there exist classifiers $h_{1},\ldots,h_{n}$
for which $\sup_{h \in \Ball(\target,r)} \min_{1 \leq i \leq n} \Px( x : h(x) \neq h_{i}(x) ) \leq r/2$,
known as the $(r/2)$-covering number for $\Ball(\target,r)$ under the $L_{1}(\Px)$ pseudo-metric.
The notion of doubling dimension has been explored in a variety of contexts in the literature \citep*[e.g.,][]{lecam:73,yang:99b,gupta:03,long:09}.
We always have $\dd(r_{0}) \lesssim \dim \Log(1/r_{0})$ \citep*{haussler:95}, though it can often be smaller than this,
and in many interesting contexts, it can even be bounded by an $r_{0}$-invariant value \citep*{long:09}.
\citet*{long:09} construct a particular $\Px$-dependent learning rule $\alg$ such that, 
for any $\eps,\conf \in (0,1)$, and any
\begin{equation}
\label{eqn:doubling-good-bound}
m \gtrsim \frac{1}{\eps} \left( \dd(\eps/c) + \Log\left(\frac{1}{\conf}\right) \right),
\end{equation}
where $c > 0$ is a specific constant,
with probability at least $1-\conf$, the classifier $\hat{h}_{m} = \alg(\L_{m})$ satisfies $\er(\hat{h}_{m}) \leq \eps$.
They also establish a weaker bound holding for all consistent learning rules:
for any $\eps > 0$, 
denoting $\eps_{0} = \eps \exp\left\{-\sqrt{\ln(1/\eps)}\right\}$,
for any 
\begin{equation}
\label{eqn:doubling-loose-bound}
m \gtrsim \frac{1}{\eps} \left( \max\{ \dim, \dd(\eps_{0}) \} \sqrt{\Log\left(\frac{1}{\eps}\right)} + \Log\left(\frac{1}{\conf}\right) \right),
\end{equation}
with probability at least $1-\conf$, $\sup_{h \in V_{m}} \er(h) \leq \eps$.

\citet*{hanneke:15b} have proven that we always have
$\dd(r_{0}) \lesssim \dim \Log( \dc(r_{0}) )$,
which immediately implies that \eqref{eqn:doubling-good-bound}
is never larger than the bound \eqref{eqn:gk-dc} for consistent learning rules (aside from constant factors),
though \eqref{eqn:gk-dc} may often offer improvements over the weaker bound \eqref{eqn:doubling-loose-bound}. 
Here we note that a related argument can be used to prove the following
bound: for any $r_{0} > 0$ and $c \geq 8$,
\begin{equation}
\label{eqn:doubling-zc-bound}
\dd(r_{0}) \leq 2 \dim \log_{2}( 96 \zc_{c}(r_{0}) ).
\end{equation}
In particular, this implies that the bound \eqref{eqn:doubling-good-bound}
is never larger than the bound in Theorem~\ref{thm:vc-erm-subregion} 
for consistent learning rules (aside from constant factors),
though again Theorem~\ref{thm:vc-erm-subregion} may often offer improvements
over the weaker bound \eqref{eqn:doubling-loose-bound}.
We also note that, combined with the above mentioned result of \citet*{zhang:14}
that $\zc_{c}(r_{0}) \lesssim \Log(c)$
for $\C$ the class of homogeneous linear separators in $\reals^{k}$ and $\Px$ any isotropic log-concave distribution, 
\eqref{eqn:doubling-zc-bound} immediately implies a bound $\dd(r_{0}) \lesssim k$ for the doubling dimension
in this scenario (recalling that $\dim = k$ for this class, from \citealp*{cover:65}), which appears to be new to the literature.
The proof of \eqref{eqn:doubling-zc-bound} is included in Appendix~\ref{app:doubling-zc-bound}.

\section{Learning with Noise}
\label{sec:noise}

The previous sections demonstrate how variations on the basic technique of \citet*{hanneke:thesis}
lead to refined analyses of certain learning methods, in the \emph{realizable case}, where $\exists \target \in \C$ with $\er(\target) = 0$.
We can also apply this general technique in the more-general setting of learning with \emph{classification noise}.
Specifically, in this setting, there is a \emph{joint} distribution $\PXY$ on $\X \times \Y$,
and the error rate of a classifier $h$ is then defined as $\er(h) = \P( h(X) \neq Y )$ for $(X,Y) \sim \PXY$.
As above, we denote by $\Px$ the marginal distribution $\PXY(\cdot \times \Y)$ on $\X$.
We then let $(X_{1},Y_{1}),(X_{2},Y_{2}),\ldots$ denote a sequence of independent $\PXY$-distributed random samples,
and denoting $\L_{m} = \{(X_{1},Y_{1}),\ldots,(X_{m},Y_{m})\}$,
we are interested in obtaining bounds on $\er(\hat{h}_{m}) - \inf_{f \in \C} \er(f)$ (the \emph{excess error rate}),
where $\hat{h}_{m} = \alg(\L_{m})$ for some learning rule $\alg$.
This notation is consistent with the above, which represents the special case 
in which $\P( Y = \target(X) | X ) = 1$ almost surely (i.e., the \emph{realizable case}).
While there are various noise models commonly studied in the literature,
for our present discussion, we are primarily interested in two such models.

\begin{itemize}
\item[$\bullet$] For $\bound \in (0,1/2)$, $\PXY$ satisfies the $\bound$-\emph{bounded noise} condition
if $\exists \agtarget \in \C$ such that $\P( Y \neq \agtarget(X) | X ) \leq \bound$ almost surely, where $(X,Y) \sim \PXY$.
\item[$\bullet$] For $\tsybca \in [1,\infty)$ and $\tsyba \in [0,1]$, $\PXY$ satisfies 
the $(\tsybca,\tsyba)$-\emph{Bernstein class} condition if, for $\agtarget = \argmin_{h \in \C} \er(h)$,\footnote{For simplicity, we suppose
the minimum error rate is achieved in $\C$.  One can easily generalize the condition to the more-general case 
where the minimum is not necessarily achieved \citep*[see e.g.,][]{koltchinskii:06},
and the results below continue to hold with only minor technical adjustments to the proofs.}
we have $\forall h \in \C$, $\Px(x : h(x) \neq \agtarget(x)) \leq \tsybca \left( \er(h) - \er(\agtarget) \right)^{\tsyba}$.
\end{itemize}

Note that $\bound$-bounded noise distributions also satisfy the Bernstein class condition, 
with $\tsyba = 1$ and $\tsybca = \frac{1}{1-2\bound}$.
These two conditions have been studied extensively in both the passive and active learning literatures
\citep*[e.g.,][]{mammen:99,tsybakov:04,bartlett:06,massart:06,koltchinskii:06,bartlett:06b,gine:06,hanneke:thesis,hanneke:11a,hanneke:12a,hanneke:fntml,el-yaniv:11,ailon:14,zhang:14,hanneke:15b}.
In particular, for passive learning, much of this literature focuses on the analysis of \emph{empirical risk minimization}.
Specifically, for any $m \in \nats$ and $L \in (\X \times \Y)^{m}$, 
define $\ERM(\C,L) = \{ h \in \C : \er_{L}(h) = \min_{g \in \C} \er_{L}(g) \}$, the set of \emph{empirical risk minimizers}.
\citet*{massart:06} established that, for any $\PXY$ satisfying the $(\tsybca,\tsyba)$-Bernstein class condition, 
for any $\conf \in (0,1)$, with probability at least $1-\conf$,
\begin{equation}
\label{eqn:mn-bernstein}
\sup_{h \in \ERM(\C,\L_{m})} \er(h) - \inf_{h \in \C} \er(h) \lesssim \left( \frac{\tsybca \left( \dim \Log\left( \frac{1}{\tsybca} \left(\frac{m}{\tsybca\dim}\right)^{\frac{\tsyba}{2-\tsyba}} \right) + \Log\left(\frac{1}{\conf}\right) \right)}{m}\right)^{\frac{1}{2-\tsyba}}.
\end{equation}
In the case of $\bound$-bounded noise, \citet*{gine:06} showed that the logarithmic factor $\Log\left( \frac{m (1-2\bound)^{2}}{\dim} \right)$
implied by \eqref{eqn:mn-bernstein} can be replaced by $\Log\left( \dc\left( \frac{\dim}{m (1-2\bound)^{2}} \right) \right)$,
where the disagreement coefficient $\dc(r_{0})$ is defined as above, except with $\agtarget$ in place of $\target$ in the definition.
Furthermore, applying their arguments to the general case of the $(\tsybca,\tsyba)$-Bernstein class condition
(see \citealp*{hanneke:12b}, for an explicit derivation), one arrives at the fact that, with probability at least $1-\conf$,
\begin{equation}
\label{eqn:gk-bernstein}
\sup_{h \in \ERM(\C,\L_{m})} \er(h) - \inf_{h \in \C} \er(h) \lesssim \left( \frac{\tsybca \left( \dim \Log\left( \dc\left(\tsybca \left(\frac{\tsybca\dim}{m}\right)^{\frac{\tsyba}{2-\tsyba}} \right) \right) + \Log\left(\frac{1}{\conf}\right) \right)}{m}\right)^{\frac{1}{2-\tsyba}}.
\end{equation}

Since \citet*{hanneke:15b} have argued that $\dc(r_{0}) \leq \min\left\{ \s, \frac{1}{r_{0}} \right\}$ (with equality in the worst case),
\eqref{eqn:gk-bernstein} further implies that, with probability at least $1-\conf$,
\begin{equation}
\label{eqn:star-bernstein}
\sup_{h \in \ERM(\C,\L_{m})} \er(h) - \inf_{h \in \C} \er(h) \lesssim \left( \frac{\tsybca \left( \dim \Log\left( \min\left\{ \s, \frac{1}{\tsybca} \left(\frac{m}{\tsybca\dim}\right)^{\frac{\tsyba}{2-\tsyba}} \right\} \right) + \Log\left(\frac{1}{\conf}\right) \right)}{m}\right)^{\frac{1}{2-\tsyba}}.
\end{equation}
Via the same integration argument used in Corollary~\ref{cor:erm-star}, this further implies
\begin{equation}
\label{eqn:star-bernstein-expectation}
\E\left[ \sup_{h \in \ERM(\C,\L_{m})} \er(h) \right] - \inf_{h \in \C} \er(h) \lesssim \left( \frac{\tsybca \dim \Log\left( \min\left\{ \s, \frac{1}{\tsybca} \left(\frac{m}{\tsybca\dim}\right)^{\frac{\tsyba}{2-\tsyba}} \right\} \right)}{m}\right)^{\frac{1}{2-\tsyba}}.
\end{equation}

It is worth noting that the bound \eqref{eqn:star-bernstein} does not quite recover the bound of Corollary~\ref{cor:erm-star} 
in the realizable case (corresponding to $\tsybca = \tsyba = 1$).  Specifically, it contains a logarithmic factor 
$\Log\left(\frac{\min\{\s\dim,m\}}{\dim}\right)$, rather than $\Log\left(\frac{\min\{\s,m\}}{\dim}\right)$.  
I conjecture that this logarithmic factor in \eqref{eqn:star-bernstein} can generally be improved
so that, for any $\tsybca$ and $\tsyba$, it is bounded by a numerical constant whenever $\s \lesssim \dim$.
This problem is intimately connected to a conjecture in active learning, proposed by \citet*{hanneke:15b}, 
concerning the joint dependence on $\s$ and $\dim$ in the minimax label complexity of active learning under
the Bernstein class condition.

\subsection{Necessary and Sufficient Conditions for $\boldsymbol{1/m}$ Minimax Rates under Bounded Noise}
\label{sec:nec-suf-noise}

In the case of bounded noise (where $\tsybca = \frac{1}{1-2\bound}$ and $\tsyba = 1$), \citet*{massart:06} have shown that for some concept spaces $\C$, 
the factor $\Log\left( \frac{m (1-2\bound)^{2}}{\dim} \right)$ is present even in a \emph{lower} bound on the minimax excess error rate,
so that it cannot generally be removed.
\citet*{raginsky:11} further discuss a range of lower bounds on the minimax excess error rate for various spaces $\C$ they construct,
where the appropriate factor ranges between $\Log\left( \frac{m (1-2\bound)^{2}}{\dim} \right)$ at the highest,
to a constant factor at the lowest.
The bound in \eqref{eqn:star-bernstein} provides a \emph{sufficient} condition for all empirical risk minimization algorithms
to achieve excess error rate with  $O(1/m)$ asymptotic dependence on $m$ under $\bound$-bounded noise: namely $\s < \infty$.
Recall that this condition was both sufficient \emph{and necessary} for $O(1/m)$ error rates to be achievable 
by every algorithm of this type for all distributions in the realizable case (Corollary~\ref{cor:erm-nec-suf}).  It is therefore natural to wonder whether this 
remains the case for bounded noise as well.  In this section, we find this is indeed the case.
In fact, following a generalization of the technique of \citet*{raginsky:11} explored by \citet*{hanneke:15b} for active learning,
we are here able to provide a general lower bound on the \emph{minimax} excess error rate of passive learning, expressed in terms of $\s$.  
This immediately implies a corollary that $\s < \infty$ is both \emph{necessary} and \emph{sufficient} for the \emph{minimax optimal} bound on the excess error 
rate to have dependence on $m$ of $\Theta(1/m)$ under bounded noise, and otherwise the minimax optimal bound is $\Theta(\Log(m)/m)$.
Note that this is a stronger type of result than that given by Corollary~\ref{cor:erm-nec-suf}, as the lower bounds here apply to \emph{all} learning rules.
Formally, we have the following theorem.  The proof is included in Appendix~\ref{app:bounded-lower-proof}.

\begin{theorem}
\label{thm:bounded-lower-bound}
For any $\bound \in (0,1/2)$, $m \in \nats$, and $\conf \in (0,1/24]$, 
for any (passive) learning rule $\alg$, there exists a choice of $\PXY$ satisfying the $\bound$-bounded noise condition
such that, denoting $\hat{h}_{m} = \alg(\L_{m})$, with probability greater than $\conf$,
\begin{equation*}
\er(\hat{h}_{m}) - \inf_{h \in \C} \er(h) \gtrsim \frac{\dim + \bound \Log\left( \min\left\{ \s, (1-2\bound)^{2} m \right\} \right) + \Log\left(\frac{1}{\conf}\right)}{(1-2\bound) m} \land (1-2\bound). 
\end{equation*}
Furthermore, 
\begin{equation*}
\E\left[ \er(\hat{h}_{m}) \right] - \inf_{h \in \C} \er(h) \gtrsim \frac{\dim + \bound \Log\left( \min\left\{ \s, (1-2\bound)^{2} m \right\} \right)}{(1-2\bound) m} \land (1-2\bound).
\end{equation*}
\end{theorem}

As was the case in Theorem~\ref{thm:erm-lb}, the joint dependence on $\dim$ and $m$ in this lower bound
does not match that in \eqref{eqn:star-bernstein} in the case $\s=\infty$.  One can show that the dependence
in this lower bound can be made to nearly match that in \eqref{eqn:star-bernstein} for certain specially-constructed
spaces $\C$ under bounded noise \citep*{massart:06,raginsky:11,hanneke:15b} (the only gap being that $\s$ is replaced by $\s/\dim$ in \eqref{eqn:star-bernstein} to obtain the lower bound); 
however, there also exist spaces $\C$ where these lower bounds 
are nearly tight (for $\bound$ bounded away from $0$), so that they cannot be improved in the general case (see \citealp*{hanneke:15b}, 
for construction of spaces $\C$ with arbitrary $\dim$ and $\s$, for which one can show this is the case).

As mentioned above, an immediate corollary of Theorem~\ref{thm:bounded-lower-bound}, in combination with \eqref{eqn:star-bernstein}, 
is that $\s < \infty$ is necessary and sufficient for the minimax excess error rate to have $O(1/m)$ dependence on $m$
for bounded noise.  Formally, for $m \in \nats$, $\bound \in [0,1/2)$, and $\conf \in (0,1)$, let $R_{m}(\conf,\bound)$ denote the 
smallest value such that there exists a learning rule $\alg$ for which, for all $\PXY$ satisfying the $\bound$-bounded noise condition, 
with probability at least $1-\conf$, $\er(\alg(\L_{m})) - \inf_{h \in \C} \er(h) \leq R_{m}(\conf,\bound)$.
Also let $\bar{R}_{m}(\bound)$ denote the smallest value such that there exists a learning rule $\alg$ for which,
for all $\PXY$ satisfying the $\bound$-bounded noise condition, $\E[\er(\alg(\L_{m}))] - \inf_{h \in \C} \er(h) \leq \bar{R}_{m}(\bound)$.
We have the following corollary (which applies to any $\C$ with $0 < \dim < \infty$).

\begin{corollary}
\label{cor:bounded-nec-suf}
Fix any $\bound \in (0,1/2)$.
$\bar{R}_{m}(\bound) = \Theta\left(\frac{1}{m}\right)$ if and only if $\s < \infty$, and otherwise $\bar{R}_{m}(\bound) = \Theta\left(\frac{\Log(m)}{m}\right)$.
Likewise, $\forall \conf \in (0,1/24]$, $R_{m}(\conf,\bound) = \Theta\left(\frac{1}{m}\right)$ if and only if $\s < \infty$, and otherwise $R_{m}(\conf,\bound) = \Theta\left(\frac{\Log(m)}{m}\right)$.
\end{corollary}

Again, note that this is a stronger type of result than Corollary~\ref{cor:erm-nec-suf} above, 
which only found $\s < \infty$ as necessary and sufficient for a \emph{particular family} 
of learning rules to obtain $O(1/m)$ rates.  In contrast, this result applies even to the 
\emph{minimax optimal} learning rule.

We conclude this section by noting that the technique leading to Theorem~\ref{thm:bounded-lower-bound}
appears not to straightforwardly extend to the general $(\tsybca,\tsyba)$-Bernstein class condition.
Indeed, though one can certainly exhibit specific spaces $\C$ for which the minimax excess risk has
$\Theta\left( \left(\frac{\Log(m)}{m}\right)^{\frac{1}{2-\tsyba}} \right)$ dependence on $m$ (e.g., impulse functions on $\reals$; 
see \citealp*{hanneke:15b}, for related discussions),
it appears a much more challenging problem to construct general 
lower bounds describing the range of possible dependences on $m$.  
Thus, the more general question of establishing necessary and sufficient conditions
for $O\left(1/m^{\frac{1}{2-\tsyba}}\right)$ excess error rates under the $(\tsybca,\tsyba)$-Bernstein class condition 
remains open.

\subsection{Using Subregions to Achieve Improved Excess Error Bounds}
\label{sec:noisy-zc}

In general, note that plugging into \eqref{eqn:gk-bernstein} the parameters $\tsybca = \tsyba = 1$ admitted by the realizable case, \eqref{eqn:gk-bernstein} recovers the bound \eqref{eqn:gk-dc}.
Recalling that we were able to refine the bound \eqref{eqn:gk-dc} via techniques from the subregion-based analysis of \citet*{zhang:14},
yielding Theorem~\ref{thm:vc-erm-subregion} above, it is natural to consider whether we might be able to refine \eqref{eqn:gk-bernstein} in a similar way.
We find that this is indeed the case, though we establish this refinement for a different learning rule (described in Appendix~\ref{app:zcnoise-proof}).
Letting $c=128$, for any $r_{0} \in [0,1)$, $\tsybca \geq 1$ and $\tsyba \in (0,1]$,
define
\begin{equation*}
\maxzc_{\tsybca,\tsyba}(r_{0}) = \sup_{h \in \C} \sup_{r > r_{0}} \frac{\Phi\left(\Ball(h, r), (r/\tsybca)^{1/\tsyba}/c \right) }{r} \lor 1.
\end{equation*}
For completeness, also define $\maxzc_{\tsybca,\tsyba}(r_{0}) = 1$ for any $r_{0} \geq 1$, $\tsybca \geq 1$, and $\tsyba \in [0,1]$.
We have the following theorem.

\begin{theorem}
\label{thm:zc-noise}
For any $\tsybca \geq 1$ and $\tsyba \in (0,1]$, for any probability measure $\Px$ over $\X$,
for any $\conf \in (0,1)$, there exists a learning rule $\alg$ such that,
for any $\PXY$ satisfying the $(\tsybca,\tsyba)$-Bernstein class condition
with marginal distribution $\Px$ over $\X$, for any $m \in \nats$,
letting $\hat{h}_{m} = \alg(\L_{m})$, with probability at least $1-\conf$, 
\begin{equation*}
\er(\hat{h}_{m}) - \inf_{h \in \C} \er(h) \lesssim \left( \frac{\tsybca \left( \dim \Log\left( \maxzc_{\tsybca,\tsyba}\left(\tsybca \left( \frac{\tsybca \dim}{m} \right)^{\frac{\tsyba}{2-\tsyba}}\right) \right) + \Log\left(\frac{1}{\conf}\right) \right)}{m} \right)^{\frac{1}{2-\tsyba}}.
\end{equation*}
\end{theorem}

The proof is included in Appendix~\ref{app:zcnoise-proof}.
We should emphasize that the bound in Theorem~\ref{thm:zc-noise} is established
for a particular learning method (described in Appendix~\ref{app:zcnoise-proof}), \emph{not} for 
empirical risk minimization.  Thus, whether or not this bound can be established for 
the general family of empirical risk minimization rules remains an open question.
We should also note that $\maxzc_{\tsybca,\tsyba}(r_{0})$ involves a supremum over $h \in \C$
only so that we may allow the algorithm to explicitly depend on $\maxzc_{\tsybca,\tsyba}(r_{0})$
(noting that, as stated, Theorem~\ref{thm:zc-noise} allows $\Px$-dependence in the algorithm).
It is conceivable that this dependence on $\maxzc_{\tsybca,\tsyba}(r_{0})$ in $\alg$ can be removed, 
for instance via a stratification and model selection technique \citep*[see e.g.,][]{koltchinskii:06}, 
in which case this supremum over $h$ would be replaced by fixing $h = \agtarget$.

We conclude this section with some basic observations about the bound in Theorem~\ref{thm:zc-noise}.
First, in the special case of $\C$ the class of homogeneous linear separators on $\reals^{k}$ and $\Px$
any isotropic log-concave distribution,
Theorem~\ref{thm:zc-noise} recovers a bound of \citet*{balcan:13} (established for a closely related method),
since a result of \citet*{zhang:14} implies
$\maxzc_{\tsybca,\tsyba}(\tsybca \eps^{\tsyba}) \lesssim \Log\left( \tsybca \eps^{\tsyba-1} \right)$
in that case. 
Additionally, we note that a result similar to \eqref{eqn:star-bernstein} also generally holds for the method $\alg$ from Theorem~\ref{thm:zc-noise}, since \eqref{eqn:sup-zc} implies we always have
\begin{equation*}
\frac{\Phi(\Ball(h,\tsybca \eps^{\tsyba}), \eps/c)}{\tsybca \eps^{\tsyba}} \leq \left( 1 - \frac{1}{ c \tsybca } \eps^{1-\tsyba} \right) \min\left\{ \s, \frac{1}{\tsybca \eps^{\tsyba}} \right\}.
\end{equation*}

\appendix

\section{A Technical Lemma}
\label{app:technical-lemmas}

The following lemma is useful in the proofs of several of the main results of this paper.\footnote{This lemma and proof also appear in a sibling paper \citep*{hanneke:16a}.}

\begin{lemma}
\label{lem:log-factors-abstract}
For any $\a,\b,c_{1} \in [1,\infty)$ and $c_{2} \in [0,\infty)$, 
\begin{equation*}
\a \ln\left( c_{1} \left( c_{2} + \frac{\b}{\a} \right) \right)
\leq \a \ln\left( c_{1} (c_{2} + e) \right) + \frac{1}{e} \b.
\end{equation*}
\end{lemma}
\begin{proof}
By subtracting $\a \ln(c_{1})$ from both sides, we see that it suffices to verify that
$\a \ln\left( c_{2} + \frac{\b}{\a} \right) \leq \a \ln( c_{2} + e ) + \frac{1}{e} \b$.
If $\frac{\b}{\a} \leq e$, then monotonicity of $\ln(\cdot)$ implies 
\begin{equation*}
\a \ln\left( c_{2} + \frac{\b}{\a} \right)
\leq \a \ln( c_{2} + e ),
\end{equation*}
which is clearly no greater than $\a \ln( c_{2} + e ) + \frac{1}{e} \b$.
On the other hand, if $\frac{\b}{\a} > e$, then 
\begin{equation*}
\a \ln\left( c_{2} + \frac{\b}{\a} \right)
\leq \a \ln\left( \max\{ c_{2}, 2 \} \frac{\b}{\a} \right)
= \a \ln\left( \max\{ c_{2}, 2 \} \right) + \a \ln\left( \frac{\b}{\a} \right).
\end{equation*}
The first term in the rightmost expression is at most $\a \ln( c_{2} + 2 ) \leq \a \ln( c_{2} + e )$.
The second term in the rightmost expression can be rewritten as $\b \frac{\ln( \b/\a )}{\b / \a}$.
Since $x \mapsto \ln(x)/x$ is nonincreasing on $(e,\infty)$, in the case $\frac{\b}{\a} > e$ this is at most $\frac{1}{e} \b$.
Together, we have that 
\begin{equation*}
\a \ln\left( c_{2} + \frac{\b}{\a} \right)
\leq \a \ln( c_{2} + e ) + \frac{1}{e} \b
\end{equation*}
in this case as well.
\end{proof}

\section{Proof of Theorem~\ref{thm:PDIS}}
\label{app:cal}

Here we present the proof of Theorem~\ref{thm:PDIS}.

\begin{proof}[of Theorem~\ref{thm:PDIS}]
The structure of the proof is nearly identical to that of Theorem~\ref{thm:monotone-compression},
with only a few small changes to account for the fact that $\hat{\cs}_{1:m}$ 
depends on the specific samples, and in particular, on the order of the samples.

The proof proceeds by induction on $m$.
Since $\Px(\DIS(V_{m})) \leq 1$ always, 
the stated bound is trivially satisfied for all $\conf \in (0,1)$ if $m \leq 16$.
Now, as an inductive hypothesis, fix any integer $m \geq 17$ such that,
$\forall \conf \in (0,1)$, with probability at least $1-\conf$,
\begin{equation*}
\Px(\DIS(V_{\lfloor m/2 \rfloor})) \leq \frac{16}{\lfloor m/2 \rfloor} \left( 2 \hat{\cs}_{1:\lfloor m/2 \rfloor} + \ln\left(\frac{3}{\conf}\right) \right).
\end{equation*}

Fix any $\conf \in (0,1)$.
Define
\begin{equation*}
N = \left| \left\{ X_{\lfloor m/2 \rfloor + 1},\ldots,X_{m} \right\} \cap \DIS(V_{\lfloor m/2 \rfloor}) \right|,
\end{equation*}
and enumerate the elements of $\{ X_{\lfloor m/2 \rfloor + 1},\ldots, X_{m} \} \cap \DIS(V_{\lfloor m/2 \rfloor})$
as $\hat{X}_{1},\ldots,\hat{X}_{N}$.
Let $\L_{t} = \{ (X_{1},\target(X_{1})), \ldots, (X_{t},\target(X_{t})) \}$ for every $t \in [m]$,
and 
$\hat{\cs}^{\prime}_{m} = \left| \hat{C}_{\L_{m}} \setminus \L_{\lfloor m/2 \rfloor} \right|$,
and enumerate as $i_{1}^{\prime},\ldots,i_{\hat{\cs}^{\prime}_{m}}^{\prime}$ the indices $i \in \{ \lfloor m/2 \rfloor + 1, \ldots, m \}$
with $(X_{i},\target(X_{i})) \in \hat{C}_{\L_{m}} \setminus \L_{\lfloor m/2 \rfloor}$.
In particular, note that $\hat{\cs}^{\prime}_{m} \leq \hat{\cs}_{m}$, and
$\hat{C}_{\L_{m}} \subseteq \L_{\lfloor m/2 \rfloor} \cup \{ (X_{i_{1}^{\prime}},\target(X_{i_{1}^{\prime}})),\ldots,(X_{i_{\hat{\cs}^{\prime}_{m}}^{\prime}},\target(X_{i_{\hat{\cs}^{\prime}_{m}}^{\prime}})) \}$,
so that
$$\C\left[ \L_{\lfloor m/2 \rfloor} \cup \{ (X_{i_{1}^{\prime}},\target(X_{i_{1}^{\prime}})),\ldots,(X_{i_{\hat{\cs}^{\prime}_{m}}^{\prime}},\target(X_{i_{\hat{\cs}^{\prime}_{m}}^{\prime}})) \} \right] = V_{m}.$$
Next, let $\hat{\cs}^{\prime\prime}_{m} = \left| \{ j \in \left[\hat{\cs}^{\prime}_{m}\right] : X_{i_{j}^{\prime}} \in \DIS(V_{\lfloor m/2 \rfloor}) \} \right|$,
and enumerate as $i_{1}^{\prime\prime},\ldots,i_{\hat{\cs}^{\prime\prime}_{m}}^{\prime\prime}$ the indices $i \in [N]$
such that $(\hat{X}_{i},\target(\hat{X}_{i})) \in \{ (X_{i_{1}^{\prime}},\target(X_{i_{1}^{\prime}})),\ldots,(X_{i_{\hat{\cs}^{\prime}_{m}}^{\prime}},\target(X_{i_{\hat{\cs}^{\prime}_{m}}^{\prime}})) \}$.
Note that, since every $j \in \left[ \hat{\cs}^{\prime}_{m}\right]$ with $X_{i_{j}^{\prime}} \notin \DIS(V_{\lfloor m/2 \rfloor})$ 
has $h(X_{i_{j}^{\prime}}) = \target(X_{i_{j}^{\prime}})$ for every $h \in \C[ \L_{\lfloor m/2 \rfloor} \cup \{ (\hat{X}_{i_{1}^{\prime\prime}},\target(\hat{X}_{i_{1}^{\prime\prime}})),\ldots,(\hat{X}_{i_{\hat{\cs}^{\prime\prime}_{m}}^{\prime\prime}},\target(\hat{X}_{i_{\hat{\cs}^{\prime\prime}_{m}}^{\prime\prime}})) ]$
(by definition of $\DIS$ and monotonicity of $\L \mapsto \C[ \L ]$), we have
\begin{multline*}
\C[ \L_{\lfloor m/2 \rfloor} \cup \{ (\hat{X}_{i_{1}^{\prime\prime}},\target(\hat{X}_{i_{1}^{\prime\prime}})),\ldots,(\hat{X}_{i_{\hat{\cs}^{\prime\prime}_{m}}^{\prime\prime}},\target(\hat{X}_{i_{\hat{\cs}^{\prime\prime}_{m}}^{\prime\prime}})) ]
\\= \C\left[ \L_{\lfloor m/2 \rfloor} \cup \{ (X_{i_{1}^{\prime}},\target(X_{i_{1}^{\prime}})),\ldots,(X_{i_{\hat{\cs}^{\prime}_{m}}^{\prime}},\target(X_{i_{\hat{\cs}^{\prime}_{m}}^{\prime}})) \} \right] = V_{m},
\end{multline*}
so that $\DIS(V_{m})$ may be expressed as a fixed function of 
$X_{1},\ldots,X_{\lfloor m/2 \rfloor}$ and $\hat{X}_{i_{1}^{\prime\prime}},\ldots,\hat{X}_{i_{\hat{\cs}^{\prime\prime}_{m}}^{\prime\prime}}$.
Furthermore, note that the set 
$\DIS\left( \C[ \L_{\lfloor m/2 \rfloor} \cup \{ (\hat{X}_{i_{1}^{\prime\prime}},\target(\hat{X}_{i_{1}^{\prime\prime}})),\ldots,(\hat{X}_{i_{\hat{\cs}^{\prime\prime}_{m}}^{\prime\prime}},\target(\hat{X}_{i_{\hat{\cs}^{\prime\prime}_{m}}^{\prime\prime}})) ] \right)$
is invariant to permutations of the $i_{1}^{\prime\prime},\ldots,i_{\hat{\cs}^{\prime\prime}_{m}}^{\prime\prime}$ indices.

Now note that $N$
is conditionally ${\rm Binomial}(\lceil m/2 \rceil, \Px(\DIS(V_{\lfloor m/2 \rfloor})))$-distributed
given $X_{1},\ldots,X_{\lfloor m/2 \rfloor}$.
In particular, with probability one, if $P(\DIS(V_{\lfloor m/2 \rfloor})) = 0$, then $N = 0$.
Otherwise, if $\Px(\DIS(V_{\lfloor m/2 \rfloor})) > 0$, then note that $\hat{X}_{1},\ldots,\hat{X}_{N}$ are conditionally independent 
and $\Px(\cdot | \DIS(V_{\lfloor m/2 \rfloor}))$-distributed given $X_{1},\ldots,X_{\lfloor m/2 \rfloor}$ and $N$.
Thus, since $\DIS(V_{m}) \cap \{ \hat{X}_{1},\ldots,\hat{X}_{N} \} = \emptyset$ (since every $h \in V_{m}$ agrees with $\target$ on $X_{1},\ldots,X_{m}$),
combining the above with Lemma~\ref{lem:classic-compression} (applied under the conditional distribution given $X_{1},\ldots,X_{\lfloor m/2 \rfloor}$ and $N$),
combined with the law of total probability, implies that for every $\cs \in [m] \cup \{0\}$, 
with probability at least $1-\conf / (\cs+3)^{2}$, if $\hat{\cs}^{\prime\prime}_{m} = \cs$ and $N > \cs$, then 
\begin{equation*}
\Px\left(\DIS(V_{m}) \middle| \DIS(V_{\lfloor m/2 \rfloor}) \right) \leq \frac{1}{N-\cs} \left( \cs \Log\left( \frac{e N}{\cs} \right) + \Log\left( \frac{ (\cs+3)^{2} }{\conf} \right) \right).
\end{equation*}
By a union bound, this holds simultaneously for all $\cs \in [m] \cup \{0\}$ 
on an event $E_{1}$ of probability at least $1 - \sum_{i=0}^{m} \frac{\conf}{(i+3)^{2}} > 1 - \frac{2}{5} \conf$.
In particular, since the right hand side of the above inequality is nondecreasing in $\cs$, 
and $\hat{\cs}^{\prime\prime}_{m} \leq \hat{\cs}_{m}$, and since $\DIS(V_{m}) \subseteq \DIS(V_{\lfloor m/2 \rfloor})$,
we have that on $E_{1}$, if $N > \hat{\cs}_{m}$, then 
\begin{equation*}
\Px(\DIS(V_{m})) \leq \Px(\DIS(V_{\lfloor m/2 \rfloor})) \frac{1}{N-\hat{\cs}_{m}} \left( \hat{\cs}_{m} \Log\left( \frac{e N}{\hat{\cs}_{m}} \right) + \Log\left( \frac{ (\hat{\cs}_{m}+3)^{2} }{\conf} \right) \right).
\end{equation*}

Next, again since $N$ is conditionally ${\rm Binomial}(\lceil m/2 \rceil, \Px(\DIS(V_{\lfloor m/2 \rfloor})))$-distributed
given $X_{1},\ldots,X_{\lfloor m/2 \rfloor}$,
by a Chernoff bound (applied under the conditional distribution given $X_{1},\ldots,X_{\lfloor m/2 \rfloor}$),
combined with the law of total probability, we obtain that on an event $E_{2}$ of probability at least $1 - \conf/3$, 
if $\Px(\DIS(V_{\lfloor m/2 \rfloor})) \geq \frac{16}{m} \ln\left(\frac{3}{\conf}\right) \geq \frac{8}{\lceil m/2 \rceil} \ln\left(\frac{3}{\conf}\right)$, then
\begin{equation*}
N \geq \Px(\DIS(V_{\lfloor m/2 \rfloor})) \lceil m/2 \rceil / 2 \geq \Px(\DIS(V_{\lfloor m/2 \rfloor})) m / 4.
\end{equation*}
Also note that if $\Px(\DIS(V_{m})) \geq \frac{16}{m} \left( 2 \hat{\cs}_{m} +  \ln\left(\frac{3}{\conf}\right) \right)$, 
then monotonicity of $t \mapsto \DIS(V_{t})$ and monotonicity of probability measures imply
$\Px(\DIS(V_{\lfloor m/2 \rfloor})) \geq \frac{16}{m} \left( 2 \hat{\cs}_{m} +  \ln\left(\frac{3}{\conf}\right) \right)$ as well.
In particular, if this occurs with $E_{2}$, then we have $N \geq \Px(\DIS(V_{\lfloor m/2 \rfloor})) m / 4 > 8 \hat{\cs}_{m}$.
Thus, by monotonicity of $x \mapsto \Log(x)/x$ for $x > 0$, we have that on $E_{1} \cap E_{2}$, 
if $\Px(\DIS(V_{m})) \geq \frac{16}{m} \left( 2 \hat{\cs}_{m} + \ln\left(\frac{3}{\conf}\right) \right)$, 
then
\begin{align*}
\Px(\DIS(V_{m})) & < \Px(\DIS(V_{\lfloor m/2 \rfloor})) \frac{8}{7 N} \left( \hat{\cs}_{m} \Log\left(\frac{e N}{\hat{\cs}_{m}}\right) + \ln\left(\frac{(\hat{\cs}_{m}+3)^{2}}{\conf}\right) \right)
\\ & \leq \frac{32}{7 m} \left( \hat{\cs}_{m} \Log\left(\frac{e \Px(\DIS(V_{\lfloor m/2 \rfloor})) m}{4 \hat{\cs}_{m}}\right) + \ln\left(\frac{(\hat{\cs}_{m}+3)^{2}}{\conf}\right) \right).
\end{align*}

The inductive hypothesis implies that, on an event $E_{3}$ of probability at least $1 - \conf/4$, 
\begin{equation*}
\Px(\DIS(V_{\lfloor m/2 \rfloor}))
\leq \frac{16}{\lfloor m/2 \rfloor} \left( 2 \hat{\cs}_{1:\lfloor m/2 \rfloor} + \ln\left(\frac{12}{\conf}\right) \right).
\end{equation*}
Since $m \geq 17$, we have $\lfloor m/2 \rfloor \geq (m-2)/2 \geq (15/34)m$, so that the above implies
\begin{equation*}
\Px(\DIS(V_{\lfloor m/2 \rfloor})) \leq \frac{544}{15 m} \left( 2 \hat{\cs}_{1:\lfloor m/2 \rfloor} + \ln\left(\frac{12}{\conf}\right) \right). 
\end{equation*}
Thus, on $E_{1} \cap E_{2} \cap E_{3}$, if $\Px(\DIS(V_{m})) \geq \frac{16}{m} \left( 2 \hat{\cs}_{m} + \ln\left(\frac{3}{\conf}\right) \right)$, 
then
\begin{multline}
\Px(\DIS(V_{m})) < \frac{32}{7 m} \left( \hat{\cs}_{m} \Log\left( \frac{136 e}{15} \left( 2 \frac{\hat{\cs}_{1:\lfloor m/2 \rfloor}}{\hat{\cs}_{m}} + \frac{1}{\hat{\cs}_{m}} \ln\left(\frac{12}{\conf}\right) \right) \right) + \ln\left(\frac{(\hat{\cs}_{m}+3)^{2}}{\conf}\right) \right)
\\ \leq \frac{32}{7 m} \left( \hat{\cs}_{1:m} \Log\left( \frac{136 e}{15} \left( 2 + \frac{1}{\hat{\cs}_{1:m}} \ln(4) + \frac{1}{\hat{\cs}_{1:m}} \ln\left(\frac{3}{\conf}\right) \right) \right) + \ln\left(\frac{(\hat{\cs}_{1:m}+3)^{2}}{\conf}\right) \right). \label{eqn:PDIS-pre-logs}
\end{multline}
By straightforward calculus, one can easily verify that, when 
$\hat{\cs}_{1:m} \in \{0,1\}$, 
the right hand side of \eqref{eqn:PDIS-pre-logs} is at most
$\frac{16}{m} \left( 2 \hat{\cs}_{1:m} + \ln\left(\frac{3}{\conf}\right) \right)$
(recalling our conventions that $1/0 = \infty$ and $0 \Log(\infty) = 0$).
Otherwise, supposing $\hat{\cs}_{1:m} \geq 2$, 
Lemma~\ref{lem:log-factors-abstract} in Appendix~\ref{app:technical-lemmas} (applied with $\b = \frac{5e}{2}\ln(3/\conf)$) implies
the right hand side of \eqref{eqn:PDIS-pre-logs} is at most 
\begin{multline*}
\frac{32}{7 m} \left( \hat{\cs}_{1:m} \Log\left( \frac{136 e}{15} \left( 2 + \ln(4) + \frac{2}{5} \right) \right) + 2 \ln( \hat{\cs}_{1:m}+3 ) + \frac{7}{2} \ln\left(\frac{3}{\conf}\right) \right)
\\ \leq \frac{32}{7 m} \left( 5 \hat{\cs}_{1:m} + 2\ln( \hat{\cs}_{1:m}+3 ) + \frac{7}{2} \ln\left(\frac{3}{\conf}\right) \right).
\end{multline*}
Since $5 x + 2\ln(x+3) < 7x$ for any $x \geq 2$,
the above is at most
\begin{equation*}
\frac{32}{7 m} \left( 7 \hat{\cs}_{1:m} + \frac{7}{2}\ln\left(\frac{3}{\conf}\right) \right)
= \frac{16}{m} \left( 2 \hat{\cs}_{1:m} + \ln\left(\frac{3}{\conf}\right) \right).
\end{equation*} 
Thus, since $\frac{16}{m} \left( 2 \hat{\cs}_{m} + \ln\left(\frac{3}{\conf}\right) \right) \leq \frac{16}{m} \left( 2 \hat{\cs}_{1:m} + \ln\left(\frac{3}{\conf}\right) \right)$ as well,
in either case we have that, on $E_{1} \cap E_{2} \cap E_{3}$,
\begin{equation*}
\Px(\DIS(V_{m})) \leq \frac{16}{m} \left( 2 \hat{\cs}_{1:m} + \ln\left(\frac{3}{\conf}\right) \right).
\end{equation*}
Noting that, by a union bound, the event $E_{1} \cap E_{2} \cap E_{3}$ has probability at least $1 - \frac{2}{5}\conf - \frac{1}{3}\conf - \frac{1}{4}\conf > 1-\conf$,
this extends the result to $m$.  By the principle of induction, this completes the proof of Theorem~\ref{thm:PDIS}.
\end{proof}

\section{Proof of Theorem~\ref{thm:vc-erm}}
\label{app:erm}

We now present the proof of Theorem~\ref{thm:vc-erm}.

\begin{proof}[of Theorem~\ref{thm:vc-erm}]
The result trivially holds for $m \leq \lfloor 8(\ln(37)+8\ln(6)) \rfloor = 143$, so suppose $m \geq 144$.
Let $N = | \{X_{\lfloor m/2 \rfloor+1},\ldots,X_{m}\} \cap \DIS(V_{\lfloor m/2 \rfloor}) |$
and enumerate the elements of $\{X_{\lfloor m/2 \rfloor+1},\ldots,X_{m}\} \cap \DIS(V_{\lfloor m/2 \rfloor})$
as $\hat{X}_{1},\ldots,\hat{X}_{N}$.
Note that $N$ is conditionally ${\rm Binomial}(\lceil m/2 \rceil, \Px(\DIS(V_{\lfloor m/2 \rfloor})))$ -distributed
given $X_{1},\ldots,X_{\lfloor m/2 \rfloor}$.
In particular, with probability one, if $\Px(\DIS(V_{\lfloor m/2 \rfloor})) = 0$, then $N=0$.
Otherwise, if $\Px(\DIS(V_{\lfloor m/2 \rfloor})) > 0$, then note that $\hat{X}_{1},\ldots,\hat{X}_{N}$
are conditionally independent $\Px(\cdot | \DIS(V_{\lfloor m/2 \rfloor}))$-distributed random variables,
given $X_{1},\ldots,X_{\lfloor m/2 \rfloor}$ and $N$.  
Also, note that (one can easily show) $\vc( \{ \{x : h(x) \neq \target(x)\} : h \in \C \} ) = \dim$. 
Together with Lemma~\ref{lem:classic-erm} (applied under the conditional distribution given $X_{1},\ldots,X_{\lfloor m/2 \rfloor}$ and $N$),
combined with the law of total probability, 
these observations imply that there is an event $H_{1}$ of probability at least $1-\conf/3$, on which,
if $N > 0$, then $\forall h \in V_{m}$, 
\begin{equation*}
\Px( \DIS(\{ h, \target \}) | \DIS(V_{\lfloor m/2 \rfloor}) ) \leq \frac{2}{N} \left( \dim \Log_{2}\left( \frac{2eN}{\dim} \right) + \log_{2}\left(\frac{6}{\conf}\right) \right).
\end{equation*}
In particular, noting that $\forall h \in V_{m}$, since $\target \in V_{m}$ as well, 
$\DIS(\{ h, \target \}) \subseteq \DIS(V_{m}) \subseteq \DIS(V_{\lfloor m/2 \rfloor})$, 
we have that on $H_{1}$, $\forall h \in V_{m}$, 
\begin{align*}
\er(h) & = \Px(\DIS(\{h,\target\})) 
= \Px(\DIS(\{h,\target\})|\DIS(V_{\lfloor m/2 \rfloor}))\Px(\DIS(V_{\lfloor m/2 \rfloor}))
\\ & \leq \Px(\DIS(V_{\lfloor m/2 \rfloor})) \frac{2}{N} \left( \dim \Log_{2}\left( \frac{2eN}{\dim} \right) + \log_{2}\left(\frac{6}{\conf}\right) \right).
\end{align*}

Next, again since $N$ is conditionally ${\rm Binomial}(\lceil m/2 \rceil, \Px(\DIS(V_{\lfloor m/2 \rfloor})))$-distributed
given $X_{1},\ldots,X_{\lfloor m/2 \rfloor}$, by a Chernoff bound (applied under the conditional distribution given $X_{1},\ldots,X_{\lfloor m/2 \rfloor}$),
combined with the law of total probability,
there is an event $H_{2}$ of probability at least $1-\conf/3$, on which,
if $\Px(\DIS(V_{\lfloor m/2 \rfloor})) \geq \frac{32}{\lceil m/2 \rceil}\ln\left(\frac{3}{\conf}\right)$, then 
\begin{equation*}
N \geq (3/4) \Px(\DIS(V_{\lfloor m/2 \rfloor})) \lceil m/2 \rceil \geq (3/8) \Px(\DIS(V_{\lfloor m/2 \rfloor})) m,
\end{equation*}
which (by $\Log_{2}(x) \leq \Log(x)/\ln(2)$ and monotonicity of $x \mapsto \Log(x)/x$ for $x > 0$) implies
\begin{align*}
& \frac{2}{N} \left( \dim \Log_{2}\left( \frac{2eN}{\dim} \right) + \log_{2}\left(\frac{6}{\conf}\right) \right)
\\ & \leq \frac{16}{3 \ln(2) \Px(\DIS(V_{\lfloor m/2 \rfloor})) m} \left( \dim \ln\left( \frac{3 e \Px(\DIS(V_{\lfloor m/2 \rfloor})) m}{4 \dim} \right) + \ln\left(\frac{6}{\conf}\right) \right).
\end{align*}
Also, by Theorem~\ref{thm:PDIS}, on an event $H_{3}$ of probability at least $1-\conf/3$, 
\begin{equation*}
\Px(\DIS(V_{\lfloor m/2 \rfloor})) \leq \frac{16}{ \lfloor m/2 \rfloor } \left( 2 \hat{\cs}_{1:\lfloor m/2 \rfloor} + \ln\left(\frac{9}{\conf}\right) \right).
\end{equation*}
Together with the facts that $\frac{16}{3\ln(2)} < 8$ and $\lfloor m/2 \rfloor \geq \frac{m-2}{m} \frac{m}{2} \geq \frac{142}{144} \frac{m}{2}$, 
we have that, on $H_{1} \cap H_{2} \cap H_{3}$, if $\Px(\DIS(V_{\lfloor m/2 \rfloor})) \geq \frac{32}{\lceil m/2 \rceil}\ln\left(\frac{3}{\conf}\right)$, then 
\begin{align*}
\sup_{h \in V_{m}} \er(h)  & \leq \frac{8}{m} \left( \dim \ln\left( \frac{e 24 \cdot 144 ( 2 \hat{\cs}_{1:\lfloor m/2 \rfloor} + \ln(9/\conf) )}{142 \dim} \right) + \ln\left(\frac{6}{\conf}\right) \right)
\\ & = \frac{8}{m} \left( \dim \ln\left( \frac{24 \cdot 144}{7 \cdot 142} \left( \frac{14 e \hat{\cs}_{1:\lfloor m/2 \rfloor} + 7 e \ln(3/2)}{\dim} + \frac{7 e \ln(6/\conf)}{\dim} \right) \right) + \ln\left(\frac{6}{\conf}\right) \right).
\end{align*}
By Lemma~\ref{lem:log-factors-abstract} in Appendix~\ref{app:technical-lemmas},
this last expression is at most
\begin{align*}
& \frac{8}{m} \left( \dim \ln\left( \frac{24 \cdot 144}{7 \cdot 142} \left( \frac{14 e \hat{\cs}_{1:\lfloor m/2 \rfloor} + 7 e \ln(3/2)}{\dim} + e \right) \right) + 8 \ln\left(\frac{6}{\conf}\right) \right)
\\ & \leq \frac{8}{m} \left( \dim \ln\left( \left( \frac{49 e \hat{\cs}_{1:\lfloor m/2 \rfloor}}{\dim} + 37 \right) \right) + 8 \ln\left(\frac{6}{\conf}\right) \right).
\end{align*}
Furthermore, since $\DIS(\{h,\target\}) \subseteq \DIS(V_{\lfloor m/2 \rfloor})$ for every $h \in V_{m}$, 
if $\Px(\DIS(V_{\lfloor m/2 \rfloor})) < \frac{32}{\lceil m/2 \rceil}\ln\left(\frac{3}{\conf}\right) \leq \frac{64}{m}\ln\left(\frac{3}{\conf}\right)$, then 
\begin{equation*}
\sup_{h \in V_{m}} \er(h) < \frac{64}{m}\ln\left(\frac{3}{\conf}\right) 
< \frac{8}{m} \left( \dim \Log\left( \frac{49 e \hat{\cs}_{1:\lfloor m/2 \rfloor}}{\dim} + 37 \right) + 8 \ln\left(\frac{6}{\conf}\right) \right).
\end{equation*}
Thus, in either case, we have that, on $H_{1} \cap H_{2} \cap H_{3}$,
\begin{equation*}
\sup_{h \in V_{m}} \er(h) \leq \frac{8}{m} \left( \dim \Log\left( \frac{49 e \hat{\cs}_{1:\lfloor m/2 \rfloor}}{\dim} + 37 \right) + 8 \ln\left(\frac{6}{\conf}\right) \right).
\end{equation*}
The proof is completed by noting that $\hat{\cs}_{1:\lfloor m/2 \rfloor} \leq \hat{\cs}_{1:m}$,
and that, by the union bound, the event $H_{1} \cap H_{2} \cap H_{3}$ has probability at least $1-\conf$.
\end{proof}

\section{Proof of Theorem~\ref{thm:vc-erm-subregion}}
\label{app:subregions}

We now present the proof of Theorem~\ref{thm:vc-erm-subregion}.

\begin{proof}[of Theorem~\ref{thm:vc-erm-subregion}]
The proof essentially combines the argument of \citet*{hanneke:thesis} (which proves \eqref{eqn:gk-dc})
with the subsample-based ideas of \citet*{zhang:14}.
Fix $c = 16$.  The proof proceeds by induction on $m$.
Since $\sup_{h \in \C} \er(h) \leq 1$, the result trivially holds for $m < 21 (\dim\ln(83)+3\ln(4))$.
Now, as an inductive hypothesis, fix any $m \geq 21 (\dim\ln(83)+3\ln(4))$ such that 
$\forall m^{\prime} \in [m-1]$, $\forall \conf \in (0,1)$, with probability at least $1-\conf$,
\begin{equation*}
\sup_{h \in V_{m^{\prime}}} \er(h) \leq \frac{21}{m^{\prime}} \left( \dim \Log\left( 83 \zc_{c}\left( \frac{\dim}{m^{\prime}} \right) \right) + 3 \Log\left( \frac{4}{\conf} \right) \right).
\end{equation*}

Fix any $\conf \in (0,1)$ and $\eta \in [0,1]$.
Let $\gamma^{*},\zeta^{*},\xi^{*}$ be the functions
$\gamma$, $\zeta$, and $\xi$ from Definition~\ref{def:zc} (each mapping $\X \to [0,1]$)
with $\gamma^{*}(x) + \zeta^{*}(x) + \xi^{*}(x) = 1$ for all $x \in \X$,
and $\E\left[ \gamma^{*}(X) \middle| X_{1}, \ldots, X_{\lfloor m/2 \rfloor} \right]$ minimal subject to 
\begin{equation*}
\sup_{h \in V_{\lfloor m/2 \rfloor}} \E\left[ \ind[h(X)=+1]\zeta^{*}(X)+\ind[h(X)=-1]\xi^{*}(X) \middle| X_{1},\ldots,X_{\lfloor m/2 \rfloor} \right] \leq \eta,
\end{equation*}
where $X \sim \Px$ is independent of $X_{1},X_{2},\ldots$.\footnote{Note 
that the minimum is actually achieved here, since 
the objective function is continuous and convex, and the feasible region is nonempty, closed, bounded, and convex \citep*[see][Proposition 5.50]{bowers:14}.}
Note that these functions are themselves random, having dependence on $X_{1},\ldots,X_{\lfloor m/2 \rfloor}$.
In particular, $\E\left[ \gamma^{*}(X) \middle| X_{1},\ldots,X_{\lfloor m/2 \rfloor} \right] = \Phi(V_{\lfloor m/2 \rfloor},\eta)$.

Let $\Gamma_{\lfloor m/2 \rfloor + 1}, \ldots, \Gamma_{m}$ be conditionally independent random variables given $X_{1},\ldots,X_{m}$,
with $\Gamma_{i}$ having conditional distribution ${\rm Bernoulli}(\gamma^{*}(X_{i}))$ given $X_{1},\ldots,X_{m}$, for each $i \in \{ \lfloor m/2 \rfloor + 1, \ldots, m \}$.
Let $N = | \{ i \in \{ \lfloor m/2 \rfloor + 1, \ldots, m \} : \Gamma_{i} = 1 \} |$,
and enumerate the elements of $\{ X_{i} : i \in \{ \lfloor m/2 \rfloor + 1, \ldots, m \}, \Gamma_{i} = 1 \}$
as $\hat{X}_{1},\ldots,\hat{X}_{N}$ (retaining their original order).
For $X \sim \Px$ independent of $X_{1},X_{2},\ldots$, let $\Gamma(X)$ denote a random variable that is conditionally ${\rm Bernoulli}(\gamma^{*}(X))$ given $X$ and $X_{1},\ldots,X_{\lfloor m/2 \rfloor}$.
Also define a (random) probability measure $P_{\lfloor m/2 \rfloor}$ such that, given $X_{1},\ldots,X_{\lfloor m/2 \rfloor}$, 
$P_{\lfloor m/2 \rfloor}(A) = \P(X \in A | \Gamma(X) = 1, X_{1},\ldots,X_{\lfloor m/2 \rfloor})$ for all measurable $A \subseteq \X$.

Note that $N = \sum_{t=\lfloor m/2 \rfloor+1}^{m} \Gamma_{i}$ is conditionally ${\rm Binomial}\left( \lceil m/2 \rceil, \Phi(V_{\lfloor m/2 \rfloor}, \eta) \right)$
given $X_{1},\ldots,X_{\lfloor m/2 \rfloor}$.
In particular, with probability one, if $\Phi(V_{\lfloor m/2 \rfloor},\eta) = 0$, then $N=0$.
Otherwise, if $\Phi(V_{\lfloor m/2 \rfloor},\eta) > 0$,
then $\hat{X}_{1},\ldots,\hat{X}_{N}$ are conditionally i.i.d. given $X_{1},\ldots,X_{\lfloor m/2 \rfloor}$ and $N$,
each with conditional distribution $P_{\lfloor m/2 \rfloor}$ given $X_{1},\ldots,X_{\lfloor m/2 \rfloor}$ and $N$.
Thus, since every $h \in V_{m}$ has $\{ x : h(x) \neq \target(x) \} \cap \{\hat{X}_{1},\ldots,\hat{X}_{N}\} \subseteq \{ x : h(x) \neq \target(x) \} \cap \{X_{1},\ldots,X_{m}\} = \emptyset$, 
and (one can easily show) $\vc( \{ \{ x : h(x) \neq \target(x) \} : h \in \C \} ) = \dim$, 
applying Lemma~\ref{lem:classic-erm} (under the conditional distribution given $N$ and $X_{1},\ldots,X_{\lfloor m/2 \rfloor}$), 
combined with the law of total probability, we have that on an event $E_{1}$ of probability at least $1-\conf/2$, if $N > 0$, then 
\begin{equation*}
\sup_{h \in V_{m}} P_{\lfloor m / 2 \rfloor}( x : h(x) \neq \target(x) ) \leq \frac{2}{N} \left( \dim \Log_{2}\left( \frac{2 e N}{\dim} \right) + \log_{2}\left( \frac{4}{\conf} \right) \right).
\end{equation*}

Next, since $N$ is conditionally ${\rm Binomial}\left( \lceil m/2 \rceil, \Phi(V_{\lfloor m/2 \rfloor}, \eta) \right)$
given $X_{1},\ldots,X_{\lfloor m/2 \rfloor}$, applying a Chernoff bound (under the conditional distribution given $X_{1},\ldots,X_{\lfloor m/2 \rfloor}$),
combined with the law of total probability, we obtain that on an event $E_{2}$ of probability at least $1-\conf/4$, 
if $\Phi(V_{\lfloor m/2 \rfloor},\eta) \geq \frac{18}{\lceil m / 2 \rceil} \ln\left(\frac{4}{\conf}\right)$, then 
\begin{equation*}
N \geq (2/3) \Phi(V_{\lfloor m/2 \rfloor},\eta) \lceil m/2 \rceil \geq \Phi(V_{\lfloor m/2 \rfloor},\eta) m / 3.
\end{equation*}
In particular, if $\Phi(V_{\lfloor m/2 \rfloor},\eta) \geq \frac{18}{\lceil m / 2 \rceil} \ln\left(\frac{4}{\conf}\right)$, then the right hand side is strictly greater than $0$, 
so that if this occurs with $E_{2}$, then we have $N > 0$.  Thus, by the fact that $\Log_{2}(x) \leq \Log(x)/\ln(2)$, 
combined with monotonicity of $x \mapsto \Log(x)/x$ for $x > 0$, we have that
on $E_{1} \cap E_{2}$, if $\Phi(V_{\lfloor m/2 \rfloor},\eta) \geq \frac{18}{\lceil m / 2 \rceil} \ln\left(\frac{4}{\conf}\right)$, then
\begin{equation*}
\sup_{h \in V_{m}} P_{\lfloor m/2 \rfloor}( x : h(x) \neq \target(x) ) \leq \frac{6 / \ln(2)}{\Phi(V_{\lfloor m/2 \rfloor},\eta) m} \left( \dim \Log\left( \frac{2 e \Phi(V_{\lfloor m/2 \rfloor},\eta) m}{3 \dim} \right) + \ln\left( \frac{4}{\conf} \right) \right).
\end{equation*}

Next (following an argument of \citealp*{zhang:14}), note that $\forall h \in V_{m}$, 
\begin{align*}
\er(h) 
& = \E\left[ \ind[ h(X) \neq \target(X) ] \left( \gamma^{*}(X) + \zeta^{*}(X) + \xi^{*}(X) \right) \middle| X_{1},\ldots,X_{\lfloor m/2 \rfloor}\right]
\\ & = P_{\lfloor m/2 \rfloor}( x : h(x) \neq \target(x) ) \P( \Gamma(X) = 1 | X_{1},\ldots,X_{\lfloor m/2 \rfloor} ) 
\\ & {\hskip 4mm}+ \E\Big[ \big( \ind[ h(X) = +1] \ind[ \target(X) = -1 ] 
\\ & {\hskip 24mm}+ \ind[ h(X) = -1] \ind[ \target(X) = +1 ] \big) \left( \zeta^{*}(X) + \xi^{*}(X) \right) \Big| X_{1},\ldots,X_{\lfloor m/2 \rfloor} \Big]
\\ & \leq P_{\lfloor m/2 \rfloor}( x : h(x) \neq \target(x) ) \Phi(V_{\lfloor m/2 \rfloor},\eta)
\\ & {\hskip 4mm}+ \E\left[ \ind[h(X) = +1] \zeta^{*}(X) + \ind[h(X)=-1] \xi^{*}(X) \middle| X_{1},\ldots,X_{\lfloor m/2 \rfloor} \right]
\\ & {\hskip 4mm}+ \E\left[ \ind[\target(X)=+1]\zeta^{*}(X) + \ind[ \target(X) = -1 ] \xi^{*}(X) \middle| X_{1},\ldots,X_{\lfloor m/2 \rfloor} \right].
\end{align*}
Since $h,\target \in V_{\lfloor m/2 \rfloor}$, the definition of $\zeta^{*}$ and $\xi^{*}$ implies this last expression is at most
\begin{equation*}
P_{\lfloor m/2 \rfloor}( x : h(x) \neq \target(x) ) \Phi(V_{\lfloor m/2 \rfloor},\eta) + 2\eta.
\end{equation*}
Therefore, on $E_{1} \cap E_{2}$, if $\Phi(V_{\lfloor m/2 \rfloor},\eta) \geq \frac{18}{\lceil m / 2 \rceil} \ln\left(\frac{4}{\conf}\right)$, then
\begin{equation*}
\sup_{h \in V_{m}} \er(h) \leq 2\eta + \frac{6 / \ln(2)}{m} \left( \dim \Log\left( \frac{2 e \Phi(V_{\lfloor m/2 \rfloor},\eta) m}{3 \dim} \right) + \ln\left( \frac{4}{\conf} \right) \right).
\end{equation*}

The inductive hypothesis implies that, on an event $E_{3}$ of probability at least $1-\conf/4$, 
\begin{equation*}
\sup_{h \in V_{\lfloor m/2 \rfloor}} \er(h) \leq 
\frac{21}{\lfloor m/2 \rfloor} \left( \dim \Log\left( 83 \zc_{c}\left( \frac{\dim}{\lfloor m/2 \rfloor} \right) \right) + 3 \Log\left( \frac{16}{\conf} \right) \right).
\end{equation*}
Since $m \geq \lceil 21(\dim \ln(83) + 3\ln(4)) \rceil \geq 181$, we have $\lfloor m/2 \rfloor \geq (m-2)/2 \geq (179/362) m$, so that (together with monotonicity of $\zc_{c}(\cdot)$)
the above implies $V_{\lfloor m/2 \rfloor} \subseteq \Ball(\target,r_{\lfloor m/2 \rfloor})$, where
\begin{equation*}
r_{\lfloor m/2 \rfloor} = \frac{21 \cdot 362}{179 m} \left( \dim \ln\left( 83 \zc_{c}\left( \frac{\dim}{m} \right) \right) + 3 \ln\left( \frac{16}{\conf} \right) \right).
\end{equation*}

Altogether, plugging in $\eta = (r_{\lfloor m/2 \rfloor}/c) \land 1$,
and noting that $\H \mapsto \Phi(\H,\eta)$ is nondecreasing in $\H$, and that $\dim /m \leq r_{\lfloor m/2 \rfloor}$,
we have that on $E_{1} \cap E_{2} \cap E_{3}$, if $\Phi(V_{\lfloor m/2 \rfloor}, (r_{\lfloor m/2 \rfloor}/c) \land 1) \geq \frac{18}{\lceil m / 2 \rceil} \ln\left(\frac{4}{\conf}\right)$, then
\begin{align}
& \sup_{h \in V_{m}} \er(h) \leq 
\frac{2r_{\lfloor m/2 \rfloor}}{c} + \frac{6 / \ln(2)}{m} \left( \dim \Log\left( \frac{2 e \Phi(\Ball(\target,r_{\lfloor m/2 \rfloor}),(r_{\lfloor m/2 \rfloor}/c) \land 1) m}{3 \dim} \right) + \ln\left( \frac{4}{\conf} \right) \right) \notag
\\ & \leq \frac{2r_{\lfloor m/2 \rfloor}}{c} + \frac{6 / \ln(2)}{m} \left( \dim \ln\left( \frac{2 e \zc_{c}(\dim/m) r_{\lfloor m/2 \rfloor} m}{3 \dim} \right) + \ln\left( \frac{4}{\conf} \right) \right). \label{eqn:zc-erm-unsimplified}
\end{align}
The second term in this last expression equals
\begin{align*}
& \frac{6 / \ln(2)}{m} \left( \dim \ln\left( \frac{14 \cdot 362}{179} \zc_{c}\left(\frac{\dim}{m}\right) \left( e \ln\left( 83 \zc_{c}\left( \frac{\dim}{m} \right) \right) + \frac{3e}{\dim} \ln\left( \frac{16}{\conf} \right) \right) \right) + \ln\left( \frac{4}{\conf} \right) \right)
\\ & \leq \frac{6 / \ln(2)}{m} \left( \dim \ln\!\left( \frac{14 \cdot 362 \cdot 6}{179 \cdot 7} \zc_{c}\!\left(\frac{\dim}{m}\right) \left( \frac{7e}{6} \ln\!\left( 64 \cdot 83 \zc_{c}\!\left( \frac{\dim}{m} \right) \right) \!+\! \frac{7e}{2\dim} \ln\!\left( \frac{4}{\conf} \right) \right) \right) \!+\! \ln\!\left( \frac{4}{\conf} \right) \right).
\end{align*}
Applying Lemma~\ref{lem:log-factors-abstract} (with $\b = (7e/2) \ln(4/\conf)$), this is at most
\begin{equation*}
\frac{6 / \ln(2)}{m} \left( \dim \ln\left( \frac{14 \cdot 362 \cdot 6}{179 \cdot 7} \zc_{c}\left(\frac{\dim}{m}\right) \left( \frac{7 e}{6} \ln\left( 64 \cdot 83 \zc_{c}\left( \frac{\dim}{m} \right) \right) + e \right) \right) + \frac{9}{2} \ln\left( \frac{4}{\conf} \right) \right),
\end{equation*}
and a simple relaxation of the expression in the logarithm reveals this is at most
\begin{equation*}
\frac{6 / \ln(2)}{m} \left( \frac{3}{2} \dim \ln\!\left( 83 \zc_{c}\!\left(\frac{\dim}{m}\right) \right) + \frac{9}{2} \ln\!\left( \frac{4}{\conf} \right) \right)
\leq \frac{13}{m} \left( \dim \ln\!\left( 83 \zc_{c}\!\left(\frac{\dim}{m}\right) \right) + 3 \ln\!\left( \frac{4}{\conf} \right) \right).
\end{equation*}
Additionally, some straightforward reasoning about numerical constants reveals that
\begin{equation*}
\frac{2 r_{\lfloor m/2 \rfloor}}{c} \leq \frac{8}{m} \left( \dim \ln\left( 83 \zc_{c}\left(\frac{\dim}{m}\right) \right) + 3 \ln\left(\frac{4}{\conf}\right) \right).
\end{equation*}
Plugging these two facts back into \eqref{eqn:zc-erm-unsimplified},
we have that on $E_{1} \cap E_{2} \cap E_{3}$, 
if $\Phi(V_{\lfloor m/2 \rfloor},(r_{\lfloor m/2 \rfloor}/c) \land 1) \geq \frac{18}{\lceil m / 2 \rceil} \ln\left(\frac{4}{\conf}\right)$, then
\begin{equation}
\label{eqn:zc-erm-simplified}
\sup_{h \in V_{m}} \er(h) \leq \frac{21}{m} \left( \dim \ln\left( 83 \zc_{c}\left(\frac{\dim}{m}\right) \right) + 3 \ln\left(\frac{4}{\conf}\right) \right). 
\end{equation}

On the other hand, if $\Phi(V_{\lfloor m/2 \rfloor},(r_{\lfloor m/2 \rfloor}/c) \land 1) < \frac{18}{\lceil m / 2 \rceil} \ln\left(\frac{4}{\conf}\right)$, then recalling that
(as established above) $\sup_{h \in V_{m}} \er(h) \leq 2\eta + \sup_{h \in V_{m}} P_{\lfloor m/2 \rfloor}( x : h(x) \neq \target(x) ) \Phi(V_{\lfloor m/2 \rfloor},\eta)$,
plugging in $\eta = (r_{\lfloor m/2 \rfloor}/c) \land 1$ and noting that $P_{\lfloor m/2 \rfloor}( x : h(x) \neq \target(x) ) \leq 1$, we have
\begin{align*}
\sup_{h \in V_{m}} \er(h) 
& \leq \frac{2 r_{\lfloor m/2 \rfloor}}{c} + \Phi(V_{\lfloor m/2 \rfloor},(r_{\lfloor m/2 \rfloor}/c) \land 1)
\\ & < \frac{8}{m} \left( \dim \ln\left( 83 \zc_{c}\left(\frac{\dim}{m}\right) \right) + 3 \ln\left(\frac{4}{\conf}\right) \right) + \frac{18}{\lceil m / 2 \rceil} \ln\left(\frac{4}{\conf}\right)
\\ & \leq \frac{21}{m} \left( \dim \ln\left( 83 \zc_{c}\left(\frac{\dim}{m}\right) \right) + 3 \ln\left(\frac{4}{\conf}\right) \right).
\end{align*}
Thus, in either case, on $E_{1} \cap E_{2} \cap E_{3}$, \eqref{eqn:zc-erm-simplified} holds.
Noting that, by the union bound, the event $E_{1} \cap E_{2} \cap E_{3}$ has probability at least $1-\conf$,
this extends the inductive hypothesis to $m$.  The result then follows by the principle of induction.
\end{proof}

\subsection{The Worst-Case Value of $\boldsymbol{\zc_{c}}$}
\label{sec:sup-zc}

Next, we prove \eqref{eqn:sup-zc}.
Fix any $c \geq 2$.
First, suppose $r_{0} \in (0,1)$, and let $m = \min\left\{ \s, \left\lceil \frac{1}{r_{0}} \right\rceil \right\}$;
note that our assumption that $|\C| \geq 3$ implies $\s \geq 2$, so that $m \geq 2$ here.
Let $x_{1},\ldots,x_{m} \in \X$ and $h_{0},h_{1},\ldots,h_{m} \in \C$ be as in Definition~\ref{def:star}.
Let $\Px(\{x_{i}\}) = 1/m$ for each $i \in [m]$, and take $\target = h_{0}$.

Let $r_{1}$ be any value satisfying $\max\{1/m,r_{0}\} < r_{1} \leq 1$
chosen sufficiently close to $\max\{1/m,r_{0}\}$ so that $\frac{m r_{1}}{c} < 1$.
Consider now the definition of $\Phi(\Ball(\target,r_{1}),r_{1}/c)$ from Definition~\ref{def:zc}.
For any functions $\chi_{0},\chi_{1} : \X \to [0,1]$, 
let $\zeta(x) = \ind[ h_{0}(x) = -1 ] \chi_{0}(x) + \ind[ h_{0}(x) = +1 ] \chi_{1}(x)$
and $\xi(x) = \ind[ h_{0}(x) = -1 ] \chi_{1}(x) + \ind[ h_{0}(x) = +1 ] \chi_{0}(x)$.
In particular, note that it is possible to specify any functions $\zeta,\xi : \X \to [0,1]$ 
by choosing appropriate $\chi_{0},\chi_{1}$ values (namely, 
$\chi_{0}(x) = \ind[ h_{0}(x) = -1 ] \zeta(x) + \ind[ h_{0}(x) = +1 ] \xi(x)$
and
$\chi_{1}(x) = \ind[ h_{0}(x) = -1 ] \xi(x) + \ind[ h_{0}(x) = +1 ] \zeta(x)$).
Noting that, for any classifier $h$ and any $x \in \X$, 
$\ind[h(x) = +1] \zeta(x) + \ind[h(x) = -1] \xi(x)
= \ind[ h(x) \neq h_{0}(x) ] \chi_{0}(x) + \ind[ h(x) = h_{0}(x) ] \chi_{1}(x)$,
and $\zeta(x) + \xi(x) = \chi_{0}(x) + \chi_{1}(x)$,
we may re-express the constraints in the optimization problem defining $\Phi(\Ball(\target,r_{1}),r_{1}/c)$ in Definition~\ref{def:zc} as
$\sup_{h \in \Ball(\target,r_{1})} \E[ \ind[ h(X) \neq h_{0}(X) ] \chi_{0}(X) + \ind[ h(X) = h_{0}(X) ] \chi_{1}(X) ] \leq r_{1}/c$
and $\forall x \in \X$, $\gamma(x) + \chi_{0}(x) + \chi_{1}(x) = 1$ while $\gamma(x), \chi_{0}(x), \chi_{1}(x) \in [0,1]$.
We may further simplify the problem by noting that $\gamma(x) = 1 - \chi_{0}(x) - \chi_{1}(x)$, so that these last two 
constraints become $\chi_{0}(x) + \chi_{1}(x) \leq 1$ while $\chi_{0}(x),\chi_{1}(x) \geq 0$,
and the value $\Phi(\Ball(\target,r_{1}),r_{1}/c)$ is the minimum achievable value of $\E[ 1 - \chi_{0}(X) - \chi_{1}(X) ]$ subject to these constraints.
Furthermore, noting that $h_{i} \in \Ball(\target,r_{1})$ for every $i \in [m]$, 
we have that
\begin{align*}
& \Phi(\Ball(\target,r_{1}),r_{1}/c) 
\\ & \geq \min\left\{ \E[ 1 - \chi_{0}(X) - \chi_{1}(X) ] : \phantom{\max_{i \in [m]}} \right.
\\ &{\hskip 22mm}\left. \max_{i \in [m]} \E\left[ \ind[ h_{i}(X) \neq h_{0}(X) ] \chi_{0}(X) + \ind[ h_{i}(X) = h_{0}(X) ] \chi_{1}(X) \right] \leq \frac{r_{1}}{c}, \right.
\\ &{\hskip 15mm} \left. \phantom{\max_{i \in [m]}} \text{where } \forall x \in \X, \chi_{0}(x) + \chi_{1}(x) \leq 1 \text{ and } \chi_{0}(x),\chi_{1}(x) \geq 0 \right\}
\\ & = \min\left\{ \sum_{i=1}^{m} \frac{1}{m}(1 - \chi_{0}(x_{i}) - \chi_{1}(x_{i})) : \right.
\\ &{\hskip 14mm} \left. \forall i \in [m], \chi_{0}(x_{i}) + \sum_{j \neq i} \chi_{1}(x_{j}) \leq \frac{m r_{1}}{c}, \chi_{0}(x_{i}) + \chi_{1}(x_{i}) \leq 1, \chi_{0}(x_{i}),\chi_{1}(x_{i}) \geq 0 \right\}.
\end{align*}
This is a simple linear program with linear inequality constraints.
We can explicitly solve this problem to find an optimal solution with $\chi_{1}(x_{i}) = 0$ and $\chi_{0}(x_{i}) = \frac{m r_{1}}{c}$ for all $i \in [m]$,
at which the value of the objective function $\sum_{i=1}^{m} \frac{1}{m}( 1 - \chi_{0}(x_{i}) - \chi_{1}(x_{i}) )$ is $1 - \frac{m r_{1}}{c}$.
One can easily verify that this choice of $\chi_{0}$ and $\chi_{1}$ satisfies the constraints above.
To see that this is an optimal choice, we note that the objective function can be re-expressed as 
$\sum_{i=1}^{m} \frac{1}{m} ( 1 - \chi_{0}(x_{i}) - \chi_{1}(x_{\sigma(i)}) )$, where $\sigma(i) = i+1$ for $i \in [m-1]$, and $\sigma(m) = 1$.
In particular, since $m \geq 2$, we have $\sigma(i) \neq i$ for each $i \in [m]$.
Thus, for any $\chi_{0}$ and $\chi_{1}$ satisfying the constraints above, 
we have $\chi_{0}(x_{i}) + \chi_{1}(x_{\sigma(i)}) \leq \chi_{0}(x_{i}) + \sum_{j \neq i} \chi_{1}(x_{j}) \leq \frac{m r_{1}}{c}$ for each $i \in [m]$,
so that $\sum_{i=1}^{m} \frac{1}{m} ( 1 - \chi_{0}(x_{i}) - \chi_{1}(x_{\sigma(i)}) ) \geq 1 - \frac{m r_{1}}{c}$,
which is precisely the value obtained with the above choices of $\chi_{0}$ and $\chi_{1}$.

Thus, since the above argument holds for any choice of $r_{1} > \max\{1/m,r_{0}\}$ sufficiently close to $\max\{1/m,r_{0}\}$,
we have
\begin{equation*}
\zc_{c}(r_{0}) = \sup_{r_{0} < r \leq 1} \frac{\Phi(\Ball(\target,r),r/c)}{r} \lor 1
\geq \lim_{r_{1} \searrow \max\{1/m,r_{0}\}} \frac{1 - \frac{m r_{1}}{c}}{r_{1}}
= \frac{1 - \frac{1}{c} \max\{ 1, m r_{0} \}}{\max\{ 1/m, r_{0} \}}.
\end{equation*}
If $\s < \frac{1}{r_{0}}$, then $m = \s$, and the rightmost expression above equals $(1 - 1/c)\s$.
Otherwise, if $s \geq \frac{1}{r_{0}}$, then $m = \left\lceil \frac{1}{r_{0}} \right\rceil$, 
and the rightmost expression above equals 
\begin{equation*}
\left( 1 - \frac{1}{c} \left\lceil \frac{1}{r_{0}} \right\rceil r_{0} \right) \frac{1}{r_{0}} 
\geq \left( 1 - \frac{1+r_{0}}{c} \right) \frac{1}{r_{0}} 
= \left( 1 - \frac{1}{c} \right) \left( \frac{1}{r_{0}} - \frac{1}{c-1} \right).
\end{equation*}
Either way, we have
\begin{equation*}
\zc_{c}(r_{0}) \geq \left(1 - \frac{1}{c}\right) \min\left\{ \s, \frac{1}{r_{0}} - \frac{1}{c-1} \right\}.
\end{equation*}

For the case $r_{0} = 0$, we note that $\forall \eps > 0$, any $c \geq 2$ has
\begin{equation*}
\sup_{\Px} \sup_{\target \in \C} \zc_{c}(0) \geq \sup_{\Px} \sup_{\target \in \C} \zc_{c}(\eps) \geq \left(1 - \frac{1}{c}\right) \min\left\{ \s, \frac{1}{\eps} - \frac{1}{c-1} \right\}.
\end{equation*}
Taking the limit $\eps \to 0$ yields $\sup_{\Px} \sup_{\target \in \C} \zc_{c}(0) \geq \left(1 - \frac{1}{c}\right) \s = \left(1 - \frac{1}{c}\right) \min\left\{ \s, \frac{1}{r_{0}} - \frac{1}{c-1} \right\}$.

For the upper bound, we clearly have $\zc_{c}(r_{0}) \leq (1-1/c) \dc(r_{0})$ for every $c > 1$.
To see this, take $\zeta(x) = (1/c) \ind[ x \in \DIS(\Ball(\target,r)) ] \ind[ \target(x) = -1 ] + \ind[ x \notin \DIS(\Ball(\target,r)) ] \ind[ \target(x) = -1 ]$
and $\xi(x) = (1/c) \ind[ x \in \DIS(\Ball(\target,r)) ] \ind[ \target(x) = +1 ] + \ind[ x \notin \DIS(\Ball(\target,r)) ] \ind[ \target(x) = +1 ]$
in the optimization problem defining $\Phi(\Ball(\target,r),r/c)$ in Definition~\ref{def:zc}.  
With these choices of $\zeta$ and $\xi$, we have $\E[\gamma(X)] = (1-1/c) \Px(\DIS(\Ball(\target,r)))$;
also, for any $h \in \Ball(\target,r)$, since $\DIS(\{h,\target\}) \subseteq \DIS(\Ball(\target,r))$,
we have $\E[ \ind[h(X) = +1] \zeta(X) + \ind[ h(X) = -1 ] \xi(X) ] = \E[ (1/c) \ind[ h(X) \neq \target(X) ] ] = (1/c) \Px( x : h(x) \neq \target(x) ) \leq r/c$;
one can easily verify that the remaining constraints are also satisfied.
Thus, since \citet*{hanneke:15b} prove $\sup_{\Px} \sup_{\target \in \C} \dc(r_{0}) = \min\left\{ \s, \frac{1}{r_{0}} \right\}$, 
we have $\sup_{\Px} \sup_{\target \in \C} \zc_{c}(r_{0}) \leq (1-1/c) \min\left\{ \s, \frac{1}{r_{0}} \right\}$.

We also note that, if we define $\zc_{c}^{\zo}(r_{0})$ identically to $\zc_{c}(r_{0})$ except that
$\gamma$ is restricted to have \emph{binary} values (i.e., in $\{0,1\}$),
then for $c \geq 4$, this same construction giving the lower bound above 
must have $\gamma(x_{i}) = 1$ for every $i \in [m]$,
which implies $\zc_{c}^{\zo}(r_{0}) \geq \min\left\{ \s, \frac{1}{r_{0}} \right\}$ in this case. 
To see this, consider any $r_{1} > \max\{1/m,r_{0}\}$ sufficiently small so that $\frac{m r_{1}}{c} < \frac{1}{2}$;
then to satisfy the constraints $\chi_{0}(x_{i}) + \sum_{j \neq i} \chi_{1}(x_{j}) \leq \frac{m r_{1}}{c} < \frac{1}{2}$ for every $i \in [m]$, 
while $\chi_{0}(x_{i}),\chi_{1}(x_{i}) \geq 0$, we must have every $\chi_{0}(x_{i})$ and $\chi_{1}(x_{i})$ strictly less than $\frac{1}{2}$,
so that $\gamma(x_{i}) = 1 - \chi_{0}(x_{i}) - \chi_{1}(x_{i}) > 0$ (and hence, $\gamma(x_{i}) = 1$, due to the constraint to binary values).
As we always have $\zc_{c}^{\zo}(r_{0}) \leq \dc(r_{0})$, and \citet*{hanneke:15b} have shown 
$\sup_{\Px} \sup_{\target \in \C} \dc(r_{0}) = \min\left\{ \s, \frac{1}{r_{0}} \right\}$, 
this implies $\sup_{\Px} \sup_{\target \in \C} \zc_{c}^{\zo}(r_{0}) = \min\left\{ \s, \frac{1}{r_{0}} \right\}$ as well.

\subsection{Relation of $\boldsymbol{\zc_{c}(r_{0})}$ to the Doubling Dimension}
\label{app:doubling-zc-bound}

Here we present the proof of \eqref{eqn:doubling-zc-bound},
via a modification of an argument of \citet*{hanneke:15b}.
We in fact prove the following slightly stronger inequality:
for any $c \geq 8$ and $r > 0$,
\begin{equation}
\label{eqn:doubling-phi-bound}
\log_{2}\left( \covering\!\left( r/2, \Ball(\target,r), \Px\right) \right) 
\leq 2 \dim \log_{2}\left( 96 \left( \frac{\Phi(\Ball(\target,r),r/c)}{r} \lor 1 \right) \right),
\end{equation}
which will immediately imply \eqref{eqn:doubling-zc-bound} by taking the 
supremum of both sides over $r > r_{0}$ (with some careful consideration of
the special case $r=r_{0}$; see below).

Fix any $c > 4$ and $r \in (0,1]$.  Let $G_{r}$ denote any maximal $(r/2)$-packing of $\Ball(\target,r)$:
that is, $G_{r}$ is a subset of $\Ball(\target,r)$ of maximal cardinality such that 
$\min_{h,g \in G_{r} : h \neq g} \Px( x : h(x) \neq g(x) ) > r/2$.  It is known that any
such set $G_{r}$ satisfies 
\begin{equation}
\label{eqn:packing-covering-inequalities}
\covering(r/2,\Ball(\target,r),\Px) \leq |G_{r}| \leq \covering(r/4,\Ball(\target,r),\Px)
\end{equation}
\citep*[see e.g.,][]{kolmogorov:59,kolmogorov:61,vidyasagar:03}.
In particular, since we have assumed $\dim < \infty$, in our case this further implies $|G_{r}| < \infty$ \citep*{haussler:95}.
Also, this implies that if $|G_{r}| = 1$, then \eqref{eqn:doubling-phi-bound} trivially holds,
so let us suppose $|G_{r}| \geq 2$.

Now fix any measurable functions $\gamma,\zeta,\xi$ mapping $\X \to [0,1]$
satisfying the constraint $\sup_{h \in \Ball(\target,r)} \E[\ind[h(X)=+1]\zeta(X) + \ind[h(X)=-1]\xi(X)] \leq r/c$,
where $X \sim \Px$, and $\forall x \in \X$, $\gamma(x)+\zeta(x)+\xi(x) = 1$;
for simplicity, also suppose $\E[\gamma(X)] \geq r$.
As above, for $m \in \nats$, let $X_{1},\ldots,X_{m}$ be independent $\Px$-distributed random variables.
Then let $\Gamma_{1},\ldots,\Gamma_{m}$ be conditionally independent given $X_{1},\ldots,X_{m}$, with 
the conditional distribution of each $\Gamma_{i}$ as ${\rm Bernoulli}(\gamma(X_{i}))$ given $X_{1},\ldots,X_{m}$.
Let $N_{m} = | \{ i \in [m] : \Gamma_{i} = 1 \}|$, and let $\hat{X}_{1},\ldots,\hat{X}_{N_{m}}$ denote the 
subsequence of $X_{1},\ldots,X_{m}$ for which the respective $\Gamma_{i} = 1$.

By two applications of the Chernoff bound, combined with the union bound, 
the event $E_{1} = \{ m \E[\gamma(X)] / 2 \leq N_{m} \leq 2 m \E[\gamma(X)] \}$ 
has probability at least $1 - 2\exp\{ - m \E[\gamma(X)] / 8 \}$.
Additionally, $\forall f,g \in G_{r}$ with $f \neq g$, $\forall i \in [m]$,
\begin{align*}
& \P( f(X_{i}) \neq g(X_{i}) \text{ and } \Gamma_{i} = 0 )
\\ & = \E[ \ind[ f(X) \neq g(X) ] (1-\gamma(X)) ]
= \E[ \ind[ f(X) \neq g(X) ] (\zeta(X) + \xi(X)) ]
\\ & = \E\left[ \left( \ind[ f(X) = +1 ]\ind[ g(X) = -1 ] + \ind[ f(X) = -1 ]\ind[ g(X) = +1 ] \right) (\zeta(X) + \xi(X)) \right]
\\ & \leq \E\left[ \ind[ f(X) \!=\! +1 ]\zeta(X) + \ind[ f(X) \!=\! -1 ] \xi(X) \right] 
+ \E\left[ \ind[ g(X) \!=\! -1] \xi(X) + \ind[ g(X) \!=\! +1 ] \zeta(X) \right]
\\ & \leq \frac{2 r}{c},
\end{align*}
so that
\begin{equation*}
\P( f(X_{i}) \!\neq\! g(X_{i}) \text{ and } \Gamma_{i} \!=\! 1 ) 
= \P( f(X_{i}) \!\neq\! g(X_{i}) ) - \P( f(X_{i}) \!\neq\! g(X_{i}) \text{ and } \Gamma_{i} \!=\! 0 )
> \frac{r}{2} - \frac{2r}{c}.
\end{equation*}
In particular, this implies
\begin{equation*}
\P\left( f(X_{i}) \neq g(X_{i}) \middle| \Gamma_{i} = 1 \right)
\geq \left(\frac{1}{2} - \frac{2}{c}\right) \frac{r}{\E[\gamma(X)]}.
\end{equation*}
Therefore,
\begin{align*}
& \P\left( \exists i \in [N_{m}] : f(\hat{X}_{i}) \neq g(\hat{X}_{i}) \middle| N_{m} \right)
= 1 - \left( 1 - \P( f(X_{1}) \neq g(X_{1}) | \Gamma_{1} = 1 ) \right)^{N_{m}}
\\ & \geq 1 - \left( 1 - \left(\frac{1}{2} - \frac{2}{c} \right) \frac{r}{\E[\gamma(X)]} \right)^{N_{m}}
\geq 1 - \exp\left\{ - \left(\frac{1}{2} - \frac{2}{c} \right) \frac{r}{\E[\gamma(X)]} N_{m} \right\}.
\end{align*}
On the event $E_{1}$, this is at least $1 - \exp\left\{ - \left(\frac{1}{4} - \frac{1}{c} \right) r m \right\}$.
Altogether, we have that
\begin{align*}
& \P\left( E_{1} \text{ and } \exists i \in [N_{m}] : f(\hat{X}_{i}) \neq g(\hat{X}_{i}) \right)
= \E\left[ \ind_{E_{1}} \cdot \P\left( \exists i \in [N_{m}] : f(\hat{X}_{i}) \neq g(\hat{X}_{i}) \middle| N_{m} \right) \right]
\\ & \geq \left( 1 - \exp\left\{ - \left(\frac{1}{4} - \frac{1}{c} \right) r m \right\} \right) \P(E_{1})
\\ & \geq 1 - \exp\left\{ - \left(\frac{1}{4} - \frac{1}{c} \right) r m \right\} - 2 \exp\{ - m \E[\gamma(X)] / 8 \}
\\ & \geq 1 - \exp\left\{ - \left(\frac{c-4}{4c}\right) r m \right\} - 2 \exp\{ - m r / 8 \}.
\end{align*}
In particular, choosing 
\begin{equation*}
m = \left\lceil \frac{1}{r} \left(\frac{4c}{c-4} \lor 8\right) \ln\left( 2|G_{r}|^{2} \right) \right\rceil,
\end{equation*}
we have that $\P\left( E_{1} \text{ and } \exists i \in [N_{m}] : f(\hat{X}_{i}) \neq g(\hat{X}_{i}) \right) \geq 1 - \frac{2}{|G_{r}|^{2}}$.
By a union bound, this implies that with probability at least $1 - \frac{2}{|G_{r}|^{2}} \binom{|G_{r}|}{2} = \frac{1}{|G_{r}|} > 0$, 
$E_{1}$ holds and, for every $f,g \in G_{r}$ with $f \neq g$, $\exists i \in [N_{m}]$ for which $f(\hat{X}_{i}) \neq g(\hat{X}_{i})$:
that is, every $f \in G_{r}$ classifies $\hat{X}_{1},\ldots,\hat{X}_{N_{m}}$ distinctly.
But for this to be the case, $|G_{r}|$ can be at most the number of distinct classifications 
of a sequence of $N_{m}$ points in $\X$ realizable by classifiers in $\C$, 
where (since $E_{1}$ also holds) $N_{m} \leq 2 m \E[\gamma(X)]$.
Together with the VC-Sauer lemma \citep*{vapnik:71,sauer:72}, this implies that
\begin{align*}
& \log_{2}(|G_{r}|) \leq \dim \log_{2}\left( \frac{2 e m \E[\gamma(X)]}{\dim} \lor 2 \right)
\\ & \leq \dim \log_{2}\left( \frac{35 \cdot 4 e}{33} \left(\frac{4c}{c-4} \lor 8\right) \frac{\E[\gamma(X)]}{r} \frac{1}{\dim} \left( \ln(\sqrt{2}) + \ln(|G_{r}|) \right) \lor 2 \right)
\\ & = \dim \log_{2}\left( \frac{35 \cdot 4 e}{33 \log_{2}(e)} \left(\frac{4c}{c-4} \lor 8\right) \frac{\E[\gamma(X)]}{r} \frac{1}{\dim} \left( (1/2) + \log_{2}(|G_{r}|) \right) \lor 2 \right),
\end{align*}
where the second inequality follows from the fact that $8 \ln(2 |G_{r}|^{2}) > 16.5$ (since $|G_{r}| \geq 2$), 
so that $m \leq \frac{17.5}{16.5} \frac{1}{r} \left( \frac{4c}{c-4} \lor 8 \right) \ln(2 |G_{r}|^{2}) = \frac{35}{33} \frac{1}{r} \left( \frac{4c}{c-4} \lor 8 \right) \ln(2 |G_{r}|^{2})$.

If $\log_{2}(|G_{r}|) \leq \dim$, then together with \eqref{eqn:packing-covering-inequalities}, 
the inequality \eqref{eqn:doubling-phi-bound} trivially holds.
Otherwise, if $\log_{2}(|G_{r}|) > \dim$, then letting $K = \frac{1}{\dim} \log_{2}(|G_{r}|) \geq 1$, 
the above implies
\begin{align*}
K & \leq \log_{2}\left( \frac{35 \cdot 4 e}{33 \log_{2}(e)} \left(\frac{4c}{c-4} \lor 8\right) \frac{\E[\gamma(X)]}{r} \frac{3}{2} K \right)
\\ & = \log_{2}\left( \frac{35 \cdot 4 e}{22 \log_{2}(e)} \left(\frac{4c}{c-4} \lor 8\right) \frac{\E[\gamma(X)]}{r} \right) + \log_{2}(K).
\end{align*}
Via some simple calculus \citep*[see e.g.,][Lemma 4.6]{vidyasagar:03}, this implies 
\begin{equation*}
K \leq 2\log_{2}\left( \frac{35 \cdot 4 e}{22 \log_{2}(e)} \left(\frac{4c}{c-4} \lor 8\right) \frac{\E[\gamma(X)]}{r} \right).
\end{equation*}
Noting that $\frac{35 \cdot 4 e}{22 \log_{2}(e)} < 12$,
together with \eqref{eqn:packing-covering-inequalities},
we have that
\begin{equation}
\label{eqn:raw-gamma-bound}
\log_{2}( \covering( r/2, \Ball(\target,r), \Px ) ) 
\leq 2 \dim \log_{2}\left( 12 \left(\frac{4c}{c-4} \lor 8\right) \frac{\E[\gamma(X)]}{r} \right).
\end{equation}
This inequality holds for any choice of $\gamma,\zeta,\xi$ satisfying 
the constraints in the definition of $\Phi(\Ball(\target,r),r/c)$ from 
Definition~\ref{def:zc}, with the additional constraint that $\E[\gamma(X)] \geq r$.
Thus, if $\Phi(\Ball(\target,r),r/c) \geq r$, then by minimizing the right hand side of 
\eqref{eqn:raw-gamma-bound} over the choice of $\gamma,\zeta,\xi$,
it follows that
\begin{equation*}
\log_{2}( \covering( r/2, \Ball(\target,r), \Px ) ) 
\leq 2 \dim \log_{2}\left( 12 \left(\frac{4c}{c-4} \lor 8\right) \frac{\Phi(\Ball(\target,r),r/c)}{r} \right).
\end{equation*}
Otherwise, if $\Phi(\Ball(\target,r),r/c) < r$, then we note that, for any
functions $\gamma^{*},\zeta^{*},\xi^{*}$ satisfying the constraints from
the definition of $\Phi(\Ball(\target,r),r/c)$ such that $\E[\gamma^{*}(X)] \!=\! \Phi(\Ball(\target,r),r/c)$,
there exists functions $\gamma,\zeta,\xi$ satisfying the constraints from
the definition of $\Phi(\Ball(\target,r),r/c)$ for which $\E[\gamma(X)]=r$.
For instance, we can take $\gamma$ based on a convex combination of $\gamma^{*}$ and $1$:
$\gamma(x) = \frac{1-r}{1-\E[\gamma^{*}(X)]} \gamma^{*}(x) + \frac{r-\E[\gamma^{*}(X)]}{1-\E[\gamma^{*}(X)]}$,
$\zeta(x) = (\zeta^{*}(x) - (\gamma(x) - \gamma^{*}(x))) \lor 0$,
$\xi(x) = 1 - \gamma(x) - \zeta(x)$; one can easily verify that, since $0 \leq \zeta(x) \leq \zeta^{*}(x)$ and $0 \leq \xi(x) \leq \xi^{*}(x)$,
this choice of $\gamma,\zeta,\xi$ still satisfy the requirements for $\gamma,\zeta,\xi$ above,
and that furthermore, $\E[\gamma(X)] = r$.
Therefore, \eqref{eqn:raw-gamma-bound} implies 
$\log_{2}( \covering( r/2, \Ball(\target,r), \Px) ) \leq 2 \dim \log_{2}\left( 12 \left( \frac{4c}{c-4} \lor 8 \right) \right)$. 
Thus, either way, we have established that
\begin{equation}
\label{eqn:phi-bound-with-c}
\log_{2}( \covering( r/2, \Ball(\target,r), \Px ) ) 
\leq 2 \dim \log_{2}\left( 12 \left(\frac{4c}{c-4} \lor 8\right) \left( \frac{\Phi(\Ball(\target,r),r/c)}{r} \lor 1 \right) \right).
\end{equation}
Noting that, for any $c \geq 8$, $\frac{4c}{c-4} \leq 8$, 
this establishes \eqref{eqn:doubling-phi-bound} for any $c \geq 8$ and $r \in (0,1]$.

In the case of $r > 1$, a result of \citet*{haussler:95} implies that
\begin{align*}
& \log_{2}( \covering( r/2, \Ball(\target,r), \Px) )
\leq \log_{2}( \covering( 1/2, \C, \Px) )
\leq \dim \log_{2}(4 e) + \log_{2}(e (\dim+1))
\\ & \leq \dim \log_{2}(4 e) \!+\! \dim \!+\! \log_{2}(e)
\leq \dim \log_{2}(8 e^{2})
\leq \dim \log_{2}( 96 )
\leq 2 \dim \log_{2}\!\left( 96 \left( \frac{\Phi(\Ball(\target,r),r/c)}{r} \!\lor\! 1 \right) \right),
\end{align*}
so that both \eqref{eqn:doubling-phi-bound} and \eqref{eqn:phi-bound-with-c} are also valid for $r > 1$.
This completes the proof of \eqref{eqn:doubling-phi-bound}.

As a final step in the proof of \eqref{eqn:doubling-zc-bound}, 
we note that there is a slight complication to be resolved, 
since the defintion of $\dd(r_{0})$ includes $r_{0}$ in the range of $r$, 
while the definition of $\zc_{c}(r_{0})$ does not.  However, we note that, 
for any $c > 4$, any $r_{0} > 0$, and any $r > r_{0}$ sufficiently close to $r_{0}$, 
we have $c > c r_{0} / r > 4$, 
so that \eqref{eqn:phi-bound-with-c}
would imply 
\begin{align*}
\log_{2}\!\left( \covering\!\left( r_{0}/2, \Ball(\target,r_{0}), \Px\right) \right) 
& \leq 2 \dim \log_{2}\!\left( 12 \left( \frac{4 (c r_{0}/r)}{(c r_{0}/r)-4} \!\lor\! 8 \right) \!\left( \frac{\Phi(\Ball(\target,r_{0}),r_{0}/(c r_{0}/r))}{r_{0}} \!\lor\! 1 \right) \right)
\\ & \leq 2 \dim \log_{2}\!\left( 12 \left( \frac{4 c}{(c r_{0}/r)-4} \lor \frac{8 r}{r_{0}} \right) \left( \frac{\Phi(\Ball(\target,r),r/c)}{r} \lor 1 \right) \right).
\end{align*}
Then taking the limit as $r \searrow r_{0}$ implies
\begin{align*}
\log_{2}\left( \covering\!\left( r_{0}/2, \Ball(\target,r_{0}), \Px\right) \right) 
& \leq 2 \dim \log_{2}\left( 12 \left( \frac{4 c}{c-4} \lor 8 \right) \lim_{r \searrow r_{0}} \left( \frac{\Phi(\Ball(\target,r),r/c)}{r} \lor 1 \right) \right)
\\ & \leq 2 \dim \log_{2}\left( 12 \left( \frac{4 c}{c-4} \lor 8 \right) \zc_{c}(r_{0}) \right).
\end{align*}
In particular, for any $c \geq 8$, $\frac{4c}{c-4} \leq 8$, so that 
\begin{equation*}
\log_{2}\left( \covering\!\left( r_{0}/2, \Ball(\target,r_{0}), \Px\right) \right) \leq 2 \dim \log_{2}\left( 96 \zc_{c}(r_{0}) \right).
\end{equation*}
Together with the above, we therefore have that, for any $c \geq 8$ and $r_{0} > 0$, 
\begin{align*}
D(r_{0})
& = \max\left\{ \log_{2}(\covering(r_{0}/2,\Ball(\target,r_{0}), \Px)), \sup_{r > r_{0}} \log_{2}(\covering(r/2,\Ball(\target,r),\Px)) \right\}
\\ & \leq \max\left\{ 2 \dim \log_{2}\left( 96 \zc_{c}(r_{0}) \right), \sup_{r > r_{0}} 2 \dim \log_{2}\left( 96 \left( \frac{\Phi(\Ball(\target,r),r/c)}{r} \lor 1 \right) \right) \right\}
\\ & = 2 \dim \log_{2}\left( 96 \zc_{c}(r_{0}) \right).
\end{align*}
Thus, we have established \eqref{eqn:doubling-zc-bound}.

\section{Proofs of Results on Learning with Noise}
\label{app:noise}

This appendix includes the proofs of results in Section~\ref{sec:noise}: namely, Theorems~\ref{thm:bounded-lower-bound} and \ref{thm:zc-noise}.

\subsection{Proof of Theorem~\ref{thm:bounded-lower-bound}}
\label{app:bounded-lower-proof}

We begin with the proof of Theorem~\ref{thm:bounded-lower-bound}.
The proof follows a technique of \citet*{hanneke:15b}, which identifies
a subset of classifiers in $\C$, corresponding to a certain concept space 
for which \citet*{raginsky:11} have established lower bounds.
Specifically, the following setup is taken directly from \citet*{hanneke:15b}.
Fix $\zeta \in (0,1]$, $\bound \in [0,1/2)$, and $k \in \nats$ with $k \leq \min\left\{ 1/\zeta, |\X|-1 \right\}$.
Let $\X_{k} = \{x_{1},\ldots,x_{k+1}\}$ be a set of $k+1$ distinct elements of $\X$,
and define $\C_{k} = \{x \mapsto 2 \ind_{\{x_{i}\}}(x) - 1 : i \in [k] \}$.
Let $\Px_{k,\zeta}$ be a probability measure over $\X$ with $\Px(\{x_{i}\}) = \zeta$ for each $i \in [k]$,
and $\Px_{k,\zeta}(\{x_{k+1}\}) = 1-\zeta k$.  For each $t \in [k]$, 
let $P^{\prime}_{k,\zeta,t}$ be a probability measure over $\X \times \Y$ with 
marginal distribution $\Px_{k,\zeta}$ over $\X$,
such that for $(X,Y) \sim P^{\prime}_{k,\zeta,t}$, every $i \in [k]$ has
$\P( Y = 2 \ind_{\{x_{t}\}}(X) - 1 | X = x_{i}) = 1-\bound$, 
and $\P(Y=-1|X=x_{k+1})=1$.  
\citet*{raginsky:11} prove the following result (see the proof of their Theorem 1).\footnote{As noted
by \citet*{hanneke:15b}, although technically the proof of this result by \citet*{raginsky:11} 
relies on a lemma (their Lemma 4) that imposes additional restrictions on $k$ and a parameter ``$d$'',
one can easily verify that the conclusions of that lemma continue to hold in the special case considered
here (corresponding to $d=1$ and arbitrary $k \in \nats$) by defining $\mathcal{M}_{k,1} = \{0,1\}_{1}^{k}$ 
in their construction.}

\begin{lemma}
\label{lem:rr11}
For $k$, $\zeta$, $\bound$ as above, with $k \geq 2$, 
for any $\conf \in (0,1/4)$,
for any (passive) learning rule $\alg$, and any $m \in \nats$ with
\begin{equation*}
m < \max\left\{ \frac{\bound \ln\left(\frac{1}{4\conf}\right)}{2 \zeta (1-2\bound)^{2}}, \frac{3 \bound \ln\left(\frac{k}{96}\right)}{16 \zeta (1-2\bound)^{2}} \right\},
\end{equation*}
if $\C_{k} \subseteq \C$, then there exists a $t \in [k]$ such that, if $\PXY = P^{\prime}_{k,\zeta,t}$, 
then denoting $\hat{h}_{m} = \alg(\L_{m})$, with probability greater than $\conf$, 
\begin{equation*}
\er(\hat{h}_{m}) - \inf_{h \in \C} \er(h) \geq (\zeta/2)(1-2\bound).
\end{equation*}
\end{lemma}

Continuing to follow \citet*{hanneke:15b}, we embed the above scenario
into the general case, so that Lemma~\ref{lem:rr11} provides a lower bound.
Fix any $\zeta \in (0,1]$, $\bound \in [0,1/2)$, and $k \in \nats$ with $k \leq \min\left\{ \s-1,\lfloor 1/\zeta \rfloor \right\}$,
and let $x_{1},\ldots,x_{k+1}$ and $h_{0},h_{1},\ldots,h_{k}$ be as in Definition~\ref{def:star}.
Let $\Px_{k,\zeta}$ be as above (for this choice of $x_{1},\ldots,x_{k+1}$),
and for each $t \in [k]$, let $P_{k,\zeta,t}$ denote a probability measure over $\X \times \Y$
with marginal distribution $\Px_{k,\zeta}$ over $\X$ such that, for $(X,Y) \sim P_{k,\zeta,t}$, 
$\P(Y = h_{t}(X) | X=x_{i})=1-\bound$ for every $i \in [k]$, while $\P(Y=h_{t}(X) | X=x_{k+1}) = 1$.

\begin{lemma}
\label{lem:rr11-star}
For $k$, $\zeta$, $\bound$ as above, with $k \geq 96e$, for any $\conf \in (0,1/4)$, 
for any (passive) learning rule $\alg$, and any $m \in \nats$ with 
\begin{equation*}
m < \frac{3 \bound \ln\left(\frac{k}{96}\right)}{16 \zeta (1-2\bound)^{2}},
\end{equation*}
there exists a $t \in [k]$ such that, if $\PXY = P_{k,\zeta,t}$, 
then denoting $\hat{h}_{m} = \alg(\L_{m})$, with probability greater than $\conf$,
\begin{equation*}
\er(\hat{h}_{m}) - \inf_{h \in \C} \er(h) \geq (\zeta/2)(1-2\bound).
\end{equation*}
\end{lemma}

The proof of Lemma~\ref{lem:rr11-star} is essentially identical to the proof
of \citet*[][Lemma 26]{hanneke:15b}, except that the algorithm $\alg$ here 
is restricted to be a passive learning rule so that Lemma~\ref{lem:rr11} can 
be applied (in place of Lemma 25 there).  As such, we omit the details here
for brevity.

We are now ready for the proof of Theorem~\ref{thm:bounded-lower-bound}.

\begin{proof}[of Theorem~\ref{thm:bounded-lower-bound}]
Fix any $\bound \in (0,1/2)$, $\conf \in (0,1/24)$, $m \in \nats$, and any (passive) learning rule $\alg$.
First consider the case of $\s \geq 97e$.
Fix $\eps \in (0, (1-2\bound)/(384e^{2})]$, and let $\zeta = \frac{2\eps}{1-2\bound}$ and $k = \min\left\{ \s-1, \lfloor 1/\zeta \rfloor \right\}$.
Then, noting that the distributions $P_{k,\zeta,t}$ above satisfy the $\bound$-bounded noise condition,
Lemma~\ref{lem:rr11-star} implies that if 
\begin{equation}
\label{eqn:bounded-lower-bound-rr-bound}
m < \frac{3 \bound \ln\left(\frac{k}{96}\right)}{32 \eps (1-2\bound)},
\end{equation}
then there exists a choice of $\PXY$ satisfying the $\bound$-bounded noise condition such that,
with probability greater than $\conf$, the classifier $\hat{h}_{m} = \alg(\L_{m})$ has
\begin{equation*}
\er(\hat{h}_{m}) - \inf_{h \in \C} \er(h) \geq \eps.
\end{equation*}
Note that for any $m \in \nats$ and $\eps \in (0,(1-2\bound)/(384e^{2})]$, it holds that \citep*[see e.g.,][Corollary 4.1]{vidyasagar:03}
\begin{align*}
& m \leq \frac{3 \bound}{64 \eps (1-2\bound)} \ln\left( \frac{(1-2\bound)^{2} m}{18 \bound} \right)
\\ & \implies m < \frac{3 \bound \ln\left(\frac{1-2\bound}{384 \eps}\right)}{32 \eps (1-2\bound)}
\leq \frac{3 \bound \ln\left(\frac{\lfloor 1/\zeta \rfloor}{96}\right)}{32 \eps (1-2\bound)}.
\end{align*}
Thus, the inequality in \eqref{eqn:bounded-lower-bound-rr-bound} is satisfied if both
\begin{equation*}
m < \frac{3 \bound \ln\left( \frac{\s-1}{96}\right)}{32 \eps (1-2\bound)}
\end{equation*}
and
\begin{equation*}
m \leq \frac{3 \bound}{64 \eps (1-2\bound)} \ln\left( \frac{(1-2\bound)^{2} m}{18 \bound} \right).
\end{equation*}
Solving for a value $\eps \in (0,(1-2\bound)/(384e^{2})]$ that satisfies both of these, 
we have that for any $m \in \nats$ with $m \geq \frac{18 e \bound}{(1-2\bound)^{2}}$, 
there is a choice of $\PXY$ satisfying the $\bound$-bounded noise condition such that,
with probability greater than $\conf$, 
\begin{align*}
\er(\hat{h}_{m}) - \inf_{h \in \C} \er(h)
& \geq \frac{3 \bound \ln\left( \min\left\{\frac{\s-1}{96}, \frac{(1-2\bound)^{2} m}{18 \bound} \right\}\right)}{64 (1-2\bound) m} \land \frac{1-2\bound}{384 e^{2}}
\\ & \gtrsim \frac{\bound \Log\left( \min\left\{\s, (1-2\bound)^{2} m \right\} \right)}{(1-2\bound) m} \land (1-2\bound).
\end{align*}
Furthermore, for $m < \frac{18 e \bound}{(1-2\bound)^{2}}$, 
we may also think of $\hat{h}_{m}$ as the output of $\alg^{\prime}(\L_{m^{\prime}})$
for $m^{\prime} = \left\lceil \frac{18 e \bound}{(1-2\bound)^{2}} \right\rceil > m$,
for a learning rule $\alg^{\prime}$ which simply discards the last $m^{\prime}-m$ samples
and runs $\alg(\L_{m})$ to produce its return classifier.  Thus, the above result implies
that for $m < \frac{18 e \bound}{(1-2\bound)^{2}}$, with probability greater than $\conf$, 
\begin{equation*}
\er(\hat{h}_{m}) - \inf_{h \in \C} \er(h)
\geq \frac{3 \bound \ln\left( \min\left\{\frac{\s-1}{96}, \frac{(1-2\bound)^{2} m^{\prime}}{18 \bound} \right\}\right)}{64 (1-2\bound) m^{\prime}} \land \frac{1-2\bound}{384 e^{2}}.
\end{equation*}
Since $m,m^{\prime} \in \nats$ and $m^{\prime} > m$, we know that $m^{\prime} \geq 2$, so that 
$\frac{18 e \bound}{(1-2\bound)^{2}} \leq m^{\prime} \leq \frac{36 e \bound}{(1-2\bound)^{2}}$.
Therefore,
\begin{equation*}
\frac{3 \bound \ln\left( \min\left\{\frac{\s-1}{96}, \frac{(1-2\bound)^{2} m^{\prime}}{18 \bound} \right\}\right)}{64 (1-2\bound) m^{\prime}}
\geq \frac{3 \bound}{64 (1-2\bound) m^{\prime}}
\geq \frac{3 (1-2\bound)}{64 \cdot 36 e}
> \frac{(1-2\bound)}{384 e^{2}}.
\end{equation*}
Thus, in this case, we have that with probability greater than $\conf$, 
\begin{equation*}
\er(\hat{h}_{m}) - \inf_{h \in \C} \er(h) 
\geq \frac{(1-2\bound)}{384 e^{2}}
\gtrsim (1-2\bound)
\geq \frac{\bound \Log\left( \min\left\{\s, (1-2\bound)^{2} m \right\} \right)}{(1-2\bound) m} \land (1-2\bound).
\end{equation*}

Next, we return to the general case of arbitrary $\s \in \nats \cup \{\infty\}$.
In particular, since any $\s < 97e$ has $\frac{\bound \Log\left( \min\left\{ \s, (1-2\bound)^{2} m \right\} \right)}{(1-2\bound) m} \lesssim \frac{\dim}{(1-2\bound) m}$,
to complete the proof it suffices to establish a lower bound 
\begin{equation*}
\er(\hat{h}_{m}) - \inf_{h \in \C} \er(h) \gtrsim \frac{1}{(1-2\bound) m} \left( \dim + \Log\left(\frac{1}{\conf}\right) \right) \land (1-2\bound),
\end{equation*}
holding with probability greater than $\conf$.
This lower bound is already known, and frequently referred to in the literature;
it follows from well-known constructions \citep*[see e.g.,][]{anthony:99,massart:06,hanneke:11a,hanneke:fntml}.
The case $\bound < 3/8$ is covered by the classic minimax lower bound 
of \citet*{ehrenfeucht:89} for the realizable case,
while the case $\bound \geq 3/8$ is addressed by \citet*[][Theorem 3.5]{hanneke:fntml}.
However, it seems an explicit proof of this latter result has not actually appeared in the 
literature.  As such, for completeness, we include a brief sketch of the argument here.

Suppose $\bound \geq 3/8$.
We begin with the term $\frac{1}{(1-2\bound) m} \Log\left(\frac{1}{\conf}\right)$.
Since we have assumed $|\C| \geq 3$, there must exist $x_{0},x_{1} \in \X$ and $h_{0},h_{1} \in \C$
such that $h_{0}(x_{0}) = h_{1}(x_{0})$ while $h_{0}(x_{1}) \neq h_{1}(x_{1})$.
Now fix $\eps = \frac{3}{8 (1-2\bound) m}\ln\left(\frac{1}{5\conf}\right) \land (1-2\bound)$, let $\Px(\{x_{1}\}) = \frac{\eps}{1-2\bound}$,
and let $\Px(\{x_{0}\}) = 1-\Px(\{x_{1}\})$.  Then, for $b \in \{0,1\}$, we let $P_{b}$ be a distribution
on $\X \times \Y$ with marginal $\Px$ over $\X$, and with $P_{b}( \{(x_{0},h_{0}(x_{0}))\} | \{x_{0}\} \times \Y ) = 1$
and $P_{b}( \{(x_{1},h_{b}(x_{1}))\} | \{x_{1}\} \times \Y ) = 1-\bound$.  Then one can easily check that,
for $\PXY = P_{b}$, any classifier $h$ with $h(x_{1}) \neq h_{b}(x_{1})$ has
$\er(h) - \inf_{g \in \C} \er(g) \geq \eps$.  But since 
${\rm KL}( P_{0}^{m} \| P_{1}^{m} ) = m {\rm KL}(P_{0} \| P_{1}) = m \eps \ln\left( \frac{1-\bound}{\bound} \right)$,
and $\ln\left( \frac{1-\bound}{\bound} \right) \leq \frac{1-\bound}{\bound} - 1 = \frac{1-2\bound}{\bound} \leq \frac{8}{3} (1-2\bound)$ (since $\bound \geq 3/8$),
classic hypothesis testing lower bounds \citep*[see][Theorem 2.2]{tsybakov:09} imply that
there exists a choice of $b \in \{0,1\}$ such that, with $\PXY = P_{b}$ and $\hat{h}_{m} = \alg(\L_{m})$,
$\P( \hat{h}_{m}(x_{1}) \neq h_{b}(x_{1}) ) \geq \frac{1}{4} \exp\left\{ - m \eps \frac{8}{3} (1-2\bound) \right\} \geq (5/4)\conf > \conf$.
Thus, with probability greater than $\conf$, $\er(\hat{h}_{m}) - \inf_{g \in \C} \er(g) \geq \eps \gtrsim \frac{1}{(1-2\bound) m} \Log\left(\frac{1}{\conf}\right)$.

Next, we present a proof for the term $\frac{\dim}{(1-2\bound) m}$, again for $\bound \geq 3/8$.
This term is trivially implied by the term $\frac{1}{(1-2\bound) m}\Log\left(\frac{1}{\conf}\right)$
in the case $\dim = 1$, so suppose $\dim \geq 2$.
This time, we let $\{x_{0},\ldots,x_{\dim-1}\}$ denote a subset of $\X$ shatterable by $\C$,
fix $\eps = \frac{3(\dim-1)}{64e (1-2\bound) m} \land \frac{1-2\bound}{8e}$, and 
let $\Px(\{x_{i}\}) = \frac{8e \eps}{(\dim-1) (1-2\bound)}$ for $i \in \{1,\ldots,\dim-1\}$,
and $\Px(\{x_{0}\}) = 1 - \frac{8e \eps}{1-2\bound}$.  Now for each $\bar{b} = (b_{1},\ldots,b_{\dim-1}) \in \{0,1\}^{\dim-1}$, 
let $P_{\bar{b}}$ denote a probability measure on $\X \times \Y$ with marginal $\Px$ over $\X$, 
and with $P_{\bar{b}}( \{(x_{i},2b_{i}-1)\} | \{x_{i}\} \times \Y ) = 1 - \bound$ for every $i \in \{1,\ldots,\dim-1\}$,
and $P_{\bar{b}}(\{(x_{0},-1)\} | \{x_{0}\} \times \Y) = 1$.
In particular, note that any $\bar{b},\bar{b}^{\prime} \in \{0,1\}^{\dim-1}$ with Hamming distance $\|\bar{b} - \bar{b}^{\prime}\|_{1} = 1$ 
have ${\rm KL}(P_{\bar{b}}^{m} \| P_{\bar{b}^{\prime}}^{m}) = m {\rm KL}(P_{\bar{b}} \| P_{\bar{b}^{\prime}}) = m \frac{8e \eps}{\dim-1} \ln\left(\frac{1-\bound}{\bound}\right)$,
and as above, $\ln\left(\frac{1-\bound}{\bound}\right) \leq \frac{8}{3} (1-2\bound)$.
Now Assouad's lemma \citep*[see][Theorem 2.12]{tsybakov:09} implies that there exists a $\bar{b} \in \{0,1\}^{\dim-1}$
such that, with $\PXY = P_{\bar{b}}$ and $\hat{h}_{m} = \alg(\L_{m})$, denoting $\hat{b} = ((1+\hat{h}_{m}(x_{1}))/2,\ldots,(1+\hat{h}_{m}(x_{\dim-1}))/2)$,
we have $\E\left[ \|\hat{b} - \bar{b}\|_{1} \right] \geq \frac{\dim-1}{4} \exp\left\{ - m \frac{8e \eps}{\dim-1} \frac{8}{3} (1-2\bound) \right\} \geq \frac{\dim-1}{4e}$.
Noting that $0 \leq \|\hat{b}-\bar{b}\|_{1} \leq \dim-1$, this further implies that
$\P\left( \|\hat{b}-\bar{b}\|_{1} \geq \frac{\dim-1}{8e} \right) \geq \frac{1}{8e}$.
Furthermore, note that $\er(\hat{h}_{m}) - \inf_{g \in \C} \er(g) \geq \|\hat{b}-\bar{b}\|_{1} \frac{8e \eps}{\dim-1}$.
Thus, $\P\left( \er(\hat{h}_{m}) - \inf_{g \in \C} \er(g) \geq \eps \right) \geq \frac{1}{8e} > \conf$.
Finally, note that $\eps \gtrsim \frac{\dim}{(1-2\bound) m} \land (1-2\bound)$.

Altogether, by choosing which ever of these lower bounds is greatest, 
we have that for any $m \in \nats$, there exists a choice of $\PXY$ satisfying the $\bound$-bounded noise condition
such that, with probability greater than $\conf$,
\begin{equation*}
\er(\hat{h}_{m}) - \inf_{h \in \C} \er(h) \gtrsim \frac{\max\left\{ \dim, \bound\Log\left( \min\left\{ \s, (1-2\bound)^{2} m \right\} \right), \Log\left(\frac{1}{\conf}\right)\right\}}{(1-2\bound) m} \land (1-2\bound).
\end{equation*}
Applying the relaxation $\max\{a,b,c\} \geq (1/3) (a + b + c)$ (for nonnegative values $a,b,c$)
then completes the proof of the first lower bound stated in the theorem.

For the second inequality, note that by taking $\conf = 1/24$, the inequality proven above implies that there exists
a distribution $\PXY$ satisfying the $\bound$-bounded noise condition such that, with probability greater than $1/24$, 
\begin{equation*}
\er(\hat{h}_{m}) - \inf_{h \in \C} \er(h) \gtrsim \frac{\dim + \bound \Log\left(\min\left\{ \s, (1-2\bound)^{2} m \right\} \right)}{(1-2\bound) m} \land (1-2\bound).
\end{equation*}
Furthermore, since bounded noise distributions have $\inf_{h \in \C} \er(h)$ equal the Bayes risk,
$\er(\hat{h}_{m}) - \inf_{h \in \C} \er(h)$ is always nonnegative.  We therefore have
\begin{align*}
\E\left[ \er(\hat{h}_{m}) - \inf_{h \in \C} \er(h) \right] 
& \gtrsim \frac{23}{24} 0 + \frac{1}{24} \frac{\dim + \bound \Log\left(\min\left\{ \s, (1-2\bound)^{2} m \right\} \right)}{(1-2\bound) m} \land (1-2\bound)
\\ & \gtrsim \frac{\dim + \bound \Log\left(\min\left\{ \s, (1-2\bound)^{2} m \right\} \right)}{(1-2\bound) m} \land (1-2\bound).
\end{align*}
Finally, since $\inf\limits_{h \in \C} \er(h)$ is nonrandom,
$\E\left[ \er(\hat{h}_{m}) \right ] - \inf\limits_{h \in \C} \er(h) = \E\left[ \er(\hat{h}_{m}) - \inf\limits_{h \in \C} \er(h) \right]$.
\end{proof}

\subsection{Proof of Theorem~\ref{thm:zc-noise}}
\label{app:zcnoise-proof}

Next, we present the proof of Theorem~\ref{thm:zc-noise}.
We begin by stating a classic result, due to \citet*{gine:06}
\citep*[see also][]{van-der-Vaart:11,hanneke:12b}.
For any set $\H$ of classifiers, denote $\diam_{\Px}(\H) = \sup_{h,g \in \H} \Px(x : h(x) \neq g(x))$.

\begin{lemma}
\label{lem:gk-envelopes}
There is a universal constant $c_{0} \in (1,\infty)$ such that,
for any set $\H$ of classifiers, for any $\conf \in (0,1)$ and $m \in \nats$,
defining
\begin{equation*}
U(\H,m,\conf;R) \!=\! 1 \land \inf_{r > \diam_{\Px}(\H)} c_{0} \sqrt{ r \frac{\vc(\H) \Log\!\left( \!\frac{\Px(R)}{r}\! \right) \!+\! \Log\!\left(\frac{1}{\conf}\right)}{m} } + c_{0} \frac{\vc(\H) \Log\!\left(\!\frac{\Px(R)}{r}\!\right) \!+\! \Log\!\left(\frac{1}{\conf}\right)}{m}
\end{equation*}
for every measurable $R \subseteq \X$,
with probability at least $1-\conf$, $\forall h \in \H$, 
\begin{align*}
\er(h) - \inf_{g \in \H} \er(g) & \leq \max\left\{ 2 \left( \er_{\L_{m}}(h) - \min_{g \in \H} \er_{\L_{m}}(g) \right), U(\H,m,\conf;\DIS(\H)) \right\}, \\
\er_{\L_{m}}(h) - \min_{g \in \H} \er_{\L_{m}}(g) & \leq \max\left\{ 2 \left( \er(h) - \inf_{g \in \H} \er(g) \right), U(\H,m,\conf;\DIS(\H)) \right\}.
\end{align*}
\end{lemma}

Next, we note that we lose very little by requiring the $\gamma$ function in Definition~\ref{def:zc} to be binary.
This allows us to simplify certain parts of the proof of Theorem~\ref{thm:zc-noise} below.

\begin{lemma}
\label{lem:zc-binary}
For any set $\H$ of classifiers, and any $\eta \in [0,1]$, 
for $X \sim \Px$, 
letting
\begin{multline*}
\Phi_{\{0,1\}}(\H,\eta) = \inf\left\{ \E[\gamma(X)] : \sup_{h \in \H} \E\left[ \ind[h(X)=+1]\zeta(X) + \ind[h(X)=-1]\xi(X) \right] \leq \eta, \right.
\\ \left. \phantom{\sup_{h \in \H}} \text{ where } \forall x \in \X, \gamma(x) + \zeta(x) + \xi(x) = 1 \text{ and } \zeta(x),\xi(x) \in [0,1], \gamma(x) \in \{0,1\} \right\},
\end{multline*}
we have that
\begin{equation*}
\Phi(\H,\eta) \leq \Phi_{\{0,1\}}(\H,\eta) \leq 2 \Phi(\H,\eta/2).
\end{equation*}
\end{lemma}
\begin{proof}
The left inequality is clear from the definitions.  For the right inequality, 
let $\gamma^{*},\zeta^{*},\xi^{*}$ be the functions at the optimal solution achieving $\Phi(\H,\eta/2)$ in Definition~\ref{def:zc}.
For every $x \in \X$, if $\gamma^{*}(x) \geq 1/2$, define $\gamma(x) = 1$ and $\zeta(x) = \xi(x) = 0$,
and otherwise define $\gamma(x) = 0$, $\zeta(x) = \zeta^{*}(x) / (\zeta^{*}(x)+\xi^{*}(x))$, and $\xi(x) = \xi^{*}(x) / (\zeta^{*}(x)+\xi^{*}(x))$.
By design, we have that $\gamma(x) \in \{0,1\}$, $\zeta(x),\xi(x) \in [0,1]$, and $\gamma(x)+\zeta(x)+\xi(x)=1$ for every $x \in \X$.
Since every $x \in \X$ has $\gamma(x) \leq 2 \gamma^{*}(x)$, we have $\E[\gamma(X)] \leq 2 \E[\gamma^{*}(X)] = 2 \Phi(\H,\eta/2)$.
Furthermore, for every $x \in \X$, we either have $\zeta(x) = 0 \leq 2 \zeta^{*}(x)$ and $\xi(x) = 0 \leq 2 \xi^{*}(x)$,
or else $\gamma^{*}(x) < 1/2$, in which case $\zeta^{*}(x)+\xi^{*}(x) = 1-\gamma^{*}(x) > 1/2$, 
so that $\zeta(x) = \zeta^{*}(x) / (\zeta^{*}(x)+\xi^{*}(x)) \leq 2 \zeta^{*}(x)$ and $\xi(x) = \xi^{*}(x) / (\zeta^{*}(x)+\xi^{*}(x)) \leq 2 \xi^{*}(x)$.
Therefore, 
\begin{align*}
& \sup_{h \in \H} \E\left[ \ind[h(X)=+1]\zeta(X) + \ind[h(X)=-1]\xi(X) \right] 
\\ & \leq 2 \sup_{h \in \H} \E\left[ \ind[h(X)=+1]\zeta^{*}(X) + \ind[h(X)=-1]\xi^{*}(X) \right] \leq \eta.
\end{align*}
Thus, $\gamma,\zeta,\xi$ are functions in the feasible region of the optimization problem defining $\Phi_{\{0,1\}}(\H,\eta)$,
so that $\Phi_{\{0,1\}}(\H,\eta) \leq \E[\gamma(X)] \leq 2 \Phi(\H,\eta/2)$.
\end{proof}

We will establish the claim in Theorem~\ref{thm:zc-noise} for the following algorithm (which has the data set $\L_{m}$ as input).
For simplicity, this algorithm is stated in a way that makes it $\Px$-dependent (which is consistent with 
the statement of Theorem~\ref{thm:zc-noise}).  
It may be possible to remove this dependence by replacing the $\Px$-dependent quantities with empirical estimates,
but we leave this task to future work
(e.g., see the work of \citealp*{koltchinskii:06}, for discussion of empirical estimation of $U(\H,m,\conf;R)$;
\citealp*{zhang:14}, additionally discuss estimating the minimizing function $\gamma$ from the definition of $\Phi$,
though some refinement to their concentration arguments would be needed for our purposes).
For any $k \in \{0,1,\ldots, \lfloor \log_{2}(m) \rfloor - 1 \}$, define $\conf_{k} = \frac{\conf}{(\log_{2}(2m)-k)^{2}}$,
and fix a value $\eta_{k} \geq 0$ (to be specified in the proof below).

\begin{bigbigboxit}
$\ZCPassive$:
\\ 0. $\G_{0} \gets \C$
\\ 1. For $k = 0,1,\ldots,\lfloor \log_{2}(m) \rfloor-1$
\\ 2. \quad Let $\gamma_{k}$ be the function $\gamma$ at the solution defining $\Phi_{\{0,1\}}(\G_{k},\eta_{k})$
\\ 3. \quad $R_{k} \gets \{x \in \X : \gamma_{k}(x) = 1\}$
\\ 4. \quad $D_{k} \gets \{ (X_{i},Y_{i}) : 2^{k}+1 \leq i \leq 2^{k+1}, X_{i} \in R_{k} \}$
\\ 5. \quad $\G_{k+1}\!\gets\! \left\{ h \!\in\! \G_{k} : 2^{-k}|D_{k}|\!\left( \er_{D_{k}}(h) \!-\! \min\limits_{g \in \G_{k}} \er_{D_{k}}(g) \right) \!\leq\! \max\{4\eta_{k},U(\G_{k}, 2^{k}, \conf_{k}; R_{k})\} \right\}$
\\ 6. Return any $\hat{h} \in \G_{\lfloor \log_{2}(m) \rfloor}$
\end{bigbigboxit}

For simplicity, we suppose the function $\gamma_{k}$ in Step 2 actually minimizes $\E[\gamma_{k}(X)]$ subject to the constraints in the definition of $\Phi_{\{0,1\}}(\G_{k},\eta_{k})$.
However, the proof below would remain valid for any $\gamma_{k}$ satisfying these constraints, with $\E[\gamma_{k}(X)] \leq 2 \Phi(\G_{k},\eta_{k}/2)$:
for instance, the proof of Lemma~\ref{lem:zc-binary} reveals this would be satisfied by $\gamma_{k}(x) = \ind[\gamma^{*}(x) \geq 1/2]$ for the $\gamma^{*}$
achieving the minimum value of $\E[\gamma^{*}(X)]$ in the definition of $\Phi(\G_{k},\eta_{k}/2)$.
Indeed, it would even suffice to choose $\gamma_{k}$ satisfying the constraints of $\Phi_{\{0,1\}}(\G_{k},\eta_{k})$ with $\E[\gamma_{k}(X)] \leq c^{\prime} \Phi(\G_{k},\eta_{k}/2)$,
for any finite numerical constant $c^{\prime}$, as this would only affect the numerical constant factors in Theorem~\ref{thm:zc-noise}.

We are now ready for the proof of Theorem~\ref{thm:zc-noise}.

\begin{proof}[of Theorem~\ref{thm:zc-noise}]
The proof is similar to those given above (e.g., that of Theorem~\ref{thm:vc-erm-subregion}),
except that the stronger form of Lemma~\ref{lem:gk-envelopes} (compared to Lemma~\ref{lem:classic-erm})
affords us a simplification that avoids the step in which we lower-bound the sample size under the conditional distribution given $\Gamma_{i}=1$.

Fix any $\tsybca \geq 1$ and $\tsyba \in (0,1]$, and
fix $c = 128$.  We establish the claim for $\ZCPassive$, described above.
Define $\eta_{0} = 2/c$ and $\tilde{U}_{0} = 1$,
and for each $k \in \{1,\ldots,\lfloor \log_{2}(m) \rfloor\}$, 
inductively define 
\begin{align*}
\tilde{U}_{k} & = \min\left\{ 1, 2\eta_{k-1} + \max\left\{ 8\eta_{k-1}, 2U(\G_{k-1},2^{k-1},\conf_{k-1};R_{k-1}) \right\} \right\},
\\ r_{k} & = \tsybca c_{1} \left( \tsybca 2^{1-k} \left( \dim \Log\left( \maxzc_{\tsybca,\tsyba}\left( \tsybca \left( \tsybca \dim 2^{1-k} \right)^{\frac{\tsyba}{2-\tsyba}} \right) \right) + \Log\left( \frac{1}{\conf_{k-1}}\right) \right) \right)^{\frac{\tsyba}{2-\tsyba}},
\\ \eta_{k} & = \frac{2}{c} \left( \frac{r_{k}}{\tsybca} \right)^{1/\tsyba},
\end{align*}
where $c_{1} = (32 c_{0})^{\frac{2\tsyba}{2-\tsyba}}$.
We proceed by induction on $k$ in the algorithm.
Suppose that, for some $k \in \{0,1,\ldots,\lfloor \log_{2}(m) \rfloor - 1\}$, 
there is an event $E_{k}$ of probability at least $1 - \sum_{k^{\prime} = 0}^{k-1} \conf_{k^{\prime}}$ (or probability $1$ if $k=0$),
on which $\agtarget \in \G_{k}$, and for some universal constant $c_{1} \in (1,\infty)$, every $k^{\prime} \in \{0,\ldots,k\}$ has
\begin{equation*}
\tilde{U}_{k^{\prime}}
\leq (c/2) \eta_{k^{\prime}},
\end{equation*}
and
\begin{equation*}
\G_{k^{\prime}} \subseteq \left\{ h \in \C : \er(h) - \er(\agtarget) \leq \tilde{U}_{k^{\prime}} \right\}.
\end{equation*}
In particular, these conditions are trivially satisfied for $k=0$, so this may serve as a base case for this inductive argument.
Next we must extend these conditions to $k+1$.

For each $h \in \G_{k}$, define $h_{R_{k}}(x) = h(x)\ind[x \in R_{k}] + \agtarget(x)\ind[x \notin R_{k}]$,
and denote $\H_{k} = \{ h_{R_{k}} : h \in \G_{k} \}$.
Noting that $R_{k} \supseteq \DIS(\H_{k})$, and that this implies $U\!\left(\H_{k},2^{k},\conf_{k}; R_{k}\right) \geq U\!\left(\H_{k},2^{k},\conf_{k}; \DIS(\H_{k})\right)$,
Lemma~\ref{lem:gk-envelopes} (applied under the conditional distribution given $\G_{k}$) and the law of total probability 
imply that there exists an event $E_{k+1}^{\prime}$ of probability at least $1-\conf_{k}$,
on which, $\forall h_{R_{k}} \in \H_{k}$, denoting $\tilde{\L}_{k} = \{(X_{i},Y_{i}) : 2^{k}+1 \leq i \leq 2^{k+1}\}$ (which is distributionally 
equivalent to $\L_{2^{k}}$ but independent of $\G_{k}$), 
\begin{align*}
\er(h_{R_{k}}) - \!\inf_{g_{R_{k}} \in \H_{k}} \er(g_{R_{k}}) & \leq \max\!\left\{ 2 \left( \er_{\tilde{\L}_{k}}\!(h_{R_{k}}) - \!\!\min_{g_{R_{k}} \in \H_{k}} \!\er_{\tilde{\L}_{k}}\!(g_{R_{k}}) \right), U\!\left(\H_{k},2^{k},\conf_{k};R_{k}\right) \right\}, \\
\er_{\tilde{\L}_{k}}\!(h_{R_{k}}) - \!\min_{g_{R_{k}} \in \H_{k}} \!\er_{\tilde{\L}_{k}}\!(g_{R_{k}}) & \leq \max\left\{ 2 \left( \er(h_{R_{k}}) - \!\inf_{g_{R_{k}} \in \H_{k}} \!\er(g_{R_{k}}) \right), U(\H_{k},2^{k},\conf_{k};R_{k}) \right\}.
\end{align*}

First we note that, since every $h_{R_{k}}$ and $g_{R_{k}}$ in $\H_{k}$ agree on the labels of all samples in $\tilde{\L}_{k} \setminus D_{k}$,
and they each agree with their respective classifiers $h$ and $g$ in $\G_{k}$ on $D_{k}$,
we have that
\begin{equation*}
\er_{\tilde{\L}_{k}}(h_{R_{k}}) - \min_{g_{R_{k}} \in \H_{k}} \er_{\tilde{\L}_{k}}(g_{R_{k}}) = 2^{-k}|D_{k}|\left( \er_{D_{k}}(h) - \min_{g \in \G_{k}} \er_{D_{k}}(g) \right).
\end{equation*}

Next, let $\zeta_{k}$ and $\xi_{k}$ denote the functions $\zeta$ and $\xi$ from the definition of $\Phi_{\{0,1\}}(\G_{k},\eta_{k})$
at the solution with $\gamma$ equal $\gamma_{k}$.  Note that $\zeta_{k}$ and $\xi_{k}$ are themselves random, but are competely determined by $\G_{k}$.
The definition of $R_{k}$ guarantees that for every $h,g \in \G_{k}$, for $X \sim \Px$ (independent from $\L_{m}$)
\begin{align*}
& \Px( x \notin R_{k} : h(x) \neq g(x) )
= \E\left[ \ind[h(X) \neq g(X)] ( \zeta_{k}(X) + \xi_{k}(X) ) \middle| \G_{k} \right]
\\ & = \E\left[ \left( \ind[h(X) = +1]\ind[g(X) = -1] + \ind[h(X)=-1]\ind[g(X)=+1] \right) (\zeta_{k}(X) + \xi_{k}(X)) \middle| \G_{k} \right]
\\ & \leq \E\left[ \ind[h(X)=+1]\zeta_{k}(X) + \ind[h(X)=-1]\xi_{k}(X) \middle| \G_{k} \right] 
\\ & {\hskip 8mm} + \E\left[ \ind[g(X)=+1]\zeta_{k}(X) + \ind[g(X)=-1]\xi_{k}(X) \middle| \G_{k} \right]
\leq 2\eta_{k}.
\end{align*}
Therefore, 
\begin{equation*}
\er(h_{R_{k}}) - \er(g_{R_{k}}) 
\leq \er(h) - \er(g) + \Px( x \notin R_{k} : h(x) \neq g(x) ) 
\leq \er(h) - \er(g) + 2\eta_{k},
\end{equation*}
and similarly 
\begin{equation*}
\er(h_{R_{k}}) - \er(g_{R_{k}}) 
\geq \er(h) - \er(g) - \Px( x \notin R_{k} : h(x) \neq g(x) )
\geq \er(h) - \er(g) - 2\eta_{k}.
\end{equation*}
In particular, noting that 
$\er(h_{R_{k}}) - \inf_{g_{R_{k}} \in \H_{k}} \er(g_{R_{k}})
= \sup_{g \in \G_{k}} \er(h_{R_{k}}) - \er(g_{R_{k}})$
and
$\sup_{g \in \G_{k}} \er(h) - \er(g) 
= \er(h) - \inf_{g \in \G_{k}} \er(g)$,
this implies 
\begin{equation*}
\er(h) - \inf_{g \in \G_{k}} \er(g) - 2\eta_{k}
\leq \er(h_{R_{k}}) - \inf_{g_{R_{k}} \in \H_{k}} \er(g_{R_{k}})
\leq \er(h) - \inf_{g \in \G_{k}} \er(g) + 2\eta_{k}.
\end{equation*}
We also note that $\vc(\H_{k}) \leq \vc(\G_{k})$ and $\diam_{\Px}(\H_{k}) \leq \diam_{\Px}(\G_{k})$,
which together imply $U(\H_{k},2^{k},\conf_{k};R_{k}) \leq U(\G_{k},2^{k},\conf_{k};R_{k})$.
Altogether, we have that on $E_{k+1}^{\prime}$, $\forall h \in \G_{k}$, 
\begin{align*}
& \er(h) - \inf_{g \in \G_{k}} \er(g) \leq 2\eta_{k} + \max\left\{ 2^{1-k} |D_{k}| \left( \er_{D_{k}}(h) - \min_{g \in \G_{k}} er_{D_{k}}(g) \right), U(\G_{k},2^{k},\conf_{k};R_{k}) \right\}, \\
& 2^{-k} |D_{k}| \left( \er_{D_{k}}\!(h) - \min_{g \in \G_{k}} \er_{D_{k}}\!(g) \right) \leq \max\!\left\{ 2 \left( \er(h) - \inf_{g \in \G_{k}} \er(g) + 2\eta_{k} \right), U(\G_{k},2^{k},\conf_{k};R_{k}) \right\}.
\end{align*}

In particular, defining $E_{k+1} = E_{k+1}^{\prime} \cap E_{k}$, we have that on $E_{k+1}$, $\agtarget \in \G_{k}$, and
\begin{equation*}
2^{-k} |D_{k}| \left( \er_{D_{k}}(\agtarget) - \min_{g \in \G_{k}} \er_{D_{k}}(g) \right)
\leq \max\left\{ 4\eta_{k}, U(\G_{k},2^{k},\conf_{k};R_{k}) \right\},
\end{equation*}
so that $\agtarget \in \G_{k+1}$ as well.
Furthermore, combined with the definition of $\G_{k+1}$, this further implies that on $E_{k+1}$, 
\begin{align*}
\G_{k+1} & \subseteq \left\{ h \in \C : \er(h) - \er(\agtarget) \leq 2\eta_{k} + \max\left\{ 8\eta_{k}, 2 U\left(\G_{k},2^{k},\conf_{k};R_{k}\right) \right\} \right\}
\\ & = \left\{ h \in \C : \er(h) - \er(\agtarget) \leq \tilde{U}_{k+1} \right\}.
\end{align*}

It remains only to establish the bound on $\tilde{U}_{k+1}$.
For this, we first note that,
combining the inductive hypothesis with the $(\tsybca,\tsyba)$-Bernstein class condition, 
on $E_{k+1}$ we have 
\begin{equation*}
\G_{k} \subseteq \Ball\left(\agtarget, \tsybca \tilde{U}_{k}^{\tsyba}\right)
\subseteq \Ball\left( \agtarget, r_{k} \right).
\end{equation*}
Combining this with Lemma~\ref{lem:zc-binary} and monotonicity of $\Phi(\cdot,\eta_{k}/2)$, we have that
\begin{equation*}
\Px(R_{k}) \leq 2 \Phi\left( \Ball\left( \agtarget, r_{k} \right), \eta_{k}/2 \right) 
= 2 \Phi\left( \Ball\left( \agtarget, r_{k} \right), (r_{k} / \tsybca)^{1/\tsyba} / c \right)
\leq 2 \maxzc_{\tsybca,\tsyba}(r_{k}) r_{k}.
\end{equation*}
The above also implies that $\diam_{\Px}(\G_{k}) \leq 2 r_{k}$ on $E_{k+1}$.
Together with the fact that $\vc(\G_{k}) \leq \dim$, we have that on $E_{k+1}$,
\begin{multline}
\label{eqn:zc-noise-U-bound}
U( \G_{k},2^{k},\conf_{k};R_{k} )
\leq c_{0} \sqrt{2r_{k} 2^{-k} \left( \dim \Log\left(\maxzc_{\tsybca,\tsyba}(r_{k})\right) + \Log\left(\frac{1}{\conf_{k}}\right) \right)} 
\\ + c_{0} 2^{-k} \left( \dim \Log\left(\maxzc_{\tsybca,\tsyba}(r_{k})\right) + \Log\left(\frac{1}{\conf_{k}}\right) \right).
\end{multline}
Furthermore, monotonicity of $\maxzc_{\tsybca,\tsyba}(\cdot)$ implies $\maxzc_{\tsybca,\tsyba}(r_{k}) \leq \maxzc_{\tsybca,\tsyba}\left( \tsybca ( \tsybca \dim 2^{-k} )^{\frac{\tsyba}{2-\tsyba}} \right)$.
Plugging the definition of $r_{k}$ into \eqref{eqn:zc-noise-U-bound} along with this relaxation of $\maxzc_{\tsybca,\tsyba}(r_{k})$ and simplifying,
the minimum of $1$ and the right hand side of \eqref{eqn:zc-noise-U-bound} is at most
\begin{multline*}
8 c_{0} \sqrt{c_{1}} \left( \tsybca 2^{-k} \left( \dim \Log\left( \maxzc_{\tsybca,\tsyba}\left( \tsybca \left(\tsybca \dim 2^{-k}\right)^{\frac{\tsyba}{2-\tsyba}} \right) \right) + \Log\left(\frac{1}{\conf_{k}}\right) \right) \right)^{\frac{1}{2-\tsyba}}
\\ =
8 c_{0} \sqrt{c_{1}} \left(\frac{r_{k+1}}{c_{1} \tsybca} \right)^{1/\tsyba}
= \frac{4 c_{0} c}{c_{1}^{\frac{2-\tsyba}{2\tsyba}}} \eta_{k+1}
= \frac{c}{8} \eta_{k+1}.
\end{multline*}
We may also observe that
\begin{equation*}
\eta_{k} \leq 4^{\frac{1}{2-\tsyba}} \eta_{k+1} \leq 4 \eta_{k+1}.
\end{equation*}
Combining the above with the definition of $\tilde{U}_{k+1}$, we have that on $E_{k+1}$, 
\begin{equation*}
\tilde{U}_{k+1} \leq 8\eta_{k+1} + \max\left\{ 32\eta_{k+1}, \frac{c}{4} \eta_{k+1} \right\}
= 40 \eta_{k+1} \leq 64 \eta_{k+1} = \frac{c}{2} \eta_{k+1}.
\end{equation*}
Finally, noting that the union bound implies $E_{k+1}$ has probability at least $1 - \sum_{k^{\prime}=0}^{k} \conf_{k^{\prime}}$ completes the inductive step.

By the principle of induction, we have thus established that, on an event $E_{\lfloor \log_{2}(m) \rfloor}$ of probability at least 
$1 - \sum_{k=0}^{\lfloor \log_{2}(m) \rfloor - 1} \conf_{k} > 1 - \conf \sum_{i=2}^{\infty} \frac{1}{i^{2}} > 1 - \conf$,
\begin{equation*}
\agtarget \in \G_{\lfloor \log_{2}(m) \rfloor} \subseteq \left\{ h \in \C : \er(h) - \er(\agtarget) \leq \frac{c}{2} \eta_{\lfloor \log_{2}(m) \rfloor} \right\}.
\end{equation*}
In particular, this implies that $\hat{h}$ exists in Step 6, and satisfies $\er(\hat{h}) - \inf_{g \in \C} \er(g) = \er(\hat{h}) - \er(\agtarget) \leq \frac{c}{2} \eta_{\lfloor \log_{2}(m) \rfloor}$.
Noting that
\begin{align*}
\frac{c}{2} \eta_{\lfloor \log_{2}(m) \rfloor}
& \leq c_{1}^{1/\tsyba} \left( \frac{4 \tsybca \left( \dim \Log\left( \maxzc_{\tsybca,\tsyba}\left( \tsybca \left( \frac{\tsybca \dim}{m} \right)^{\frac{\tsyba}{2-\tsyba}} \right) \right) + \Log\left( \frac{4}{\conf}\right) \right)}{m} \right)^{\frac{1}{2-\tsyba}}
\\ & \leq 6 (32 c_{0})^{2} \left( \frac{\tsybca \left( \dim \Log\left( \maxzc_{\tsybca,\tsyba}\left( \tsybca \left( \frac{\tsybca \dim}{m} \right)^{\frac{\tsyba}{2-\tsyba}} \right) \right) + \Log\left( \frac{1}{\conf}\right) \right)}{m} \right)^{\frac{1}{2-\tsyba}}
\end{align*}
completes the proof.
\end{proof}

\bibliography{learning}

\begin{thebibliography}{63}
\providecommand{\natexlab}[1]{#1}
\providecommand{\url}[1]{\texttt{#1}}
\expandafter\ifx\csname urlstyle\endcsname\relax
  \providecommand{\doi}[1]{doi: #1}\else
  \providecommand{\doi}{doi: \begingroup \urlstyle{rm}\Url}\fi

\bibitem[Ailon et~al.(2014)Ailon, Begleiter, and Ezra]{ailon:14}
N.~Ailon, R.~Begleiter, and E.~Ezra.
\newblock Active learning using smooth relative regret approximations with
  applications.
\newblock \emph{Journal of Machine Learning Research}, 15\penalty0
  (3):\penalty0 885--920, 2014.

\bibitem[Anthony and Bartlett(1999)]{anthony:99}
M.~Anthony and P.~L. Bartlett.
\newblock \emph{Neural Network Learning: Theoretical Foundations}.
\newblock Cambridge University Press, 1999.

\bibitem[Auer and Ortner(2004)]{auer:04}
P.~Auer and R.~Ortner.
\newblock A new {PAC} bound for intersection-closed concept classes.
\newblock In \emph{Proceedings of the $17^{{\rm th}}$ Conference on Learning
  Theory}, 2004.

\bibitem[Auer and Ortner(2007)]{auer:07}
P.~Auer and R.~Ortner.
\newblock A new {PAC} bound for intersection-closed concept classes.
\newblock \emph{Machine Learning}, 66\penalty0 (2-3):\penalty0 151--163, 2007.

\bibitem[Balcan and Long(2013)]{balcan:13}
M.-F. Balcan and P.~M. Long.
\newblock Active and passive learning of linear separators under log-concave
  distributions.
\newblock In \emph{Proceedings of the $26^{{\rm th}}$ Conference on Learning
  Theory}, 2013.

\bibitem[Balcan et~al.(2006)Balcan, Beygelzimer, and Langford]{balcan:06}
M.-F. Balcan, A.~Beygelzimer, and J.~Langford.
\newblock Agnostic active learning.
\newblock In \emph{Proceedings of the $23^{{\rm rd}}$ International Conference
  on Machine Learning}, 2006.

\bibitem[Balcan et~al.(2007)Balcan, Broder, and Zhang]{balcan:07}
M.-F. Balcan, A.~Broder, and T.~Zhang.
\newblock Margin based active learning.
\newblock In \emph{Proceedings of the $20^{{\rm th}}$ Conference on Learning
  Theory}, 2007.

\bibitem[Balcan et~al.(2009)Balcan, Beygelzimer, and Langford]{balcan:09}
M.-F. Balcan, A.~Beygelzimer, and J.~Langford.
\newblock Agnostic active learning.
\newblock \emph{Journal of Computer and System Sciences}, 75\penalty0
  (1):\penalty0 78--89, 2009.

\bibitem[Bartlett et~al.(2006)Bartlett, Jordan, and Mc{A}uliffe]{bartlett:06}
P.~Bartlett, M.~I. Jordan, and J.~Mc{A}uliffe.
\newblock Convexity, classification, and risk bounds.
\newblock \emph{Journal of the American Statistical Association}, 101\penalty0
  (473):\penalty0 138--156, 2006.

\bibitem[Bartlett and Mendelson(2006)]{bartlett:06b}
P.~L. Bartlett and S.~Mendelson.
\newblock Discussion: Local {R}ademacher complexities and oracle inequalities
  in risk minimization.
\newblock \emph{The Annals of Statistics}, 34\penalty0 (6):\penalty0
  2657--2663, 2006.

\bibitem[Blumer et~al.(1989)Blumer, Ehrenfeucht, Haussler, and
  Warmuth]{blumer:89}
A.~Blumer, A.~Ehrenfeucht, D.~Haussler, and M.~Warmuth.
\newblock Learnability and the {Vapnik-Chervonenkis} dimension.
\newblock \emph{Journal of the Association for Computing Machinery},
  36\penalty0 (4):\penalty0 929--965, 1989.

\bibitem[Bowers and Kalton(2014)]{bowers:14}
A.~Bowers and N.~J. Kalton.
\newblock \emph{An Introductory Course in Functional Analysis}.
\newblock Springer, 2014.

\bibitem[Bshouty et~al.(2009)Bshouty, Li, and Long]{long:09}
N.~H. Bshouty, Y.~Li, and P.~M. Long.
\newblock Using the doubling dimension to analyze the generalization of
  learning algorithms.
\newblock \emph{Journal of Computer and System Sciences}, 75\penalty0
  (6):\penalty0 323--335, 2009.

\bibitem[Cohn et~al.(1994)Cohn, Atlas, and Ladner]{cohn:94}
D.~Cohn, L.~Atlas, and R.~Ladner.
\newblock Improving generalization with active learning.
\newblock \emph{Machine Learning}, 15\penalty0 (2):\penalty0 201--221, 1994.

\bibitem[Cover(1965)]{cover:65}
T.~M. Cover.
\newblock Geometrical and statistical properties of systems of linear
  inequalities with applications in pattern recognition.
\newblock \emph{{IEEE} Transactions on Electronic Computers}, EC-14\penalty0
  (3):\penalty0 326--334, 1965.

\bibitem[Darnst\"{a}dt(2015)]{darnstadt:15}
M.~Darnst\"{a}dt.
\newblock The optimal {PAC} bound for intersection-closed concept classes.
\newblock \emph{Information Processing Letters}, 115\penalty0 (4):\penalty0
  458--461, 2015.

\bibitem[Dasgupta et~al.(2005)Dasgupta, Kalai, and Monteleoni]{dasgupta:05b}
S.~Dasgupta, A.~T. Kalai, and C.~Monteleoni.
\newblock Analysis of perceptron-based active learning.
\newblock In \emph{Proceedings of the $18^{{\rm th}}$ Conference on Learning
  Theory}, 2005.

\bibitem[Dasgupta et~al.(2007)Dasgupta, Hsu, and Monteleoni]{dasgupta:07}
S.~Dasgupta, D.~Hsu, and C.~Monteleoni.
\newblock A general agnostic active learning algorithm.
\newblock In \emph{Advances in Neural Information Processing Systems $20$},
  2007.

\bibitem[Ehrenfeucht et~al.(1989)Ehrenfeucht, Haussler, Kearns, and
  Valiant]{ehrenfeucht:89}
A.~Ehrenfeucht, D.~Haussler, M.~Kearns, and L.~Valiant.
\newblock A general lower bound on the number of examples needed for learning.
\newblock \emph{Information and Computation}, 82\penalty0 (3):\penalty0
  247--261, 1989.

\bibitem[El-Yaniv and Wiener(2010)]{el-yaniv:10}
R.~El-Yaniv and Y.~Wiener.
\newblock On the foundations of noise-free selective classification.
\newblock \emph{Journal of Machine Learning Research}, 11\penalty0
  (5):\penalty0 1605--1641, 2010.

\bibitem[El-Yaniv and Wiener(2011)]{el-yaniv:11}
R.~El-Yaniv and Y.~Wiener.
\newblock Agnostic selective classification.
\newblock In \emph{Advances in Neural Information Processing Systems $24$},
  2011.

\bibitem[El-Yaniv and Wiener(2012)]{el-yaniv:12}
R.~El-Yaniv and Y.~Wiener.
\newblock Active learning via perfect selective classification.
\newblock \emph{Journal of Machine Learning Research}, 13\penalty0
  (2):\penalty0 255--279, 2012.

\bibitem[Floyd and Warmuth(1995)]{floyd:95}
S.~Floyd and M.~Warmuth.
\newblock Sample compression, learnability, and the {V}apnik-{C}hervonenkis
  dimension.
\newblock \emph{Machine Learning}, 21\penalty0 (3):\penalty0 269--304, 1995.

\bibitem[Gin\'{e} and Koltchinskii(2006)]{gine:06}
E.~Gin\'{e} and V.~Koltchinskii.
\newblock Concentration inequalities and asymptotic results for ratio type
  empirical processes.
\newblock \emph{The Annals of Probability}, 34\penalty0 (3):\penalty0
  1143--1216, 2006.

\bibitem[Gupta et~al.(2003)Gupta, Krauthgamer, and Lee]{gupta:03}
A.~Gupta, R.~Krauthgamer, and J.~R. Lee.
\newblock Bounded geometries, fractals, and low-distortion embeddings.
\newblock In \emph{Proceedings of the $44^{{\rm th}}$ Annual IEEE Symposium on
  Foundations of Computer Science}, 2003.

\bibitem[Hanneke(2007{\natexlab{a}})]{hanneke:07a}
S.~Hanneke.
\newblock Teaching dimension and the complexity of active learning.
\newblock In \emph{Proceedings of the $20^{{\rm th}}$ Conference on Learning
  Theory}, 2007{\natexlab{a}}.

\bibitem[Hanneke(2007{\natexlab{b}})]{hanneke:07b}
S.~Hanneke.
\newblock A bound on the label complexity of agnostic active learning.
\newblock In \emph{Proceedings of the $24^{{\rm th}}$ International Conference
  on Machine Learning}, 2007{\natexlab{b}}.

\bibitem[Hanneke(2009)]{hanneke:thesis}
S.~Hanneke.
\newblock \emph{Theoretical Foundations of Active Learning}.
\newblock PhD thesis, Machine Learning Department, School of Computer Science,
  Carnegie Mellon University, 2009.

\bibitem[Hanneke(2011)]{hanneke:11a}
S.~Hanneke.
\newblock Rates of convergence in active learning.
\newblock \emph{The Annals of Statistics}, 39\penalty0 (1):\penalty0 333--361,
  2011.

\bibitem[Hanneke(2012)]{hanneke:12a}
S.~Hanneke.
\newblock Activized learning: Transforming passive to active with improved
  label complexity.
\newblock \emph{Journal of Machine Learning Research}, 13\penalty0
  (5):\penalty0 1469--1587, 2012.

\bibitem[Hanneke(2014)]{hanneke:fntml}
S.~Hanneke.
\newblock Theory of disagreement-based active learning.
\newblock \emph{Foundations and Trends in Machine Learning}, 7\penalty0
  (2--3):\penalty0 131--309, 2014.

\bibitem[Hanneke(2016)]{hanneke:16a}
S.~Hanneke.
\newblock The optimal sample complexity of {PAC} learning.
\newblock \emph{Journal of Machine Learning Research}, 17\penalty0
  (38):\penalty0 1--15, 2016.

\bibitem[Hanneke and Yang(2012)]{hanneke:12b}
S.~Hanneke and L.~Yang.
\newblock Surrogate losses in passive and active learning.
\newblock \emph{{arXiv}:1207.3772}, 2012.

\bibitem[Hanneke and Yang(2015)]{hanneke:15b}
S.~Hanneke and L.~Yang.
\newblock Minimax analysis of active learning.
\newblock \emph{Journal of Machine Learning Research}, 16\penalty0
  (12):\penalty0 3487--3602, 2015.

\bibitem[Haussler(1995)]{haussler:95}
D.~Haussler.
\newblock Sphere packing numbers for subsets of the {B}oolean n-cube with
  bounded {V}apnik-{C}hervonenkis dimension.
\newblock \emph{Journal of Combinatorial Theory A}, 69\penalty0 (2):\penalty0
  217--232, 1995.

\bibitem[Haussler et~al.(1994)Haussler, Littlestone, and Warmuth]{haussler:94}
D.~Haussler, N.~Littlestone, and M.~Warmuth.
\newblock Predicting $\{0,1\}$-functions on randomly drawn points.
\newblock \emph{Information and Computation}, 115\penalty0 (2):\penalty0
  248--292, 1994.

\bibitem[Helmbold et~al.(1990)Helmbold, Sloan, and Warmuth]{helmbold:90}
D.~Helmbold, R.~Sloan, and M.~Warmuth.
\newblock Learning nested differences of intersection-closed concept classes.
\newblock \emph{Machine Learning}, 5\penalty0 (2):\penalty0 165--196, 1990.

\bibitem[Herbrich(2002)]{herbrich:02}
R.~Herbrich.
\newblock \emph{Learning Kernel Classifiers}.
\newblock The {MIT} Press. Cambridge, MA, 2002.

\bibitem[Kolmogorov and Tikhomirov(1959)]{kolmogorov:59}
A.~N. Kolmogorov and V.~M. Tikhomirov.
\newblock $\varepsilon$-entropy and $\varepsilon$-capacity of sets in function
  spaces.
\newblock \emph{Uspekhi Matematicheskikh Nauk}, 14\penalty0 (2):\penalty0
  3--86, 1959.

\bibitem[Kolmogorov and Tikhomirov(1961)]{kolmogorov:61}
A.~N. Kolmogorov and V.~M. Tikhomirov.
\newblock $\varepsilon$-entropy and $\varepsilon$-capacity of sets in function
  spaces.
\newblock \emph{American Mathematical Society Translations, Series 2},
  17:\penalty0 277--364, 1961.

\bibitem[Koltchinskii(2006)]{koltchinskii:06}
V.~Koltchinskii.
\newblock Local {R}ademacher complexities and oracle inequalities in risk
  minimization.
\newblock \emph{The Annals of Statistics}, 34\penalty0 (6):\penalty0
  2593--2656, 2006.

\bibitem[Koltchinskii(2010)]{koltchinskii:10}
V.~Koltchinskii.
\newblock Rademacher complexities and bounding the excess risk in active
  learning.
\newblock \emph{Journal of Machine Learning Research}, 11\penalty0
  (9):\penalty0 2457--2485, 2010.

\bibitem[Kuhlmann(1999)]{kuhlmann:99}
C.~Kuhlmann.
\newblock On teaching and learning intersection-closed concept classes.
\newblock In \emph{Proceedings of the $12^{{\rm th}}$ Conference on Learning
  Theory}, 1999.

\bibitem[Le{C}am(1973)]{lecam:73}
L.~Le{C}am.
\newblock Convergence of estimates under dimensionality restrictions.
\newblock \emph{The Annals of Statistics}, 1\penalty0 (1):\penalty0 38--53,
  1973.

\bibitem[Li et~al.(2001)Li, Long, and Srinivasan]{li:01}
Y.~Li, P.~M. Long, and A.~Srinivasan.
\newblock The one-inclusion graph algorithm is near-optimal for the prediction
  model of learning.
\newblock \emph{{IEEE} Transactions on Information Theory}, 47\penalty0
  (3):\penalty0 1257--1261, 2001.

\bibitem[Littlestone and Warmuth(1986)]{littlestone:86}
N.~Littlestone and M.~Warmuth.
\newblock Relating data compression and learnability.
\newblock \emph{Unpublished manuscript}, 1986.

\bibitem[Long(2003)]{long:03}
P.~M. Long.
\newblock An upper bound on the sample complexity of {PAC} learning halfspaces
  with respect to the uniform distribution.
\newblock \emph{Information Processing Letters}, 87\penalty0 (5):\penalty0
  229--234, 2003.

\bibitem[Mammen and Tsybakov(1999)]{mammen:99}
E.~Mammen and A.B. Tsybakov.
\newblock Smooth discrimination analysis.
\newblock \emph{The Annals of Statistics}, 27\penalty0 (6):\penalty0
  1808--1829, 1999.

\bibitem[Massart and N\'{e}d\'{e}lec(2006)]{massart:06}
P.~Massart and E.~N\'{e}d\'{e}lec.
\newblock Risk bounds for statistical learning.
\newblock \emph{The Annals of Statistics}, 34\penalty0 (5):\penalty0
  2326--2366, 2006.

\bibitem[Natarajan(1987)]{natarajan:87}
B.~K. Natarajan.
\newblock On learning {B}oolean functions.
\newblock In \emph{Proceedings of the $19^{{\rm th}}$ Annual {ACM} Symposium on
  Theory of Computing}, 1987.

\bibitem[Raginsky and Rakhlin(2011)]{raginsky:11}
M.~Raginsky and A.~Rakhlin.
\newblock Lower bounds for passive and active learning.
\newblock In \emph{Advances in Neural Information Processing Systems $24$},
  2011.

\bibitem[Sauer(1972)]{sauer:72}
N.~Sauer.
\newblock On the density of families of sets.
\newblock \emph{Journal of Combinatorial Theory (A)}, 13\penalty0 (1):\penalty0
  145--147, 1972.

\bibitem[Tsybakov(2004)]{tsybakov:04}
A.~B. Tsybakov.
\newblock Optimal aggregation of classifiers in statistical learning.
\newblock \emph{The Annals of Statistics}, 32\penalty0 (1):\penalty0 135--166,
  2004.

\bibitem[Tsybakov(2009)]{tsybakov:09}
A.~B. Tsybakov.
\newblock \emph{Introduction to Nonparametric Estimation}.
\newblock Springer, 2009.

\bibitem[van~der Vaart and Wellner(2011)]{van-der-Vaart:11}
A.~van~der Vaart and J.~A. Wellner.
\newblock A local maximal inequality under uniform entropy.
\newblock \emph{Electronic Journal of Statistics}, 5:\penalty0 192--203, 2011.

\bibitem[van~der Vaart and Wellner(1996)]{van-der-Vaart:96}
A.~W. van~der Vaart and J.~A. Wellner.
\newblock \emph{Weak Convergence and Empirical Processes}.
\newblock Springer, 1996.

\bibitem[Vapnik and Chervonenkis(1971)]{vapnik:71}
V.~Vapnik and A.~Chervonenkis.
\newblock On the uniform convergence of relative frequencies of events to their
  probabilities.
\newblock \emph{Theory of Probability and its Applications}, 16\penalty0
  (2):\penalty0 264--280, 1971.

\bibitem[Vapnik and Chervonenkis(1974)]{vapnik:74}
V.~Vapnik and A.~Chervonenkis.
\newblock \emph{Theory of Pattern Recognition}.
\newblock Nauka, Moscow, 1974.

\bibitem[Vidyasagar(2003)]{vidyasagar:03}
M.~Vidyasagar.
\newblock \emph{Learning and Generalization with Applications to Neural
  Networks}.
\newblock Springer-Verlag, $2^{{\rm nd}}$ edition, 2003.

\bibitem[Warmuth(2004)]{warmuth:04}
M.~Warmuth.
\newblock The optimal {PAC} algorithm.
\newblock In \emph{Proceedings of the $17^{{\rm th}}$ Conference on Learning
  Theory}, 2004.

\bibitem[Wiener et~al.(2015)Wiener, Hanneke, and {El-Yaniv}]{hanneke:15a}
Y.~Wiener, S.~Hanneke, and R.~{El-Yaniv}.
\newblock A compression technique for analyzing disagreement-based active
  learning.
\newblock \emph{Journal of Machine Learning Research}, 16\penalty0
  (4):\penalty0 713--745, 2015.

\bibitem[Yang and Barron(1999)]{yang:99b}
Y.~Yang and A.~Barron.
\newblock Information-theoretic determination of minimax rates of convergence.
\newblock \emph{The Annals of Statistics}, 27\penalty0 (5):\penalty0
  1564--1599, 1999.

\bibitem[Zhang and Chaudhuri(2014)]{zhang:14}
C.~Zhang and K.~Chaudhuri.
\newblock Beyond disagreement-based agnostic active learning.
\newblock In \emph{Advances in Neural Information Processing Systems $27$},
  2014.

\end{thebibliography}

\end{document}